\documentclass[english]{article}
\usepackage[T1]{fontenc}
\usepackage[utf8]{inputenc}
\usepackage{babel}
\usepackage{float}
\usepackage{enumitem}
\usepackage{dsfont}
\usepackage{amsmath}
\usepackage{amsthm}
\usepackage{amssymb}
\usepackage[numbers]{natbib}
\usepackage[pdfusetitle,
 bookmarks=true,bookmarksnumbered=false,bookmarksopen=false,
 breaklinks=false,pdfborder={0 0 1},backref=false,colorlinks=false]
 {hyperref}

\makeatletter

\floatstyle{ruled}
\newfloat{algorithm}{tbp}{loa}
\providecommand{\algorithmname}{Algorithm}
\floatname{algorithm}{\protect\algorithmname}

\theoremstyle{remark}
\newtheorem{rem}{\protect\remarkname}
\theoremstyle{plain}
\newtheorem{assumption}{\protect\assumptionname}
\theoremstyle{plain}
\newtheorem{thm}{\protect\theoremname}
\theoremstyle{plain}
\newtheorem{cor}{\protect\corollaryname}
\theoremstyle{plain}
\newtheorem{lem}{\protect\lemmaname}
\theoremstyle{definition}
\newtheorem{defn}{\protect\definitionname}
\theoremstyle{definition}
\newtheorem{condition}{\protect\conditionname}
\theoremstyle{plain}
\newtheorem{prop}{\protect\propositionname}
\newlist{casenv}{enumerate}{4}
\setlist[casenv]{leftmargin=*,align=left,widest={iiii}}
\setlist[casenv,1]{label={{\itshape\ \casename} \arabic*.},ref=\arabic*}
\setlist[casenv,2]{label={{\itshape\ \casename} \roman*.},ref=\roman*}
\setlist[casenv,3]{label={{\itshape\ \casename\ \alph*.}},ref=\alph*}
\setlist[casenv,4]{label={{\itshape\ \casename} \arabic*.},ref=\arabic*}
\theoremstyle{plain}
\newtheorem{fact}{\protect\factname}

\@ifundefined{date}{}{\date{}}
\usepackage[paperwidth=8.5in, paperheight=11in, margin=1in]{geometry}

\usepackage{enumitem}
\setlist[itemize]{leftmargin=*}
\setlist[enumerate]{leftmargin=*}

\makeatother

\providecommand{\assumptionname}{Assumption}
\providecommand{\casename}{Case}
\providecommand{\conditionname}{Condition}
\providecommand{\corollaryname}{Corollary}
\providecommand{\definitionname}{Definition}
\providecommand{\factname}{Fact}
\providecommand{\lemmaname}{Lemma}
\providecommand{\propositionname}{Proposition}
\providecommand{\remarkname}{Remark}
\providecommand{\theoremname}{Theorem}

\begin{document}
\global\long\def\A{\mathsf{A}}%
\global\long\def\B{\mathcal{B}}%
\global\long\def\C{\mathsf{C}}%
\global\long\def\E{\mathbb{E}}%
\global\long\def\N{\mathbb{N}}%
\global\long\def\R{\mathbb{R}}%
\global\long\def\calR{\mathcal{R}}%
\global\long\def\P{\mathbb{P}}%
\global\long\def\F{\mathcal{F}}%
\global\long\def\S{\mathcal{S}}%
\global\long\def\J{\mathcal{J}}%
\global\long\def\G{\mathcal{G}}%
\global\long\def\X{\mathcal{X}}%
\global\long\def\argmin{\mathrm{argmin}}%
\global\long\def\bg{\boldsymbol{g}}%
\global\long\def\bh{\boldsymbol{h}}%
\global\long\def\bu{\boldsymbol{u}}%
\global\long\def\bv{\boldsymbol{v}}%
\global\long\def\bw{\boldsymbol{w}}%
\global\long\def\bx{\boldsymbol{x}}%
\global\long\def\by{\boldsymbol{y}}%
\global\long\def\bz{\boldsymbol{z}}%
\global\long\def\brho{\boldsymbol{\rho}}%
\global\long\def\bzero{\boldsymbol{0}}%
\global\long\def\p{\mathsf{p}}%
\global\long\def\defeq{\triangleq}%
\global\long\def\d{\mathrm{d}}%
\global\long\def\1{\mathds{1}}%
\global\long\def\prog{\mathsf{prog}}%
\global\long\def\rand{\mathsf{rand}}%
\global\long\def\op{\mathsf{op}}%
\global\long\def\interior{\mathsf{int}}%
\global\long\def\ortho{\mathsf{Ortho}}%
\global\long\def\init{\mathsf{init}}%
\global\long\def\uni{\mathsf{Uniform}}%
\global\long\def\reg{\mathsf{R}}%
\global\long\def\dom{\mathsf{dom}}%
\global\long\def\mydots{\dots}%
\global\long\def\err{\boldsymbol{\epsilon}}%

\global\long\def\OGD{\mathsf{OGD}}%
\global\long\def\OOGD{\mathsf{OOGD}}%
\global\long\def\DA{\mathsf{DA}}%
\global\long\def\FTRL{\mathsf{FTRL}}%
\global\long\def\Ada{\mathsf{AdaGrad}}%
\global\long\def\otnc{\mathsf{O2NC}}%
\global\long\def\SGD{\mathsf{SGD}}%
\global\long\def\OAda{\mathsf{OAdaGrad}}%
\global\long\def\OAdaR{\mathsf{OAdaGradR}}%
\global\long\def\OMD{\mathsf{OMD}}%
\global\long\def\OOMD{\mathsf{OOMD}}%

\title{Online Convex Optimization with Heavy Tails:\\Old Algorithms, New
Regrets, and Applications\thanks{A short, self-contained version, including a subset of results, has
been accepted at ALT 2026.}}
\author{Zijian Liu\thanks{Stern School of Business, New York University, zl3067@stern.nyu.edu.}}
\maketitle
\begin{abstract}
In Online Convex Optimization (OCO), when the stochastic gradient
has a finite variance, many algorithms provably work and guarantee
a sublinear regret. However, limited results are known if the gradient
estimate has a heavy tail, i.e., the stochastic gradient only admits
a finite $\mathsf{p}$-th central moment for some $\mathsf{p}\in\left(1,2\right]$.
Motivated by it, this work examines different old algorithms for OCO
(e.g., Online Gradient Descent) in the more challenging heavy-tailed
setting. Under the standard bounded domain assumption, we establish
new regrets for these classical methods without any algorithmic modification.
Remarkably, these regret bounds are fully optimal in all parameters
(can be achieved even without knowing $\mathsf{p}$), suggesting that
OCO with heavy tails can be solved effectively without any extra operation
(e.g., gradient clipping). Our new results have several applications.
A particularly interesting one is the first provable and optimal convergence
result for nonsmooth nonconvex optimization under heavy-tailed noise
without gradient clipping. Furthermore, we explore broader settings
(e.g., smooth OCO) and extend our ideas to optimistic algorithms to
handle different cases simultaneously.
\end{abstract}

\section{Introduction\label{sec:introduction}}

This paper studies the online learning problem with convex losses,
also known as Online Convex Optimization (OCO), a widely applicable
framework that learns under streaming data \citep{Cesa-Bianchi_Lugosi_2006,OPT-013,orabona2019modern,MAL-018}.
OCO has tons of implications for both designing and analyzing algorithms
in different areas, for example, stochastic optimization \citep{JMLR:v12:duchi11a,kingma2014adam,McMahanS10},
PAC learning \citep{1327806}, control theory \citep{pmlr-v97-agarwal19c,hazan2025introductiononlinecontrol},
etc.

In an OCO problem, a learning algorithm $\A$ would interact with
the environment in $T$ rounds, where $T\in\N$ can be either known
or unknown. Formally, in each round $t$, the learner $\A$ first
decides an output $\bx_{t}\in\X$ from a convex feasible set $\X\subseteq\R^{d}$,
then the environment reveals a convex loss function $\ell_{t}:\X\to\R$,
and $\A$ incurs a loss of $\ell_{t}(\bx_{t})$. After $T$ many rounds,
the quantity measuring the algorithm's performance is called regret,
defined relative to any fixed competitor $\bx\in\X$ as follows:
\[
\reg_{T}^{\A}(\bx)\defeq\sum_{t=1}^{T}\ell_{t}(\bx_{t})-\ell_{t}(\bx).
\]

In the classical setting, instead of observing full information about
$\ell_{t}$, the learner $\A$ is only guaranteed to receive a subgradient
$\nabla\ell_{t}(\bx_{t})\in\partial\ell_{t}(\bx_{t})$ at its decision,
where $\partial\ell_{t}(\bx_{t})$ denotes the subdifferential set
of $\ell_{t}$ at $\bx_{t}$ \citep{rockafellar1997convex}. This
turns out to be enough for our purpose of minimizing the regret, since
any OCO problem can be reduced to an Online Linear Optimization (OLO)
instance via the inequality $\ell_{t}(\bx_{t})-\ell_{t}(\bx)\leq\left\langle \nabla\ell_{t}(\bx_{t}),\bx_{t}-\bx\right\rangle $,
which holds due to convexity. Under the standard bounded domain assumption,
i.e., $\X$ has a finite diameter $D$, many classical algorithms,
e.g., Online Gradient Descent ($\OGD$) \citep{zinkevich2003online},
guarantee an optimal sublinear regret $GD\sqrt{T}$ for $G$-Lipschitz
$\ell_{t}$. Even better, in the case that computing an exact subgradient
is intractable, and one could only query a stochastic estimate $\bg_{t}$
satisfying $\E\left[\bg_{t}\mid\bx_{t}\right]\in\partial\ell_{t}(\bx_{t})$,
the $\OGD$ algorithm can still solve OCO effectively with a provable
$(G+\sigma)D\sqrt{T}$ regret bound in expectation if the stochastic
noise $\bg_{t}-\nabla\ell_{t}(\bx_{t})$ has a bounded second moment
$\sigma^{2}$ for some $\sigma\geq0$, which is called the finite
variance condition.

However, many works have pointed out that even for the easier stochastic
optimization (i.e., $\ell_{t}=F$ for a common $F$), the typical
finite variance assumption is too optimistic and can be violated in
different tasks \citep{pmlr-v139-hodgkinson21a,pmlr-v97-simsekli19a,NEURIPS2020_b05b57f6},
and their observations suggest that the stochastic gradient only admits
a finite $\mathsf{p}$-th central moment upper bounded by $\sigma^{\p}$
for some $\mathsf{p}\in\left(1,2\right]$, which is named heavy-tailed
noise. This new assumption generalizes the classical finite variance
condition ($\p=2$) and becomes challenging when $\p<2$. A particular
evidence is that the famous Stochastic Gradient Descent ($\SGD$)
algorithm \citep{10.1214/aoms/1177729586} (which is exactly $\OGD$
for stochastic optimization) provably diverges \citep{NEURIPS2020_b05b57f6}.

Though heavy-tailed stochastic optimization has been extensively studied
\citep{liu2023stochasticV2,NEURIPS2023_4c454d34,pmlr-v202-sadiev23a},
limited results are known for OCO with heavy tails. The only work
under this topic that we are aware of is \citep{NEURIPS2022_349956de},
which established a parameter-free regret bound in high probability
(more discussions provided later). However, their algorithm includes
many nontrivial modifications like gradient clipping and significantly
deviates from the existing simple OCO algorithms used in practice.
Especially, consider $\OGD$ as an example. Though the heavy-tailed
issue is known, $\OGD$ (or just think of it as $\SGD$) still works
(sometimes very well) in practice even without gradient clipping and
is arguably one of the most popular optimizers, which seemingly contradicts
the theory of nonconvergence mentioned before. This indicates that,
for classical OCO algorithms under heavy-tailed noise, a huge gap
exists between the empirical convergence (or even the effective practical
performance) and theoretical guarantees. Therefore, we are naturally
led to the following question:
\begin{center}
\textit{In what context can old OCO algorithms work under heavy tails,
in what sense, and to what extent?}
\par\end{center}

\subsection{Contributions}

Motivated by the above question, we examine three classical algorithms
for OCO: Online Gradient Descent ($\OGD$) \citep{zinkevich2003online},
Dual Averaging ($\DA$) \citep{nesterov2009primal,NIPS2009_7cce53cf},
and $\Ada$ \citep{JMLR:v12:duchi11a,McMahanS10}, and answer it as
follows:
\begin{center}
\textit{Under the standard bounded domain assumption, the in-expectation
regret $\E\left[\reg_{T}^{\A}(\bx)\right]$ is finite and optimal
for any $\A\in\left\{ \OGD,\DA,\Ada\right\} $, without any algorithmic
modification.}
\par\end{center}

In detail, our new results for heavy-tailed OCO are summarized here:
\begin{itemize}
\item We prove the only and the first optimal regret bound $\E\left[\reg_{T}^{\A}(\bx)\right]\lesssim GD\sqrt{T}+\sigma DT^{1/\p},\forall\bx\in\X$
for any $\A\in\left\{ \OGD,\DA,\Ada\right\} $. Remarkably, $\Ada$
can achieve this result without knowing any of the Lipschitz parameter
$G$, noise level $\sigma$, and tail index $\p$.
\item We extend the analysis of $\OGD$ to Online Strongly Convex Optimization
with heavy tails and establish the first provable result $\E\left[\reg_{T}^{\OGD}(\bx)\right]\lesssim\frac{G^{2}\log T}{\mu}+\frac{\sigma^{\p}D^{2-\p}}{\mu^{\p-1}}T^{2-\p},\forall\bx\in\X$,
where $\mu>0$ is the modulus of strong convexity and $T^{0}$ should
be read as $\log T$.
\end{itemize}
Based on the new regret bounds for OCO with heavy tails, we provide
the following applications:
\begin{itemize}
\item For nonsmooth convex optimization with heavy tails, we show the first
optimal in-expectation rate $GD/\sqrt{T}+\sigma D/T^{1-1/\p}$ achieved
without gradient clipping, which applies to both the average iterate
and last iterate, demonstrating that $\SGD$ does converge once the
domain is bounded. Moreover, we also give the rate when strong convexity
additionally holds.
\item For nonsmooth nonconvex optimization with heavy tails, we show the
first provable sample complexity of $G^{2}\delta^{-1}\epsilon^{-3}+\sigma^{\frac{\p}{\p-1}}\delta^{-1}\epsilon^{-\frac{2\p-1}{\p-1}}$
for finding a $(\delta,\epsilon)$-stationary point without gradient
clipping. In addition, we also establish the first lower bound for
nonsmooth nonconvex optimization under heavy tails, matching our sample
complexity in the high accuracy and noisy regime (i.e., $\epsilon$
is small enough with $\sigma>0$). These two results together provide
a nearly complete characterization of the complexity of finding $(\delta,\epsilon)$-stationary
points in the heavy-tailed setting. Moreover, we give the first convergence
result when the problem-dependent parameters (like $G$, $\sigma$,
and $\p$) are unknown in advance, resolving a question asked by \citep{pmlr-v235-liu24bo}.
\end{itemize}
Furthermore, we explore broader settings. For example, when each $\ell_{t}$
is $H$-smooth, we present new regrets that extend the classical $L^{\star}$
bounds \citep{orabona2019modern,NIPS2010_76cf99d3} to the heavy-tailed
noise case. As an important implication, we show that $\SGD$ converges
at a rate of $HD^{2}/T+\sigma D/T^{1-1/\p}$ for smooth convex optimization
with heavy tails (even for the last iterate). Finally, we extend our
ideas to optimistic algorithms to address various cases simultaneously
and employ optimistic algorithms to give the first provable result
for H\"{o}lder smooth nonconvex optimization under heavy tails, where
the problem-dependent parameters can be either known or unknown.

\subsection{Discussion on \citep{NEURIPS2022_349956de}}

As noted, \citep{NEURIPS2022_349956de} is the only work for OCO with
heavy tails, as far as we know. There are two major discrepancies
between them and us. First, they consider the case where the feasible
set $\X$ is unbounded and aim to establish a parameter-free regret
bound, i.e., the regret bound has a linear dependency on $\left\Vert \bx\right\Vert $
(up to an extra $\mathrm{polylog}\left\Vert \bx\right\Vert $) for
any competitor $\bx\in\X$. Second, they focus on high-probability
rather than in-expectation analysis. As such, their regret is in the
form of $\reg_{T}^{\A}(\bx)\lesssim(G+\sigma)\left\Vert \bx\right\Vert T^{1/\p},\forall\bx\in\X$
(up to extra polylogarithmic factors) with high probability. Without
a doubt, their setting is harder than ours implying their bound is
stronger as it can convert to an in-expectation regret $\E\left[\reg_{T}^{\A}(\bx)\right]\lesssim(G+\sigma)DT^{1/\p}$
for any bounded domain $\X$ with a diameter $D$. 

We emphasize that the motivation behind \citep{NEURIPS2022_349956de}
differs heavily from ours. They aim to solve heavy-tailed OCO with
a new proposed method that contains many nontrivial technical tricks,
including gradient clipping, artificially added regularization, and
solving the additional fixed-point equation. However, their result
cannot reflect why the existing simple OCO algorithms like $\OGD$
work in practice under heavy-tailed noise. In contrast, our goal is
to examine whether, when, and how the classical OCO algorithms work
under heavy tails, thereby filling the missing piece in the literature.

Moreover, we would like to mention two limitations of \citep{NEURIPS2022_349956de}.
First, though the $T^{1/\p}$ regret seems tight as it matches the
lower bound \citep{nemirovskij1983problem,5394945,pmlr-v178-vural22a},
this may not be the best, since an optimal bound should recover the
standard $\sqrt{T}$ regret in the deterministic case (i.e., $\sigma=0$),
as one can imagine. This suggests that their bound is not entirely
optimal. Second, we remark that they require knowing both problem-dependent
parameters $G$, $\sigma$, $\p$ and time horizon $T$ in the algorithm,
which may be hard to satisfy in the online setting. In comparison,
our regret bound $GD\sqrt{T}+\sigma DT^{1/\p}$ is fully optimal in
all parameters\footnote{The optimality is justified by a matching lower bound $GD\sqrt{T}+\sigma DT^{1/\p}$
that can be obtained by combining Theorem 5.1 of \citep{orabona2019modern}
and Section 5.3.1 of \citep{nemirovskij1983problem}.}. Importantly, $\Ada$ can achieve it while oblivious to the problem
information.

\subsection{Discussion on High-Probability Bounds}

One may notice that all of our new results are measured in expectation
and naturally wonder whether high-probability bounds (i.e., polylogarithmic
dependency on the failure probability) are achievable in the same
setting of this work. Though we cannot fully address this question
at this time, some preliminary discussions are provided in Appendix
\ref{sec:hardness}, including a negative result for $\OGD/\SGD$.

\section{Preliminary}

\textbf{Notation. }$\N$ denotes the set of natural numbers (excluding
$0$). $\left[T\right]\triangleq\left\{ 1,\mydots,T\right\} ,\forall T\in\N$.
$a\land b\defeq\min\left\{ a,b\right\} $ and $a\lor b\defeq\max\left\{ a,b\right\} $.
We write $a\lesssim b$ (resp., $a\gtrsim b$) if $a\leq Cb$ (resp.,
$a\geq Cb$) for a universal constant $C>0$. $\left\lfloor \cdot\right\rfloor $
and $\left\lceil \cdot\right\rceil $ respectively represent the floor
and ceiling functions. $\left\langle \cdot,\cdot\right\rangle $ denotes
the Euclidean inner product and $\left\Vert \cdot\right\Vert \defeq\sqrt{\left\langle \cdot,\cdot\right\rangle }$
is the standard $2$-norm. Given $\bx\in\R^{d}$ and $D>0$, $\B^{d}(\bx,D)$
is the Euclidean ball in $\R^{d}$ centered at $\bx$ with a radius
$D$. In the case $\bx=\bzero$, we use the shorthand $\B^{d}(D)$.
Given $A\subseteq\R^{d}$, $\interior A$ stands for the interior
points of $A$. For nonempty closed convex $A\subseteq\R^{d}$, $\Pi_{A}$
is the Euclidean projection operator onto $A$. For a convex function
$f$, $\partial f(\bx)$ denotes its subgradient set at $\bx$.\textbf{}
\begin{rem}
\label{rem:SMD}We choose the Euclidean norm only for simplicity.
Extending the results of this work to any general norm by Online Mirror
Descent ($\OMD$) \citep{BECK2003167,nemirovskij1983problem} via
the Bregman divergence is straightforward.
\end{rem}
This work studies OCO in the context of Assumption \ref{assu:OCO}.
\begin{assumption}
\label{assu:OCO}We consider the following series of assumptions:
\begin{itemize}
\item $\X\subset\R^{d}$ is a nonempty closed convex set bounded by $D$,
i.e., $\sup_{\bx,\by\in\X}\left\Vert \bx-\by\right\Vert \leq D$.
\item $\ell_{t}:\X\to\R$ is closed convex for all $t\in\left[T\right]$.
\item $\ell_{t}$ is $G$-Lipschitz on $\X$, i.e., $\left\Vert \nabla\ell_{t}(\bx)\right\Vert \leq G,\forall\bx\in\X,\nabla\ell_{t}(\bx)\in\partial\ell_{t}(\bx)$,
for all $t\in\left[T\right]$.
\item Given a point $\bx_{t}\in\X$ at the $t$-th iteration, one can query
$\bg_{t}\in\R^{d}$ satisfying $\nabla\ell_{t}(\bx_{t})\defeq\E\left[\bg_{t}\mid\F_{t-1}\right]\in\partial\ell_{t}(\bx_{t})$
and $\E\left[\left\Vert \err_{t}\right\Vert ^{\p}\right]\leq\sigma^{\p}$
for some $\p\in\left(1,2\right]$ and $\sigma\geq0$, where $\F_{t}\defeq\sigma(\bg_{1},\mydots,\bg_{t})$
denotes the natural filtration and $\err_{t}\defeq\bg_{t}-\nabla\ell_{t}(\bx_{t})$
is the stochastic noise.
\end{itemize}
\end{assumption}
\begin{rem}
$D$ is recognized as known, like ubiquitously assumed in the OCO
literature. Moreover, $\bx_{t}$ denotes the decision/output of the
online learning algorithm by default.
\end{rem}
\begin{rem}
The idea presented in this work can be extended to the composite
setting (i.e., $\ell_{t}+\psi_{t}$, where $\psi_{t}$ is convex and
assumed to be known at time $t-1$) via the proximal update, provably.
\end{rem}
In Assumption \ref{assu:OCO}, the first three points are standard,
and the fourth is the heavy-tailed noise assumption. In particular,
$\p=2$ recovers the standard finite variance condition.

\section{Old Algorithms under Heavy Tails\label{sec:OCO}}

In this section, we revisit three classical algorithms for OCO: $\OGD$,
$\DA$, and $\Ada$, whose regret bounds are well-studied in the finite
variance case but remain unknown under heavy-tailed noise.

The basic idea of proving these algorithms work under heavy tails
is to leverage the boundedness property of $\X$. We will describe
it in more detail using $\OGD$ as an illustrated example. The analysis
of $\DA$ follows a similar way at a high level, but differs in some
details. However, though $\Ada$ can be viewed as $\OGD$ with an
adaptive stepsize, the way to utilize the boundedness property is
entirely different. All formal proofs are deferred to the appendix
due to space limitations.

\subsection{New Regret for Online Gradient Descent}

\begin{algorithm}[H]
\caption{\label{alg:OGD}Online Gradient Descent ($\protect\OGD$) \citep{zinkevich2003online}}

\textbf{Input:} initial point $\bx_{1}\in\X$, stepsize $\eta_{t}>0$

\textbf{for} $t=1$ \textbf{to} $T$ \textbf{do}

$\quad$$\bx_{t+1}=\Pi_{\X}(\bx_{t}-\eta_{t}\bg_{t})$

\textbf{end for}
\end{algorithm}

We begin from arguably the most basic algorithm for OCO, Online Gradient
Descent ($\OGD$).

\textbf{A well known analysis. }The regret bound of $\OGD$ has been
extensively studied \citep{OPT-013,orabona2019modern,MAL-018}. The
most well known analysis is perhaps the following one: for any $\bx\in\X$,
there is
\[
\left\Vert \bx_{t+1}-\bx\right\Vert ^{2}=\left\Vert \Pi_{\X}(\bx_{t}-\eta_{t}\bg_{t})-\Pi_{\X}(\bx)\right\Vert ^{2}\leq\left\Vert \bx_{t}-\eta_{t}\bg_{t}-\bx\right\Vert ^{2},
\]
where the inequality holds by the nonexpansive property of $\Pi_{\X}$.
Expanding both sides and rearranging terms yield that
\begin{equation}
\left\langle \bg_{t},\bx_{t}-\bx\right\rangle \leq\frac{\left\Vert \bx_{t}-\bx\right\Vert ^{2}-\left\Vert \bx_{t+1}-\bx\right\Vert ^{2}}{2\eta_{t}}+\frac{\eta_{t}\left\Vert \bg_{t}\right\Vert ^{2}}{2}.\label{eq:OGD-1}
\end{equation}
If $\bg_{t}$ admits a finite variance, i.e., $\p=2$ in Assumption
\ref{assu:OCO}, taking expectations on both sides, then following
a standard analysis for $\eta_{t}=\frac{D}{(G+\sigma)\sqrt{t}}$ (or
$\eta_{t}=\frac{D}{(G+\sigma)\sqrt{T}}$ if $T$ is known) gives the
regret
\[
\E\left[\reg_{T}^{\OGD}(\bx)\right]\lesssim(G+\sigma)D\sqrt{T},\forall\bx\in\X.
\]
However, the step of taking expectations on the R.H.S. of (\ref{eq:OGD-1})
crucially relies on the finite variance condition of $\bg_{t}$. Therefore,
one may naturally think $\OGD$ would not guarantee a finite regret
if $\p<2$.

\textbf{A less well known analysis}{\bfseries\footnote{To clarify, the phrase ``less well known'' is compared to the first
one. This analysis itself is also famous in the literature. For example,
see Lemma 6.10 of \citep{orabona2019modern} and Lemma 3.1 of \citep{lan2020first}.}}\textbf{. }As discussed, the failure of the above proof under heavy-tailed
noise is due to (\ref{eq:OGD-1}). Therefore, if a tighter inequality
than (\ref{eq:OGD-1}) exists, then it might be possible to show that
$\OGD$ still works for $\p<2$. However, does it exist?

Actually, there is another less well known analysis to produce a better
inequality than (\ref{eq:OGD-1}). That is, first showing for any
$\bx\in\X$, by the optimality condition of the update rule,
\[
\left\langle \bg_{t},\bx_{t+1}-\bx\right\rangle \leq\frac{\left\langle \bx_{t}-\bx_{t+1},\bx_{t+1}-\bx\right\rangle }{\eta_{t}}=\frac{\left\Vert \bx_{t}-\bx\right\Vert ^{2}-\left\Vert \bx_{t+1}-\bx\right\Vert ^{2}-\left\Vert \bx_{t}-\bx_{t+1}\right\Vert ^{2}}{2\eta_{t}},
\]
and then obtaining
\begin{equation}
\left\langle \bg_{t},\bx_{t}-\bx\right\rangle \leq\frac{\left\Vert \bx_{t}-\bx\right\Vert ^{2}-\left\Vert \bx_{t+1}-\bx\right\Vert ^{2}}{2\eta_{t}}+\left\langle \bg_{t},\bx_{t}-\bx_{t+1}\right\rangle -\frac{\left\Vert \bx_{t}-\bx_{t+1}\right\Vert ^{2}}{2\eta_{t}}.\label{eq:OGD-2}
\end{equation}
Note that (\ref{eq:OGD-2}) is tighter than (\ref{eq:OGD-1}) as $\left\langle \bg_{t},\bx_{t}-\bx_{t+1}\right\rangle \leq\left\Vert \bg_{t}\right\Vert \left\Vert \bx_{t}-\bx_{t+1}\right\Vert \leq\frac{\eta_{t}\left\Vert \bg_{t}\right\Vert ^{2}}{2}+\frac{\left\Vert \bx_{t}-\bx_{t+1}\right\Vert ^{2}}{2\eta_{t}}$,
where the first step is due to Cauchy-Schwarz inequality and the second
one is by AM-GM inequality.

\textbf{Handle $\p<2$ in a simple way.} Though we have tightened
(\ref{eq:OGD-1}) into (\ref{eq:OGD-2}), can inequality (\ref{eq:OGD-2})
help to overcome heavy tails? The answer is surprisingly positive,
and our solution is fairly simple. Instead of directly applying AM-GM
inequality in the second step, we recall $\bg_{t}=\nabla\ell_{t}(\bx_{t})+\err_{t}$
and use triangle inequality to obtain
\begin{equation}
\left\langle \bg_{t},\bx_{t}-\bx_{t+1}\right\rangle \leq\left\Vert \bg_{t}\right\Vert \left\Vert \bx_{t}-\bx_{t+1}\right\Vert \leq\left(\left\Vert \nabla\ell_{t}(\bx_{t})\right\Vert +\left\Vert \err_{t}\right\Vert \right)\left\Vert \bx_{t}-\bx_{t+1}\right\Vert .\label{eq:OGD-3}
\end{equation}
On the one hand, by $\left\Vert \nabla\ell_{t}(\bx_{t})\right\Vert \leq G$
and AM-GM inequality, there is
\begin{equation}
\left\Vert \nabla\ell_{t}(\bx_{t})\right\Vert \left\Vert \bx_{t}-\bx_{t+1}\right\Vert \leq G\left\Vert \bx_{t}-\bx_{t+1}\right\Vert \leq\eta_{t}G^{2}+\frac{\left\Vert \bx_{t}-\bx_{t+1}\right\Vert ^{2}}{4\eta_{t}}.\label{eq:OGD-4}
\end{equation}
On the other hand, let $\p_{\star}\defeq\frac{\p}{\p-1}$ and $\C(\p)\defeq\frac{(4\p-4)^{\p-1}}{\p^{\p}}$,
we have
\begin{align}
\left\Vert \err_{t}\right\Vert \left\Vert \bx_{t}-\bx_{t+1}\right\Vert  & =\left(\frac{4\eta_{t}}{\p_{\star}}\right)^{\frac{1}{\p_{\star}}}\left\Vert \err_{t}\right\Vert \left\Vert \bx_{t}-\bx_{t+1}\right\Vert ^{1-\frac{2}{\p_{\star}}}\cdot\left(\frac{\p_{\star}\left\Vert \bx_{t}-\bx_{t+1}\right\Vert ^{2}}{4\eta_{t}}\right)^{\frac{1}{\p_{\star}}}\nonumber \\
 & \overset{(a)}{\leq}\frac{\left(\frac{4\eta_{t}}{\p_{\star}}\right)^{\frac{\p}{\p_{\star}}}\left\Vert \err_{t}\right\Vert ^{\p}\left\Vert \bx_{t}-\bx_{t+1}\right\Vert ^{\p-\frac{2\p}{\p_{\star}}}}{\p}+\frac{\left\Vert \bx_{t}-\bx_{t+1}\right\Vert ^{2}}{4\eta_{t}}\nonumber \\
 & \overset{(b)}{\leq}\C(\p)\eta_{t}^{\p-1}\left\Vert \err_{t}\right\Vert ^{\p}D^{2-\p}+\frac{\left\Vert \bx_{t}-\bx_{t+1}\right\Vert ^{2}}{4\eta_{t}},\label{eq:OGD-5}
\end{align}
where $(a)$ is by Young's inequality and $(b)$ is due to $\left\Vert \bx_{t}-\bx_{t+1}\right\Vert \leq D$,
$\p_{\star}=\frac{\p}{\p-1}$, and $\C(\p)=\frac{(4\p-4)^{\p-1}}{\p^{\p}}$.
Next, we plug (\ref{eq:OGD-4}) and (\ref{eq:OGD-5}) back into (\ref{eq:OGD-3}),
then combine with (\ref{eq:OGD-2}) to know
\begin{equation}
\left\langle \bg_{t},\bx_{t}-\bx\right\rangle \leq\frac{\left\Vert \bx_{t}-\bx\right\Vert ^{2}-\left\Vert \bx_{t+1}-\bx\right\Vert ^{2}}{2\eta_{t}}+\eta_{t}G^{2}+\C(\p)\eta_{t}^{\p-1}\left\Vert \err_{t}\right\Vert ^{\p}D^{2-\p}.\label{eq:OGD-6}
\end{equation}
Notably, the term $\left\Vert \err_{t}\right\Vert ^{\p}$ has a correct
exponent $\p$. Thus, we can safely take expectations on both sides.
Finally, a standard analysis yields the following Theorem \ref{thm:main-OGD}
(see Appendix \ref{sec:OGD} for a formal proof).
\begin{thm}
\label{thm:main-OGD}Under Assumption \ref{assu:OCO}, taking $\eta_{t}=\frac{D}{G\sqrt{t}}\land\frac{D}{\sigma t^{1/\p}}$
in $\OGD$ (Algorithm \ref{alg:OGD}), we have
\[
\E\left[\reg_{T}^{\OGD}(\bx)\right]\lesssim GD\sqrt{T}+\sigma DT^{1/\p},\forall\bx\in\X.
\]
\end{thm}
As far as we know, Theorem \ref{thm:main-OGD} is the first and only
provable result for $\OGD$ under heavy tails. Remarkably, it is not
only tight in $T$ \citep{nemirovskij1983problem,5394945,pmlr-v178-vural22a}
but also fully optimal in all parameters, in contrast to the bound
$(G+\sigma)DT^{1/\p}$ of \citep{NEURIPS2022_349956de}. This reveals
that OCO with heavy tails can be optimally solved as effectively as
the finite variance case once the domain is bounded, a classical condition
adopted in many existing works.

\textbf{Strongly convex functions.} We highlight that the above idea
can also be applied to Online Strongly Convex Optimization and leads
to a sublinear regret $T^{2-\p}$ better than $T^{1/\p}$. This extension
can be found in Appendix \ref{sec:OGD}.

\subsection{New Regret for Dual Averaging}

\begin{algorithm}[H]
\caption{\label{alg:DA}Dual Averaging ($\protect\DA$) \citep{nesterov2009primal,NIPS2009_7cce53cf}}

\textbf{Input:} initial point $\bx_{1}\in\X$, stepsize $\eta_{t}>0$

\textbf{for} $t=1$ \textbf{to} $T$ \textbf{do}

$\quad$$\bx_{t+1}=\Pi_{\X}(\bx_{1}-\eta_{t}\sum_{s=1}^{t}\bg_{s})$

\textbf{end for}
\end{algorithm}

\begin{rem}
It is known that $\DA$ is a special realization of the more general
Follow-the-Regularized-Leader $(\FTRL)$ framework \citep{pmlr-v15-mcmahan11b}.
To keep the work concise, we focus only on $\DA$. The key idea for
proving Theorem \ref{thm:main-DA} can be directly extended to show
new regret for $\FTRL$ under heavy-tailed noise.
\end{rem}
We turn our attention to the second candidate, the Dual Averaging
($\DA$) algorithm, which is given in Algorithm \ref{alg:DA}. Though
$\DA$ coincides with $\OGD$ when $\X=\R^{d}$ and $\eta_{t}=\eta$,
these two methods in general are not equivalent and can have significant
performance differences in practice. Therefore, it is also important
to understand $\DA$ under heavy tails.

Despite the proof strategies for $\OGD$ and $\DA$ are in different
flavors (even for $\p=2$), the basic idea presented before for $\OGD$
still works here, i.e., apply the boundedness property of $\X$ to
make the term $\left\Vert \err_{t}\right\Vert $ have a correct exponent.
Armed with this thought, we can prove the following new regret bound
for $\DA$ under heavy-tailed noise. We refer the reader to Appendix
\ref{sec:DA} for its proof.
\begin{thm}
\label{thm:main-DA}Under Assumption \ref{assu:OCO}, taking $\eta_{t}=\frac{D}{G\sqrt{t}}\land\frac{D}{\sigma t^{1/\p}}$
in $\DA$ (Algorithm \ref{alg:DA}), we have
\[
\E\left[\reg_{T}^{\DA}(\bx)\right]\lesssim GD\sqrt{T}+\sigma DT^{1/\p},\forall\bx\in\X.
\]
\end{thm}
As far as we know, Theorem \ref{thm:main-DA} is the first provable
and optimal regret for $\DA$ under heavy tails. It guarantees the
same tight bound as in Theorem \ref{thm:main-OGD} up to different
constants.

\subsection{New Regret for $\protect\Ada$}

\begin{algorithm}[H]
\caption{\label{alg:AdaGrad}$\protect\Ada$ \citep{JMLR:v12:duchi11a,McMahanS10}}

\textbf{Input:} initial point $\bx_{1}\in\X$, stepsize $\eta>0$

\textbf{for} $t=1$ \textbf{to} $T$ \textbf{do}

$\quad$$\eta_{t}=\eta V_{t}^{-1/2}$ where $V_{t}=\sum_{s=1}^{t}\left\Vert \bg_{s}\right\Vert ^{2}$

$\quad$$\bx_{t+1}=\Pi_{\X}(\bx_{t}-\eta_{t}\bg_{t})$

\textbf{end for}
\end{algorithm}

\begin{rem}
Algorithm \ref{alg:AdaGrad} is also named $\mathsf{AdaGrad\text{-}Norm}$
(e.g., \citep{pmlr-v97-ward19a}). We simply call it $\Ada$. It is
straightforward to generalize Theorem \ref{thm:main-AdaGrad} below
to the per-coordinate update version.
\end{rem}
Although Theorems \ref{thm:main-OGD} and \ref{thm:main-DA} are optimal,
they both suffer from an undesired point. That is, the stepsize $\eta_{t}=\frac{D}{G\sqrt{t}}\land\frac{D}{\sigma t^{1/\p}}$
requires knowing all problem-dependent parameters. However, it may
not be easy to obtain them in an online setting. Especially, it heavily
depends on the prior information about the tail index $\p$, which
is hard to know (even approximately) in advance. In other words, they
both lack the adaptive property to an unknown environment.

To handle this issue, we consider $\Ada$, a classical adaptive algorithm
for OCO. As can be seen, $\Ada$ is just $\OGD$ with an adaptive
stepsize. However, it is this adaptive stepsize that can help us to
overcome the above undesired point.
\begin{thm}
\label{thm:main-AdaGrad}Under Assumption \ref{assu:OCO}, taking
$\eta=D/\sqrt{2}$ in $\Ada$ (Algorithm \ref{alg:AdaGrad}), we have
\[
\E\left[\reg_{T}^{\Ada}(\bx)\right]\lesssim GD\sqrt{T}+\sigma DT^{1/\p},\forall\bx\in\X.
\]
\end{thm}
\begin{rem}
We also establish a similar result for $\DA$ with an adaptive stepsize.
See Theorem \ref{thm:DA-ada} in Appendix \ref{sec:DA} for details.
\end{rem}
Theorem \ref{thm:main-AdaGrad} provides the first regret bound for
$\Ada$ under heavy tails. Impressively, it is optimal even without
knowing any of $G$, $\sigma$, and $\p$. This surprising result
once again demonstrates the power of the adaptive method, indicating
it is robust to an unknown environment and even heavy-tailed noise,
which may partially explain the favorable performance of many adaptive
optimizers designed based on $\Ada$ like $\mathsf{RMSProp}$ \citep{tieleman2012lecture}
and $\mathsf{Adam}$ \citep{kingma2014adam}.

We point out that the key to establishing Theorem \ref{thm:main-AdaGrad}
differs from the idea used before for $\OGD$ and $\DA$. Actually,
Theorem \ref{thm:main-AdaGrad} can be obtained in an embarrassingly
simple way. It is known that $\Ada$ with $\eta=D/\sqrt{2}$ on a
bounded domain guarantees the following path-wise regret 
\begin{equation}
\sum_{t=1}^{T}\left\langle \bg_{t},\bx_{t}-\bx\right\rangle \lesssim D\sqrt{\sum_{t=1}^{T}\left\Vert \bg_{t}\right\Vert ^{2}}.\label{eq:AdaGrad-path}
\end{equation}
Observe that $\sqrt{\sum_{t=1}^{T}\left\Vert \bg_{t}\right\Vert ^{2}}\lesssim\sqrt{\sum_{t=1}^{T}\left\Vert \nabla\ell_{t}(\bx_{t})\right\Vert ^{2}}+\sqrt{\sum_{t=1}^{T}\left\Vert \err_{t}\right\Vert ^{2}}\leq G\sqrt{T}+\left(\sum_{t=1}^{T}\left\Vert \err_{t}\right\Vert ^{\p}\right)^{\frac{1}{\p}}$,
where the last step is due to $\left\Vert \cdot\right\Vert _{2}\leq\left\Vert \cdot\right\Vert _{\p}$
for any $\p\in\left[1,2\right]$. After taking expectations on both
sides of (\ref{eq:AdaGrad-path}) and applying H\"{o}lder's inequality
to obtain $\E\left[\left(\sum_{t=1}^{T}\left\Vert \err_{t}\right\Vert ^{\p}\right)^{\frac{1}{\p}}\right]\leq\left(\sum_{t=1}^{T}\E\left[\left\Vert \err_{t}\right\Vert ^{\p}\right]\right)^{\frac{1}{\p}}\leq\sigma T^{\frac{1}{\p}}$,
we conclude Theorem \ref{thm:main-AdaGrad}. To make the work self-consistent,
we produce the formal proof of Theorem \ref{thm:main-AdaGrad} in
Appendix \ref{sec:AdaGrad}.

\section{Applications\label{sec:applications}}

We provide some applications based on the new regret bounds established
in Section \ref{sec:OCO}. The basic problem we study is optimizing
a single objective $F$, which could be either convex or nonconvex.

\subsection{Nonsmooth Convex Optimization}

In this section, we consider nonsmooth convex optimization with heavy
tails.

\textbf{Convergence of the average iterate.} First, we focus on the
average-iterate convergence. By the classical online-to-batch conversion
\citep{1327806}, the following corollary holds.
\begin{cor}
\label{cor:main-cvx-avg}Under Assumption \ref{assu:OCO} for $\ell_{t}(\bx)=\left\langle \nabla F(\bx_{t}),\bx\right\rangle $
and let $\bar{\bx}_{T}\defeq\frac{1}{T}\sum_{t=1}^{T}\bx_{t}$, for
any $\A\in\left\{ \OGD,\DA,\Ada\right\} $, we have
\[
\E\left[F(\bar{\bx}_{T})-F(\bx)\right]\leq\frac{\E\left[\reg_{T}^{\A}(\bx)\right]}{T}\lesssim\frac{GD}{\sqrt{T}}+\frac{\sigma D}{T^{1-\frac{1}{\p}}},\forall\bx\in\X.
\]
\end{cor}
\begin{proof}
By convexity, $F(\bar{\bx}_{T})-F(\bx)\leq\frac{\sum_{t=1}^{T}F(\bx_{t})-F(\bx)}{T}\leq\frac{\reg_{T}^{\A}(\bx)}{T}$
is valid for any OCO algorithm $\A$. We conclude from invoking Theorems
\ref{thm:main-OGD}, \ref{thm:main-DA} and \ref{thm:main-AdaGrad}.
\end{proof}

To the best of our knowledge, Corollary \ref{cor:main-cvx-avg} gives
the first and optimal convergence rate for these three algorithms
in stochastic optimization with heavy tails. Especially, it implies
that once the domain is bounded, the widely implemented $\SGD$ algorithm
provably converges under heavy-tailed noise without any algorithmic
change considered in many prior works, e.g., gradient clipping \citep{liu2023stochasticV2,NEURIPS2023_4c454d34}.

We are only aware of two works \citep{liu2024revisiting,pmlr-v178-vural22a}
based on Stochastic Mirror Descent ($\mathsf{SMD}$) \citep{nemirovskij1983problem}
that gave convergence results without clipping. However, they share
a common shortcoming, i.e., their bounds are both in the form of $(G+\sigma)D/T^{1-1/\p}$,
which cannot recover the optimal rate $GD/\sqrt{T}$ when $\sigma=0$.
\begin{rem}
As mentioned in Remark \ref{rem:SMD}, our analysis also provably
extends to $\OMD$, which degenerates to $\mathsf{SMD}$ for stochastic
optimization. However, we highlight a major difference between the
existing works mentioned above and ours: the condition on the mirror
map. Concretely, the mirror map in \citep{liu2024revisiting,pmlr-v178-vural22a}
is required to be $\frac{\p}{\p-1}$-uniformly convex w.r.t. the studied
norm. In contrast, as one can check, our proof technique only needs
the mirror map to be $1$-strongly convex w.r.t. the studied norm
regardless of the value of $\p$. We emphasize that this difference
is already significant for the standard $2$-norm considered in the
work. As explicitly discussed in Section 3.1 of \citep{pmlr-v178-vural22a},
their framework cannot recover standard $\SGD$ when $\p\neq2$, since
their mirror map is chosen to be proportional to $\left\Vert \cdot\right\Vert ^{\frac{\p}{\p-1}}$
to satisfy the $\frac{\p}{\p-1}$-uniformly convex requirement. In
comparison, for any $\p\in\left(1,2\right]$, $\OMD/\mathsf{SMD}$
with the $1$-strongly convex mirror map $\frac{1}{2}\left\Vert \cdot\right\Vert ^{2}$
exactly corresponds to the $\OGD/\SGD$ algorithm.
\end{rem}
Lastly, we highlight that for $\A=\Ada$, Corollary \ref{cor:main-cvx-avg}
is not only optimal but also adaptive to the tail index $\p$. As
far as we know, no result has achieved this property before. This
once again evidences the benefit of adaptive gradient methods.

\textbf{Convergence of the last iterate.} Next, we consider the more
challenging last-iterate convergence, which has a long history in
stochastic optimization and fruitful results in the case of $\p=2$
(see, e.g., \citep{pmlr-v99-harvey19a,doi:10.1137/19M128908X,orabona2020blog,pmlr-v28-shamir13,zhang2004solving}).
However, less is known about heavy-tailed problems. So far, only two
works \citep{liu2024revisiting,parletta2025improvedanalysisclippedstochastic}
have established the last-iterate convergence. The former is based
on $\mathsf{SMD}$, and the latter employs gradient clipping in $\SGD$.
Unfortunately, their rates are both in the suboptimal order $(G+\sigma)D/T^{1-1/\p}$.

We will provide an optimal last-iterate rate based on the following
lemma, which reduces the last-iterate convergence to an online learning
problem.
\begin{lem}[Theorem 1 of \citep{defazio2023optimal}]
\label{lem:main-cvx-last-equation}Suppose $\bx_{1},\mydots,\bx_{T}$
and $\by_{1},\mydots,\by_{T}$ are two sequences of vectors satisfying
$\bx_{t}\in\X,$ $\bx_{1}=\by_{1}$ and 
\begin{equation}
\by_{t+1}=\by_{t}+\frac{T-t}{T}\left(\bx_{t+1}-\bx_{t}\right).\label{eq:last-recursion}
\end{equation}
Given a convex function $F(\bx)$, let $\ell_{t}(\bx)=\left\langle \nabla F(\by_{t}),\bx\right\rangle $.
Then for any online learner $\A$, we have
\[
F(\by_{T})-F(\bx)\leq\frac{\reg_{T}^{\A}(\bx)}{T},\forall\bx\in\X.
\]
\end{lem}
We emphasize that the stochastic gradient $\bg_{t}$ received by $\A$
is an estimate of $\nabla F(\by_{t})$ instead of $\nabla F(\bx_{t})$.
This flexibility is due to the generality of the OCO framework. Moreover,
for $\OGD$, suppose there is no projection step, then (\ref{eq:last-recursion})
is equivalent to $\by_{t+1}=\by_{t}-\frac{T-t}{T}\eta_{t}\bg_{t}$,
which can be viewed as $\SGD$ with a stepsize $\frac{T-t}{T}\eta_{t}$.
For proof of Lemma \ref{lem:main-cvx-last-equation}, we refer the
interested reader to \citep{defazio2023optimal}.
\begin{cor}
\label{cor:main-cvx-last}Under Assumption \ref{assu:OCO} for $\ell_{t}(\bx)=\left\langle \nabla F(\by_{t}),\bx\right\rangle $,
where $\by_{t}$ satisfies (\ref{eq:last-recursion}), for any $\A\in\left\{ \OGD,\DA,\Ada\right\} $,
we have
\[
\E\left[F(\by_{T})-F(\bx)\right]\leq\frac{\E\left[\reg_{T}^{\A}(\bx)\right]}{T}\lesssim\frac{GD}{\sqrt{T}}+\frac{\sigma D}{T^{1-\frac{1}{\p}}},\forall\bx\in\X.
\]
\end{cor}
\begin{proof}
Combine Lemma \ref{lem:main-cvx-last-equation} and Theorems \ref{thm:main-OGD},
\ref{thm:main-DA} and \ref{thm:main-AdaGrad} to conclude.
\end{proof}

As far as we know, Corollary \ref{cor:main-cvx-last} is the first
optimal last-iterate convergence rate for stochastic convex optimization
with heavy tails, closing the gap in existing works.

One may notice that $\by_{t}$ itself is not the decision made by
the online learner and naturally may ask whether $\bx_{t}$ ensures
the last-iterate convergence if we simply pick $\ell_{t}=F$. The
answer turns out to be positive at least for $\OGD$ (which is equivalent
to $\SGD$ now). However, to prove this result, we rely on a technique
specialized to stochastic optimization recently developed by \citep{liu2024revisiting,doi:10.1137/24M1717762}.
To not diverge from the topic of OCO, we defer the last-iterate convergence
of $\OGD$ (i.e., $\E\left[F(\bx_{T})\right]$) to Appendix \ref{sec:cvx},
in which Theorem \ref{thm:OGD-last-core} gives a general result for
any stepsize $\eta_{t}$ (even for a broader function class) and Corollary
\ref{cor:OGD-cvx-last} provides a last-iterate rate similar to Corollary
\ref{cor:main-cvx-last} (up to an extra logarithmic factor) under
the same stepsize $\eta_{t}=\frac{D}{G\sqrt{t}}\land\frac{D}{\sigma t^{1/\p}}$
as in Theorem \ref{thm:main-OGD}. If $T$ is assumed to be known,
Corollary \ref{cor:OGD-cvx-last} also shows how to set the stepsize
in $\OGD$ to converge in the optimal rate $\frac{GD}{\sqrt{T}}+\frac{\sigma D}{T^{1-\frac{1}{\p}}}$.
\begin{rem}
If strong convexity further holds, Corollary \ref{cor:OGD-str-last}
gives a last-iterate rate in the order $\frac{\log T}{T^{\p-1}}$.
\end{rem}

\subsection{Nonsmooth Nonconvex Optimization}

This section contains another application, nonsmooth nonconvex optimization
with heavy tails. Due to limited space, we will provide only the necessary
background. For more details, we refer the reader to \citep{NEURIPS2022_2c8d9636,pmlr-v195-jordan23a,JMLR:v23:21-1507,kornowski2022on,tian2024no,pmlr-v162-tian22a}
for recent progress. We start with a new set of conditions.
\begin{assumption}
\label{assu:ncvx}We consider the following series of assumptions:
\begin{itemize}
\item The objective $F$ is lower bounded by $F_{\star}\triangleq\inf_{\bx\in\R^{d}}F(\bx)\in\R$.
\item $F$ is differentiable and well-behaved, i.e., $F(\bx)-F(\by)=\int_{0}^{1}\left\langle \nabla F(\by+t(\bx-\by)),\bx-\by\right\rangle \d t$.
\item $F$ is $G$-Lipschitz on $\R^{d}$, i.e., $\left\Vert \nabla F(\bx)\right\Vert \leq G,\forall\bx\in\R^{d}$.
\item Given $\bz_{t}\in\R^{d}$ at the $t$-th iteration, one can query
$\bg_{t}\in\R^{d}$ satisfying $\E\left[\bg_{t}\mid\F_{t-1}\right]=\nabla F(\bz_{t})$
and $\E\left[\left\Vert \err_{t}\right\Vert ^{\p}\right]\leq\sigma^{\p}$
for some $\p\in\left(1,2\right]$ and $\sigma\geq0$, where $\F_{t}$
denotes the natural filtration and $\err_{t}\defeq\bg_{t}-\nabla F(\bz_{t})$
is the stochastic noise.
\end{itemize}
\end{assumption}
\begin{rem}
The second point is a mild regularity condition introduced by \citep{pmlr-v202-cutkosky23a}
and becomes standard in the literature \citep{NEURIPS2024_ac8ec9b4,pmlr-v235-liu24bo,pmlr-v235-zhang24k}.
See Definition 1 and Proposition 2 of \citep{pmlr-v202-cutkosky23a}
for more details. In the fourth point, we use the same notation $\bz_{t}$
as in the algorithm being studied later. In fact, it can be arbitrary.
\end{rem}
In nonsmooth nonconvex optimization, we aim to find a $(\delta,\epsilon)$-stationary
point \citep{pmlr-v119-zhang20p} (see the formal Definition \ref{def:delta-eps-stationary}
in Appendix \ref{sec:ncvx}). This goal can be reduced to finding
a point $\bx\in\R^{d}$ such that $\left\Vert \nabla F(\bx)\right\Vert _{\delta}\leq\epsilon$,
where $\left\Vert \nabla F(\bx)\right\Vert _{\delta}$ is a quantity
introduced by \citep{pmlr-v202-cutkosky23a} as follows.
\begin{defn}[Definition 5 of \citep{pmlr-v202-cutkosky23a}]
\label{def:delta-norm}Given a point $\bx\in\R^{d}$, a number $\delta>0$
and an almost-everywhere differentiable function $F$, define $\left\Vert \nabla F(\bx)\right\Vert _{\delta}\defeq\inf_{S\subset\B(\bx,\delta),\frac{1}{\left|S\right|}\sum_{\by\in S}\by=\bx}\left\Vert \frac{1}{\left|S\right|}\sum_{\by\in S}\nabla F(\by)\right\Vert .$
\end{defn}
The only existing sample complexity under Assumption \ref{assu:ncvx}
is $(G+\sigma)^{\frac{\p}{\p-1}}\delta^{-1}\epsilon^{-\frac{2\p-1}{\p-1}}$
in high probability \citep{pmlr-v235-liu24bo}, where we only report
the dominant term and hide the dependency on the failure probability.

However, on the theoretical side, their result cannot recover the
optimal bound $G^{2}\delta^{-1}\epsilon^{-3}$ \citep{pmlr-v202-cutkosky23a}
in the deterministic case. On the practical side, their method also
employs the gradient clipping step, which introduces a new clipping
parameter to tune. In fact, as stated in their Section 5, they observed
in experiments that their algorithm without the clipping operation
(exactly the algorithm we study next) still works under heavy tails.
In addition, in their Section 6, they also explicitly ask whether
the requirement to know $G$ and $\A$ can be removed.

As will be seen later, we can address these points with the new regret
bounds presented before.

\subsubsection{Online-to-Nonconvex Conversion under Heavy Tails}

\begin{algorithm}[H]
\caption{\label{alg:O2NC}Online-to-Nonconvex Conversion ($\protect\otnc$)
\citep{pmlr-v202-cutkosky23a}}

\textbf{Input:} initial point $\by_{0}\in\R^{d}$, $K\in\N$, $T\in\N$,
online learning algorithm $\A$.

\textbf{for} $n=1$ \textbf{to} $KT$ \textbf{do}

$\quad$Receive $\bx_{n}$ from $\A$

$\quad$$\by_{n}=\by_{n-1}+\bx_{n}$

$\quad$$\bz_{n}=\by_{n-1}+s_{n}\bx_{n}$ where $s_{n}\sim\uni\left[0,1\right]$
i.i.d.

$\quad$Query a stochastic gradient $\bg_{n}$ at $\bz_{n}$

$\quad$Send $\bg_{n}$ to $\A$

\textbf{end for}
\end{algorithm}

\begin{rem}
Note that $\otnc$ is a randomized algorithm. Therefore, the definition
of the natural filtration is adjusted to $\F_{n}\defeq\sigma(s_{1},\bg_{1},\mydots,s_{n},\bg_{n},s_{n+1})$
accordingly.
\end{rem}
We provide the Online-to-Nonconvex Conversion ($\otnc$) framework
in Algorithm \ref{alg:O2NC}, which serves as a meta algorithm. Roughly
speaking, Algorithm \ref{alg:O2NC} reduces a nonconvex optimization
problem to an OCO (in fact, OLO) problem, for which the $K$-shifting
regret (see (\ref{eq:ncvx-shift})) of the online learner $\A$ crucially
affects the final convergence rate. However, the existing Theorem
8 of \citep{pmlr-v202-cutkosky23a}, a general convergence result
for the above reduction, cannot directly apply to heavy-tailed noise,
since its proof relies on the finite variance condition on $\bg_{n}$
(see Appendix \ref{sec:ncvx} for more details).
\begin{thm}
\label{thm:main-ncvx-core}Under Assumption \ref{assu:ncvx} and
let $\bv_{k}\defeq-D\frac{\sum_{n=(k-1)T+1}^{kT}\nabla F(\bz_{n})}{\left\Vert \sum_{n=(k-1)T+1}^{kT}\nabla F(\bz_{n})\right\Vert },\forall k\in\left[K\right]$
for arbitrary $D>0$, then for any online learning algorithm $\A$
in $\otnc$ (Algorithm \ref{alg:O2NC}), we have
\[
\E\left[\sum_{k=1}^{K}\frac{1}{K}\left\Vert \frac{1}{T}\sum_{n=(k-1)T+1}^{kT}\nabla F(\bz_{n})\right\Vert \right]\lesssim\frac{F(\by_{0})-F_{\star}}{DKT}+\frac{\E\left[\reg_{T}^{\A}(\bv_{1},\mydots,\bv_{K})\right]}{DKT}+\frac{\sigma}{T^{1-\frac{1}{\p}}}.
\]
\end{thm}
$\reg_{T}^{\A}(\bv_{1},\mydots,\bv_{K})$ in Theorem \ref{thm:main-ncvx-core}
is called \textit{$K$-shifting regret} \citep{pmlr-v202-cutkosky23a},
defined as follows:
\begin{eqnarray}
\reg_{T}^{\A}\left(\bv_{1},\mydots,\bv_{K}\right)\defeq\sum_{k=1}^{K}\sum_{n=(k-1)T+1}^{kT}\ell_{n}(\bx_{n})-\ell_{n}(\bv_{k}) & \text{where} & \ell_{n}(\bx)\defeq\left\langle \bg_{n},\bx\right\rangle .\label{eq:ncvx-shift}
\end{eqnarray}
Theorem \ref{thm:main-ncvx-core} here provides a new and the first
theoretical guarantee for $\otnc$ under heavy tails. Especially,
it recovers Theorem 8 of \citep{pmlr-v202-cutkosky23a} when $\p=2$.
A remarkable point is that the $\otnc$ algorithm itself does not
need any information about $\p$. The proof of Theorem \ref{thm:main-ncvx-core}
can be found in Appendix \ref{sec:ncvx}.

\subsubsection{Convergence Rates}

Theorem \ref{thm:main-ncvx-core} enables us to apply the results
presented in Section \ref{sec:OCO}. Concretely, for $\X=\B^{d}(D)$
and any $\A\in\left\{ \OGD,\DA,\Ada\right\} $, if we reset the stepsize
in $\A$ after every $T$ iterations, there will be $\E\left[\reg_{T}^{\A}(\bv_{1},\mydots,\bv_{K})\right]\lesssim GDK\sqrt{T}+\sigma DKT^{1/\p}$
by our new regret bounds, since $\bv_{k}\in\X$. With a carefully
picked $D$, we obtain the following Theorem \ref{thm:main-ncvx-general}.
Its proof is deferred to Appendix \ref{sec:ncvx}.
\begin{thm}
\label{thm:main-ncvx-general}Under Assumption \ref{assu:ncvx} and
let $\Delta\defeq F(\by_{0})-F_{\star}$ and $\bar{\bz}_{k}\defeq\frac{1}{T}\sum_{n=(k-1)T+1}^{kT}\bz_{n},\forall k\in\left[K\right]$,
setting any $\A\in\left\{ \OGD,\DA,\Ada\right\} $ in $\otnc$ (Algorithm
\ref{alg:O2NC}) with a domain $\X=\B^{d}(D)$ for $D=\delta/T$ and
resetting the stepsize in $\A$ after every $T$ iterations, we have
\[
\E\left[\frac{1}{K}\sum_{k=1}^{K}\left\Vert \nabla F(\bar{\bz}_{k})\right\Vert _{\delta}\right]\lesssim\frac{\Delta}{\delta K}+\frac{G}{\sqrt{T}}+\frac{\sigma}{T^{1-\frac{1}{\p}}}.
\]
\end{thm}
Notably, this is the first time confirming that gradient clipping
is indeed unnecessary for the $\otnc$ framework, matching the experimental
observation of \citep{pmlr-v235-liu24bo}.
\begin{cor}
\label{cor:main-ncvx-dep}Under the same setting of Theorem \ref{thm:main-ncvx-general},
suppose we have $N\geq2$ stochastic gradient budgets, taking $K=\left\lfloor N/T\right\rfloor $
and $T=\left\lceil N/2\right\rceil \land\left(\left\lceil (\delta GN/\Delta)^{\frac{2}{3}}\right\rceil \lor\left\lceil (\delta\sigma N/\Delta)^{\frac{\p}{2\p-1}}\right\rceil \right)$,
we have
\[
\E\left[\frac{1}{K}\sum_{k=1}^{K}\left\Vert \nabla F(\bar{\bz}_{k})\right\Vert _{\delta}\right]\lesssim\frac{G}{\sqrt{N}}+\frac{\sigma}{N^{1-\frac{1}{\p}}}+\frac{\Delta}{\delta N}+\frac{G^{\frac{2}{3}}\Delta^{\frac{1}{3}}}{(\delta N)^{\frac{1}{3}}}+\frac{\sigma^{\frac{\p}{2\p-1}}\Delta^{\frac{\p-1}{2\p-1}}}{(\delta N)^{\frac{\p-1}{2\p-1}}}.
\]
\end{cor}
Corollary \ref{cor:main-ncvx-dep} is obtained by optimizing $K$
and $T$ in Theorem \ref{thm:main-ncvx-general}. It implies a sample
complexity of $G^{2}\delta^{-1}\epsilon^{-3}+\sigma^{\frac{\p}{\p-1}}\delta^{-1}\epsilon^{-\frac{2\p-1}{\p-1}}$
for finding a $(\delta,\epsilon)$-stationary point, improved over
the previous bound $(G+\sigma)^{\frac{\p}{\p-1}}\delta^{-1}\epsilon^{-\frac{2\p-1}{\p-1}}$
\citep{pmlr-v235-liu24bo}. Furthermore, leveraging the adaptive feature
of $\Ada$, Corollary \ref{cor:ncvx-free} in Appendix \ref{sec:ncvx}
shows how to set $K$ and $T$ without $G$, $\sigma$, and $\p$,
resulting in the first provable rate for $\otnc$ when no problem
information is known in advance, which solves the problem asked by
\citep{pmlr-v235-liu24bo}.

\subsubsection{Lower Bounds}

In this part, we provide the first lower bound for nonsmooth nonconvex
optimization under heavy tails in the following Theorem \ref{thm:main-ncvx-lb},
the proof of which follows the framework first established in \citep{arjevani2023lower}
and later developed by \citep{pmlr-v202-cutkosky23a} but with some
necessary (though minor) variation to make it compatible with heavy-tailed
noise.
\begin{thm}
\label{thm:main-ncvx-lb}For any given $\Delta>0$, $G>0$, $\p\in\left(1,2\right]$,
$\sigma\geq0$, $\delta>0$ and $0<\epsilon\lesssim\frac{\Delta}{\delta}\land\frac{G^{2}\delta}{\Delta}$,
there exists a dimension $d>0$ depending on the previous parameters
such that, for any randomized first-order algorithm (see Definition
\ref{def:randomized-algorithm} for a formal description), there exists
a $G$-Lipschitz differentiable function $F:\R^{d}\to\R$ satisfying
$F(\bzero)-F_{\star}\leq\Delta$ and a function $\bg:\R^{d}\times\left\{ 0,1\right\} \to\R$
satisfying $\E_{r}\left[\bg(\bx,r)\right]=\nabla F(\bx)$ and $\E_{r}\left[\left\Vert \bg(\bx,r)-\nabla F(\bx)\right\Vert ^{\p}\right]\leq\sigma^{\p}$
where $r$ follows a certain probability distribution over $\left\{ 0,1\right\} $
such that the algorithm requires $\gtrsim\Delta\delta^{-1}\epsilon^{-1}+\Delta\sigma^{\frac{\p}{\p-1}}\delta^{-1}\epsilon^{-\frac{2\p-1}{\p-1}}$
queries of $\bg$ to find a point $\bz$ such that $\E\left[\left\Vert \nabla F(\bz)\right\Vert _{\delta}\right]\leq\epsilon$.
\end{thm}
For small enough $\epsilon$ and $\sigma>0$, Theorem \ref{thm:main-ncvx-lb}
can be further simplified into a lower bound of $\Delta\sigma^{\frac{\p}{\p-1}}\delta^{-1}\epsilon^{-\frac{2\p-1}{\p-1}}$,
matching the leading term in the sample complexity derived from the
previous Corollary \ref{cor:main-ncvx-dep}. Therefore, Theorem \ref{thm:main-ncvx-lb}
suggests that our Corollary \ref{cor:main-ncvx-dep} is tight in the
high accuracy and noisy regime. As such, Corollary \ref{cor:main-ncvx-dep}
and Theorem \ref{thm:main-ncvx-lb} together provide a nearly complete
characterization of the complexity of finding $(\delta,\epsilon)$-stationary
points in the heavy-tailed setting.

However, for any general $\epsilon>0$ or the case $\sigma=0$, there
is still a gap between the upper and lower bounds. Closing this gap
could be an interesting direction for the future.

\section{Further Extensions\label{sec:extension}}

So far, we revisit three classical OCO algorithms, $\OGD$, $\DA$,
and $\Ada$, show that they all guarantee optimal regrets, and provide
some applications based on their new bounds. However, all the results
are restricted to the Lipschitz case, which is standard in the literature,
but sometimes inadequate to derive a better bound. For example, if
$\ell_{t}$ is smooth (i.e., the gradient of $\ell_{t}$ is Lipschitz),
then one could establish a bound depending on the cumulative competitor
loss, i.e., $\sum_{t=1}^{T}\ell_{t}(\bx)$.

In this section, we first show that our idea presented in Section
\ref{sec:OCO} can be directly extended to the smooth case, then discuss
what more we can do. Formally, we need the following condition.
\begin{condition}
\label{cond:smooth}$\ell_{t}$ is $H$-smooth on $\R^{d}$, i.e.,
$\left\Vert \nabla\ell_{t}(\bx)-\nabla\ell_{t}(\by)\right\Vert \leq H\left\Vert \bx-\by\right\Vert ,\forall\bx,\by\in\R^{d}$,
for all $t\in\left[T\right]$.
\end{condition}
\begin{rem}
Strictly speaking, Condition \ref{cond:smooth} may not be well-defined,
since Assumption \ref{assu:OCO} only requires $\ell_{t}$ to be defined
on $\X$. To avoid any further complications, this section recognizes
$\ell_{t}$ as a real-valued convex function defined on $\R^{d}$,
and $\X$ is the constraint set, i.e., the domain of the problem.
\end{rem}

\subsection{New Regrets for Old Algorithms under Smooth $\ell_{t}$ and Applications}
\begin{thm}
\label{thm:main-OCO-smooth}Under Assumption \ref{assu:OCO} (with
replacing the third point by Condition \ref{cond:smooth}) and additionally
assuming $\ell_{t}\geq0$ on $\R^{d}$:
\begin{itemize}
\item taking $\eta_{t}=\frac{1}{4H}\land\frac{\gamma D}{\sqrt{H}}\land\frac{D}{\sigma t^{1/\p}}$
for any $\gamma>0$ in $\A\in\left\{ \OGD,\DA\right\} $, we have
\[
\E\left[\reg_{T}^{\A}(\bx)\right]\lesssim HD^{2}+\sqrt{H}D\left(\frac{1}{\gamma}+\gamma\sum_{t=1}^{T}\ell_{t}(\bx)\right)+\sigma DT^{1/\p},\forall\bx\in\X.
\]
\item taking $\eta=D/\sqrt{2}$ in $\A=\Ada$, we have
\[
\E\left[\reg_{T}^{\A}(\bx)\right]\lesssim HD^{2}+D\sqrt{H\sum_{t=1}^{T}\ell_{t}(\bx)}+\sigma DT^{1/\p},\forall\bx\in\X.
\]
\end{itemize}
\end{thm}
Theorem \ref{thm:main-OCO-smooth} provides new regrets that extend
the classical $L^{\star}$ bounds \citep{orabona2019modern} to the
heavy-tailed noise case, under an additional nonnegative condition
$\ell_{t}\geq0$ that is widely used in the literature \citep{NIPS2010_76cf99d3}.
The full version without such a requirement is provided in Theorem
\ref{thm:ext-OCO-smooth} in Appendix \ref{sec:ext-details}.

As one can see, the optimal value of $\gamma$ for $\OGD$ and $\DA$
should be in the order of $1/\sqrt{\sum_{t=1}^{T}\ell_{t}(\bx)}$,
which is however not possible to take as the competitor $\bx$ is
not fixed here. In contrast, $\Ada$ contains the improved term $D\sqrt{H\sum_{t=1}^{T}\ell_{t}(\bx)}$,
which once again suggests the benefit of adaptive methods. More importantly,
the regret for $\Ada$ also indicates that it can be oblivious to
the problem class and adapt to the best-possible bound automatically,
even under heavy tails.

Theorem \ref{thm:main-OCO-smooth} could be particularly useful for
smooth convex optimization when a global minimizer is contained in
$\X$, as given in the following corollary.
\begin{cor}
\label{cor:main-cvx-smooth-avg}Under Assumption \ref{assu:OCO} (with
replacing the third point by Condition \ref{cond:smooth}) for $\ell_{t}(\bx)=F(\bx)$
and additionally assuming that $\X$ contains a global minimzer $\bx^{\star}$
of $F$ (i.e., $\bx^{\star}\in\mathrm{arginf}_{\bx\in\R^{d}}F(\bx)$),
let $\bar{\bx}_{T}\defeq\frac{1}{T}\sum_{t=1}^{T}\bx_{t}$, taking
$\eta_{t}=\frac{1}{4H}\land\frac{D}{\sigma t^{1/\p}}$ in $\A\in\left\{ \OGD,\DA\right\} $
or $\eta=D/\sqrt{2}$ in $\A=\Ada$, we have
\[
\E\left[F(\bar{\bx}_{T})-F(\bx^{\star})\right]\leq\frac{\E\left[\reg_{T}^{\A}(\bx^{\star})\right]}{T}\lesssim\frac{HD^{2}}{T}+\frac{\sigma D}{T^{1-\frac{1}{\p}}}.
\]
\end{cor}
\begin{proof}
Equivalently, we can redefine $\ell_{t}(\bx)=F(\bx)-F(\bx^{\star})$,
then $\ell_{t}$ is nonnegative on $\R^{d}$. By convexity, $F(\bar{\bx}_{T})-F(\bx^{\star})\leq\frac{\sum_{t=1}^{T}F(\bx_{t})-F(\bx^{\star})}{T}=\frac{\reg_{T}^{\A}(\bx^{\star})}{T}$
is valid for any OCO algorithm $\A$. We conclude from invoking Theorem
\ref{thm:main-OCO-smooth} with $\gamma=\frac{1}{4D\sqrt{H}}$ for
$\A\in\left\{ \OGD,\DA\right\} $ and $\eta=D/\sqrt{2}$ for $\A=\Ada$
and combining the fact $\ell_{t}(\bx^{\star})=0$.
\end{proof}

An important implication of Corollary \ref{cor:main-cvx-smooth-avg}
is the first to show that $\SGD$ can converge at the optimal rate
of $HD^{2}/T+\sigma D/T^{1-1/\p}$ for heavy-tailed smooth convex
optimization (though under an extra condition). In fact, a valid but
worse bound without assuming the existence of $\bx^{\star}\in\X$
is possible, which we defer to Corollary \ref{cor:ext-cvx-smooth-avg}
in Appendix \ref{sec:ext-details}.

\subsection{Optimistic Algorithms for Broader Cases}

As presented, we have successfully handled smooth OCO with heavy tails
for classical $\OGD$, $\DA$, and $\Ada$. In fact, our ideas can
also be applied to another famous family of methods known as optimistic
algorithms \citep{pmlr-v23-chiang12,pmlr-v76-joulani17a,pmlr-v30-Rakhlin13}
to deal with broader cases. Specifically, we will study an optimistic
version of $\Ada$ called $\OAda$ in Appendix \ref{sec:ext-details}
to address various cases simultaneously (see Theorems \ref{thm:OAda}
and \ref{thm:OAda-star}) and provide more applications for stochastic
optimization in Appendix \ref{sec:ext-applications}.

For example, in the special case of heavy-tailed H\"{o}lder smooth
optimization (i.e., $\left\Vert \nabla F(\bx)-\nabla F(\by)\right\Vert \leq H\left\Vert \bx-\by\right\Vert ^{\nu}$
for some $H>0$ and $\nu\in\left(0,1\right]$), Corollary \ref{cor:ext-cvx-avg}
provides a rate $\frac{HD^{1+\nu}}{T^{\frac{1+\nu}{2}}}+\frac{\sigma D}{T^{1-\frac{1}{\p}}}$
for convex problems, Corollaries \ref{cor:ext-ncvx-holder-dep} and
\ref{cor:ext-ncvx-holder-free} establish provable results for nonconvex
problems, where the problem-dependent parameters may be known or unknown.
All of these results are new to the best of our knowledge.

\section{Conclusion and Future Work\label{sec:conclusion}}

This paper shows that three classical OCO algorithms, $\OGD$, $\DA$,
and $\Ada$, can achieve the optimal in-expectation regret under heavy
tails without any algorithmic modification if the feasible set is
bounded, and provides some applications in stochastic optimization.
The main limitation of our work is that all the proof crucially relies
on the bounded domain assumption, which may not always be suitable
in practice. Finding a weaker sufficient condition, under which the
classical OCO algorithms work with heavy tails provably, is a direction
worth studying in the future.

\bibliographystyle{plainnat}
\bibliography{ref}

\clearpage

\appendix

\section{Hardness of High-Probability Bounds Even on Bounded Domains\label{sec:hardness}}

Without additional assumptions on noise or algorithmic changes to
the classical methods studied in the paper, even on a bounded domain,
we are inclined to believe that a high-probability bound is unlikely
(at least within the current analysis framework).

\textbf{An upper-bound perspective.} When deriving regret bounds for
$\OGD/\DA/\Ada$ in the current analysis, one needs to upper bound
the term $\sum_{t=1}^{T}\left\Vert \err_{t}\right\Vert ^{\p}$. Under
Assumption \ref{assu:OCO}, this can be easily done in expectation.
However, if one instead wants to obtain a high-probability bound with
a polylogarithmic dependency on the failure probability $\delta$,
this is unlikely to be achieved.

As an example, we consider $d=1$ for simplicity. Let $\err_{t}=e_{t}\xi_{t}$
where $e_{t}$ and $\xi_{t}$ are independent of each other and the
history, satisfying that $\Pr\left[e_{t}=\pm1\right]=\frac{1}{2}$
and $\Pr\left[\xi_{t}>z\right]=\left(\frac{z_{\star}}{z}\right)^{\p+a}\1\left[z\geq z_{\star}\right]+\1\left[z<z_{\star}\right]$,
where $a>0$ can be arbitrarily small and $z_{\star}>0$ can be any
number. In other words, $e_{t}$ follows the Rademacher distribution
and $\xi_{t}$ follows the Pareto distribution. As one can check,
there are $\E\left[\err_{t}\mid\F_{t-1}\right]=0$ and $\E\left[\left\Vert \err_{t}\right\Vert ^{\p}\right]<\infty$.
Crucially, we can find that, for $z\geq z_{\star}^{\p}$ and any $T\geq1$,
\[
\Pr\left[\sum_{t=1}^{T}\left\Vert \err_{t}\right\Vert ^{\p}>z\right]=\Pr\left[\sum_{t=1}^{T}\xi_{t}^{\p}>z\right]\geq\Pr\left[\xi_{1}>z^{\frac{1}{\p}}\right]=\frac{z_{\star}^{\p+a}}{z^{1+\frac{a}{\p}}},
\]
indicating that the tail of $\sum_{t=1}^{T}\left\Vert \err_{t}\right\Vert ^{\p}$
decays at least polynomially, which further suggests that a polylogarithmic
dependency on the failure probability $\delta$ (equivalently, an
exponentially decaying tail) is impossible.

\textbf{A lower-bound perspective. }We consider the special case of
OCO, stochastic convex optimization (i.e., $\ell_{t}=F$), and show
a negative result for $\OGD/\SGD$ when $d=1$ and $\p=2$.

We set $F(\bx)=\left|\bx\right|$ with $\X\defeq\left[-1,1\right]$
and $\bg_{t}=\bg(\bx_{t};\xi_{t})=\mathrm{sgn}(\bx_{t})\xi_{t}$ (with
$\mathrm{sgn}(0)=1$), where $\xi_{t}$ is independent of each other
and follows the Pareto distribution satisfying $\Pr\left[\xi_{t}>z\right]=\left(\frac{b-1}{bz}\right)^{b}\1\left[z\geq\frac{b-1}{b}\right]+\1\left[z<\frac{b-1}{b}\right]$,
in which $b\defeq2+2a$ and $a>0$ can be arbitrarily small. As one
can check, $\E\left[\bg_{t}\mid\F_{t-1}\right]=\nabla F(\bx_{t})$
due to $\E\left[\xi_{t}\right]=1$ and $\E\left[\left\Vert \err_{t}\right\Vert ^{2}\right]<\infty$
due to $\E\left[\xi_{t}^{2}\right]<\infty$.
\begin{prop}
Consider the setting described above, if $\OGD/\SGD$ with $\eta_{t}=\frac{\eta}{\sqrt{t}}$
(the standard time-varying stepsize) for $\eta\leq1$ guarantees a
high-probability rate $\Pr\left[F(\bx_{t})>\frac{\mathrm{polylog}\left(1/\delta\right)}{t^{c}}\right]\leq\delta,\forall\delta\in\left(0,1\right],\forall t\in\N$,
then there must be $c\leq\frac{a}{b}=\frac{a}{2+2a}$. Since $a>0$
can be arbitrarily small, this suggests that no meaningful high-probability
rates (i.e., decay polynomially in $t$) exist, at least for the last
iterate.
\end{prop}
\begin{rem}
We clarify that the above result does not rule out the possibility
that $\OGD/\SGD$ may admit a meaningful high-probability bound for
more sophisticated stepsize or the average iterate.
\end{rem}
\begin{proof}
For $t\in\N$, we take $\delta=\frac{1}{t^{2}}$ to have $\Pr\left[F(\bx_{t})>\frac{\mathrm{polylog}(t^{2})}{t^{c}}\right]\leq\frac{1}{t^{2}}$.
By the Borel-Cantelli Lemma, we obtain
\begin{equation}
\Pr\left[F(\bx_{t})>\frac{\mathrm{polylog}(t^{2})}{t^{c}}\text{ i.o.}\right]=0\Leftrightarrow\Pr\left[F(\bx_{t})\leq\frac{\mathrm{polylog}(t^{2})}{t^{c}}\text{ eventually}\right]=1.\label{eq:hardness-1}
\end{equation}

Next, we define the event $A_{t}\defeq\left\{ \xi_{t}>\frac{1}{\eta_{t}t^{a/b}}\right\} $.
Note that $\frac{1}{\eta_{t}t^{a/b}}=\frac{t^{1/2-a/b}}{\eta}\geq1>\frac{b-1}{b}$,
implying that $\Pr\left[A_{t}\right]=\left(\frac{(b-1)\eta}{b}\right)^{b}\frac{1}{t}\Rightarrow\sum_{t=1}^{\infty}\Pr\left[A_{t}\right]=\infty$.
Since $A_{t}$ is independent of each other, then by the second Borel-Cantelli
Lemma,
\begin{equation}
\Pr\left[A_{t}\text{ i.o.}\right]=1.\label{eq:hardness-2}
\end{equation}
Now, let us write $A_{t}=A_{t,1}\cup A_{t,2}$, where
\begin{eqnarray*}
A_{t,1}\defeq A_{t}\cap\left\{ \left|\bx_{t}-\eta_{t}\bg_{t}\right|>1\right\}  & \text{and} & A_{t,2}\defeq A_{t}\cap\left\{ \left|\bx_{t}-\eta_{t}\bg_{t}\right|\leq1\right\} 
\end{eqnarray*}
are two disjoint events. Moreover, we introduce another event $B_{t}\defeq\left\{ \left|\bx_{t+1}\right|+\left|\bx_{t}\right|\geq\frac{1}{t^{a/b}}\right\} $
and observe that:
\begin{itemize}
\item There is $A_{t,1}\subseteq\left\{ \left|\bx_{t}-\eta_{t}\bg_{t}\right|>1\right\} \subseteq\left\{ \left|\bx_{t+1}\right|=1\right\} \subseteq B_{t}$.
\item There is $A_{t,2}=A_{t}\cap\left\{ \bx_{t+1}=\bx_{t}-\eta_{t}\bg_{t}\right\} \subseteq A_{t}\cap\left\{ \left|\bx_{t+1}\right|+\left|\bx_{t}\right|\geq\eta_{t}\left|\bg_{t}\right|=\eta_{t}\xi_{t}\right\} \subseteq B_{t}$.
\end{itemize}
Therefore, we always have 
\begin{equation}
A_{t}\subseteq B_{t}\overset{(\ref{eq:hardness-2})}{\Rightarrow}\Pr\left[B_{t}\text{ i.o.}\right]=1.\label{eq:hardness-3}
\end{equation}

Finally, for a sample path satisfying $\left|\bx_{t}\right|=F(\bx_{t})\leq\frac{\mathrm{polylog}(t^{2})}{t^{c}}$
eventually and $\left|\bx_{t+1}\right|+\left|\bx_{t}\right|\geq\frac{1}{t^{a/b}}$
i.o. (the existence is due to (\ref{eq:hardness-1}) and (\ref{eq:hardness-3})),
there must be $\frac{1}{t^{a/b}}\lesssim\frac{\mathrm{polylog}(t^{2})}{t^{c}}$
for infinitely many $t\in\N$, which implies that $c\leq\frac{a}{b}=\frac{a}{2+2a}$.
\end{proof}

\section{Missing Proofs for Online Gradient Descent\label{sec:OGD}}

This section provides missing proofs for regret bounds of $\OGD$.
Before showing the formal proof, we recall the following core inequality
that holds for any $\bx\in\X$ given in (\ref{eq:OGD-6}):
\begin{equation}
\left\langle \bg_{t},\bx_{t}-\bx\right\rangle \leq\frac{\left\Vert \bx_{t}-\bx\right\Vert ^{2}-\left\Vert \bx_{t+1}-\bx\right\Vert ^{2}}{2\eta_{t}}+\eta_{t}G^{2}+\C(\p)\eta_{t}^{\p-1}\left\Vert \err_{t}\right\Vert ^{\p}D^{2-\p}.\label{eq:OGD-core-1}
\end{equation}
The key to establishing the above result is showing
\begin{equation}
\left\langle \bg_{t},\bx_{t}-\bx_{t+1}\right\rangle -\frac{\left\Vert \bx_{t}-\bx_{t+1}\right\Vert ^{2}}{2\eta_{t}}\leq\eta_{t}G^{2}+\C(\p)\eta_{t}^{\p-1}\left\Vert \err_{t}\right\Vert ^{\p}D^{2-\p},\label{eq:OGD-core-2}
\end{equation}
the proof of which is by combining (\ref{eq:OGD-3}), (\ref{eq:OGD-4}),
and (\ref{eq:OGD-5}) established in the main text.

\subsection{Proof of Theorem \ref{thm:main-OGD}}

\begin{proof}
For any $\bx\in\X$, sum up (\ref{eq:OGD-core-1}) from $t=1$ to
$T$ and drop the term $-\frac{\left\Vert \bx_{T+1}-\bx\right\Vert ^{2}}{2\eta_{T}}$
to obtain
\begin{align}
 & \sum_{t=1}^{T}\left\langle \bg_{t},\bx_{t}-\bx\right\rangle \nonumber \\
\leq & \frac{\left\Vert \bx_{1}-\bx\right\Vert ^{2}}{2\eta_{1}}+\sum_{t=1}^{T-1}\left(\frac{1}{\eta_{t+1}}-\frac{1}{\eta_{t}}\right)\frac{\left\Vert \bx_{t+1}-\bx\right\Vert ^{2}}{2}+\sum_{t=1}^{T}\eta_{t}G^{2}+\C(\p)\eta_{t}^{\p-1}\left\Vert \err_{t}\right\Vert ^{\p}D^{2-\p}\label{eq:OGD-cvx-1}\\
\leq & \frac{D^{2}}{\eta_{T}}+\sum_{t=1}^{T}\eta_{t}G^{2}+\C(\p)\eta_{t}^{\p-1}\left\Vert \err_{t}\right\Vert ^{\p}D^{2-\p},\label{eq:OGD-cvx-2}
\end{align}
where the last step is due to $\left\Vert \bx_{t}-\bx\right\Vert \leq D,\forall t\in\left[T\right]$
and $\eta_{t+1}\leq\eta_{t},\forall t\in\left[T-1\right]$.

Taking expectations on both sides of (\ref{eq:OGD-cvx-2}) yields
that
\begin{equation}
\E\left[\reg_{T}^{\OGD}(\bx)\right]\leq\frac{D^{2}}{\eta_{T}}+\sum_{t=1}^{T}\eta_{t}G^{2}+\C(\p)\eta_{t}^{\p-1}\sigma^{\p}D^{2-\p},\label{eq:OGD-cvx-3}
\end{equation}
where for the L.H.S., we use $\E\left[\left\langle \bg_{t},\bx_{t}-\bx\right\rangle \right]=\E\left[\E\left[\left\langle \bg_{t},\bx_{t}-\bx\right\rangle \mid\F_{t-1}\right]\right]$
and
\begin{equation}
\E\left[\left\langle \bg_{t},\bx_{t}-\bx\right\rangle \mid\F_{t-1}\right]=\left\langle \E\left[\bg_{t}\mid\F_{t-1}\right],\bx_{t}-\bx\right\rangle =\left\langle \nabla\ell_{t}(\bx_{t}),\bx_{t}-\bx\right\rangle \geq\ell_{t}(\bx_{t})-\ell_{t}(\bx),\label{eq:OGD-cvx-4}
\end{equation}
for the R.H.S., we use $\E\left[\left\Vert \err_{t}\right\Vert ^{\p}\right]\leq\sigma^{\p}$. 

Finally, we plug $\eta_{t}=\frac{D}{G\sqrt{t}}\land\frac{D}{\sigma t^{1/\p}},\forall t\in\left[T\right]$
into (\ref{eq:OGD-cvx-3}), then use $\sum_{t=1}^{T}\frac{1}{\sqrt{t}}\lesssim\sqrt{T}$
and $\sum_{t=1}^{T}\frac{1}{t^{1-1/\p}}\lesssim T^{1/\p}$ to conclude
\[
\E\left[\reg_{T}^{\OGD}(\bx)\right]\lesssim GD\sqrt{T}+\sigma DT^{1/\p}.
\]
\end{proof}

\subsection{Extension to Online Strongly Convex Optimization}

Next, we extend Theorem \ref{thm:main-OGD} to the strongly convex
case, i.e., $\exists\mu>0$ such that for all $t\in\left[T\right]$,
\begin{equation}
\frac{\mu}{2}\left\Vert \bx-\by\right\Vert ^{2}+\left\langle \nabla\ell_{t}(\by),\bx-\by\right\rangle +\ell_{t}(\by)\leq\ell_{t}(\bx),\forall\bx,\by\in\X,\nabla\ell_{t}(\by)\in\partial\ell_{t}(\by).\label{eq:str}
\end{equation}
In this setting, it is well known that $\OGD$ achieves a logarithmic
regret bound when $\p=2$ \citep{OPT-013,orabona2019modern}. Theorem
\ref{thm:OGD-str} below provides the first provable result for $\p<2$.
\begin{thm}
\label{thm:OGD-str}Under Assumption \ref{assu:OCO} and additionally
assuming (\ref{eq:str}), taking $\eta_{t}=\frac{1}{\mu t}$ in $\OGD$
(Algorithm \ref{alg:OGD}), we have
\[
\E\left[\reg_{T}^{\OGD}(\bx)\right]\lesssim\frac{G^{2}\left(1+\log T\right)}{\mu}+\frac{\sigma^{\p}D^{2-\p}}{\mu^{\p-1}}\times\begin{cases}
T^{2-\p} & \p\in\left(1,2\right)\\
1+\log T & \p=2
\end{cases},\forall\bx\in\X.
\]
\end{thm}
Theorem \ref{thm:OGD-str} shows that under strongly convexity, $\OGD$
for $\p\in\left(1,2\right)$ achieves a better sublinear regret $T^{2-\p}$
than $T^{1/\p}$ in Theorem \ref{thm:main-OGD} as $2-\p\leq1/\p,\forall\p>0$.
One point we highlight here is that the stepsize $\eta_{t}=\frac{1}{\mu t}$
is commonly used in the OCO literature and is independent of the tail
index $\p$.

However, in contrast to Theorem \ref{thm:main-OGD}, we suspect Theorem
\ref{thm:OGD-str} is not tight in $T$ for $\p\in\left(1,2\right)$.
The reason is that for nonsmooth strongly convex optimization with
heavy tails (i.e., $\ell_{t}=F,\forall t\in\left[T\right]$ where
$F$ is strongly convex), Theorem \ref{thm:OGD-str} can convert to
a convergence rate only in the order of $1/T^{\p-1}$, which is worse
than the lower bound $1/T^{2-2/\p}$ \citep{NEURIPS2020_b05b57f6}.
Therefore, we conjecture that a way to obtain a better regret bound
than $T^{2-\p}$ exists, which we leave as future work.

\begin{proof}[Proof of Theorem \ref{thm:OGD-str}]
For any $\bx\in\X$, we take expectations on both sides of (\ref{eq:OGD-cvx-1})
to have
\begin{align}
\E\left[\reg_{T}^{\OGD}(\bx)\right]\leq & \left(\frac{1}{\eta_{1}}-\mu\right)\frac{\left\Vert \bx_{1}-\bx\right\Vert ^{2}}{2}+\sum_{t=1}^{T-1}\left(\frac{1}{\eta_{t+1}}-\frac{1}{\eta_{t}}-\mu\right)\frac{\E\left[\left\Vert \bx_{t+1}-\bx\right\Vert ^{2}\right]}{2}\nonumber \\
 & +\sum_{t=1}^{T}\eta_{t}G^{2}+\C(\p)\eta_{t}^{\p-1}\sigma^{\p}D^{2-\p},\label{eq:OGD-str-1}
\end{align}
where for the L.H.S., we follow a similar step of reasoning out (\ref{eq:OGD-cvx-4})
but instead using 
\[
\left\langle \nabla\ell_{t}(\bx_{t}),\bx_{t}-\bx\right\rangle \geq\ell_{t}(\bx_{t})-\ell_{t}(\bx)+\frac{\mu}{2}\left\Vert \bx_{t}-\bx\right\Vert ^{2},
\]
for the R.H.S., we use $\E\left[\left\Vert \err_{t}\right\Vert ^{\p}\right]\leq\sigma^{\p}$.

Next, we plug $\eta_{t}=\frac{1}{\mu t},\forall t\in\left[T\right]$
into (\ref{eq:OGD-str-1}) to obtain
\[
\E\left[\reg_{T}^{\OGD}(\bx)\right]\lesssim\sum_{t=1}^{T}\frac{G^{2}}{\mu t}+\frac{\sigma^{\p}D^{2-\p}}{\mu^{\p-1}t^{\p-1}}\lesssim\frac{G^{2}\left(1+\log T\right)}{\mu}+\frac{\sigma^{\p}D^{2-\p}}{\mu^{\p-1}}\times\begin{cases}
T^{2-\p} & \p\in\left(1,2\right)\\
1+\log T & \p=2
\end{cases}.
\]
\end{proof}

\section{Missing Proofs for Dual Averaging\label{sec:DA}}

This section provides missing proofs for regret bounds of $\DA$.

\subsection{Proof of Theorem \ref{thm:main-DA}}

\begin{proof}
Let $L_{t}(\bx)\defeq\frac{\left\Vert \bx-\bx_{1}\right\Vert ^{2}}{2\eta_{t-1}}+\sum_{s=1}^{t-1}\left\langle \bg_{s},\bx\right\rangle ,\forall t\in\left[T+1\right]$,
where $\eta_{0}\defeq\eta_{1}$. Then $\DA$ can be equivalently written
as
\[
\bx_{t}=\argmin_{\bx\in\X}L_{t}(\bx),\forall t\in\left[T+1\right].
\]

By Lemma 7.1 of \citep{orabona2019modern}, for any $\bx\in\X$,
\begin{align*}
\sum_{t=1}^{T}\left\langle \bg_{t},\bx_{t}-\bx\right\rangle  & =\frac{\left\Vert \bx-\bx_{1}\right\Vert ^{2}}{2\eta_{T}}+L_{T+1}(\bx_{T+1})-L_{T+1}(\bx)+\sum_{t=1}^{T}L_{t}(\bx_{t})+\left\langle \bg_{t},\bx_{t}\right\rangle -L_{t+1}(\bx_{t+1})\\
 & \leq\frac{\left\Vert \bx-\bx_{1}\right\Vert ^{2}}{2\eta_{T}}+\sum_{t=1}^{T}L_{t}(\bx_{t})-L_{t+1}(\bx_{t+1})+\left\langle \bg_{t},\bx_{t}\right\rangle ,
\end{align*}
where the inequality holds by $L_{T+1}(\bx_{T+1})\leq L_{T+1}(\bx),\forall\bx\in\X$
due to $\bx_{T+1}=\argmin_{\bx\in\X}L_{T+1}(\bx)$. Note that for
any $t\in\left[T\right]$,
\begin{align*}
 & L_{t}(\bx_{t})-L_{t+1}(\bx_{t+1})+\left\langle \bg_{t},\bx_{t}\right\rangle \\
= & L_{t}(\bx_{t})-L_{t}(\bx_{t+1})+\left\langle \bg_{t},\bx_{t}-\bx_{t+1}\right\rangle +\frac{\left\Vert \bx_{t+1}-\bx_{1}\right\Vert ^{2}}{2\eta_{t-1}}-\frac{\left\Vert \bx_{t+1}-\bx_{1}\right\Vert ^{2}}{2\eta_{t}}\\
\overset{(a)}{\leq} & L_{t}(\bx_{t})-L_{t}(\bx_{t+1})+\left\langle \bg_{t},\bx_{t}-\bx_{t+1}\right\rangle \\
\overset{(b)}{\leq} & \left\langle \bg_{t},\bx_{t}-\bx_{t+1}\right\rangle -\frac{\left\Vert \bx_{t}-\bx_{t+1}\right\Vert ^{2}}{2\eta_{t-1}},
\end{align*}
where $(a)$ is by $\eta_{t}\leq\eta_{t-1},\forall t\in\left[T\right]$
and $(b)$ holds because $L_{t}$ is $\frac{1}{\eta_{t-1}}$-strongly
convex and $\bx_{t}=\argmin_{\bx\in\X}L_{t}(\bx)$, which together
imply
\[
L_{t}(\bx_{t})-L_{t}(\bx_{t+1})\leq\left\langle \nabla L_{t}(\bx_{t}),\bx_{t}-\bx_{t+1}\right\rangle -\frac{\left\Vert \bx_{t}-\bx_{t+1}\right\Vert ^{2}}{2\eta_{t-1}}\leq-\frac{\left\Vert \bx_{t}-\bx_{t+1}\right\Vert ^{2}}{2\eta_{t-1}}.
\]
Therefore, we have
\begin{equation}
\sum_{t=1}^{T}\left\langle \bg_{t},\bx_{t}-\bx\right\rangle \leq\frac{\left\Vert \bx-\bx_{1}\right\Vert ^{2}}{2\eta_{T}}+\sum_{t=1}^{T}\left\langle \bg_{t},\bx_{t}-\bx_{t+1}\right\rangle -\frac{\left\Vert \bx_{t}-\bx_{t+1}\right\Vert ^{2}}{2\eta_{t-1}}.\label{eq:DA-cvx-1}
\end{equation}

By the same argument as proving (\ref{eq:OGD-core-2}) but replacing
$\eta_{t}$ with $\eta_{t-1}$, there is
\[
\left\langle \bg_{t},\bx_{t}-\bx_{t+1}\right\rangle -\frac{\left\Vert \bx_{t}-\bx_{t+1}\right\Vert ^{2}}{2\eta_{t-1}}\leq\eta_{t-1}G^{2}+\C(\p)\eta_{t-1}^{\p-1}\left\Vert \err_{t}\right\Vert ^{\p}D^{2-\p}.
\]
As such, we know
\begin{equation}
\sum_{t=1}^{T}\left\langle \bg_{t},\bx_{t}-\bx\right\rangle \leq\frac{\left\Vert \bx-\bx_{1}\right\Vert ^{2}}{2\eta_{T}}+\sum_{t=1}^{T}\eta_{t-1}G^{2}+\C(\p)\eta_{t-1}^{\p-1}\left\Vert \err_{t}\right\Vert ^{\p}D^{2-\p}.\label{eq:DA-cvx-2}
\end{equation}
Finally, following similar steps in proving Theorem \ref{thm:main-OGD}
in Appendix \ref{sec:OGD}, we conclude
\[
\E\left[\reg_{T}^{\DA}(\bx)\right]\lesssim GD\sqrt{T}+\sigma DT^{1/\p}.
\]
\end{proof}

\subsection{Dual Averaging with an Adaptive Stepsize}

We show that $\DA$ with an adaptive stepsize can also achieve the
optimal regret $GD\sqrt{T}+\sigma DT^{1/\p}$.
\begin{thm}
\label{thm:DA-ada}Under Assumption \ref{assu:OCO}, taking $\eta_{t}=2DV_{t}^{-1/2}$
and $V_{t}=\sum_{s=1}^{t}\left\Vert \bg_{s}\right\Vert ^{2}$ in $\DA$
(Algorithm \ref{alg:DA}), we have
\[
\E\left[\reg_{T}^{\DA}(\bx)\right]\lesssim GD\sqrt{T}+\sigma DT^{1/\p},\forall\bx\in\X.
\]
\end{thm}
\begin{proof}
For any $\bx\in\X$, we have
\begin{equation}
\sum_{t=1}^{T}\left\langle \bg_{t},\bx_{t}-\bx\right\rangle \overset{(\ref{eq:DA-cvx-1})}{\leq}\frac{\left\Vert \bx-\bx_{1}\right\Vert ^{2}}{2\eta_{T}}+\sum_{t=1}^{T}\left\langle \bg_{t},\bx_{t}-\bx_{t+1}\right\rangle -\frac{\left\Vert \bx_{t}-\bx_{t+1}\right\Vert ^{2}}{2\eta_{t-1}},\label{eq:DA-cvx-ada-1}
\end{equation}
where $\eta_{0}\defeq\eta_{1}$. On the one hand, we can use AM-GM
inequality to bound
\[
\left\langle \bg_{t},\bx_{t}-\bx_{t+1}\right\rangle -\frac{\left\Vert \bx_{t}-\bx_{t+1}\right\Vert ^{2}}{2\eta_{t-1}}\leq\frac{\eta_{t-1}\left\Vert \bg_{t}\right\Vert ^{2}}{2}.
\]
On the other hand, we know
\begin{equation}
\left\langle \bg_{t},\bx_{t}-\bx_{t+1}\right\rangle -\frac{\left\Vert \bx_{t}-\bx_{t+1}\right\Vert ^{2}}{2\eta_{t-1}}\leq\left\langle \bg_{t},\bx_{t}-\bx_{t+1}\right\rangle \leq\left\Vert \bg_{t}\right\Vert \left\Vert \bx_{t}-\bx_{t+1}\right\Vert \leq\left\Vert \bg_{t}\right\Vert D,\label{eq:DA-cvx-ada-2}
\end{equation}
where the second step is by Cauchy-Schwarz inequality. Therefore,
for any $t\geq2$,
\begin{align}
\left\langle \bg_{t},\bx_{t}-\bx_{t+1}\right\rangle -\frac{\left\Vert \bx_{t}-\bx_{t+1}\right\Vert ^{2}}{2\eta_{t-1}} & \leq\frac{\eta_{t-1}\left\Vert \bg_{t}\right\Vert ^{2}}{2}\land\left\Vert \bg_{t}\right\Vert D\overset{(a)}{\leq}\frac{2}{\frac{2}{\eta_{t-1}\left\Vert \bg_{t}\right\Vert ^{2}}+\frac{1}{\left\Vert \bg_{t}\right\Vert D}}\nonumber \\
 & \overset{(b)}{=}\frac{2D\left\Vert \bg_{t}\right\Vert ^{2}}{\sqrt{\sum_{s=1}^{t-1}\left\Vert \bg_{s}\right\Vert ^{2}}+\left\Vert \bg_{t}\right\Vert }\overset{(c)}{\leq}\frac{2D\left\Vert \bg_{t}\right\Vert ^{2}}{\sqrt{\sum_{s=1}^{t}\left\Vert \bg_{s}\right\Vert ^{2}}},\label{eq:DA-cvx-ada-3}
\end{align}
where $(a)$ is due to $x\land y\leq\frac{2}{x^{-1}+y^{-1}},\forall x,y>0$,
$(b)$ is by $\eta_{t-1}=\frac{2D}{\sqrt{\sum_{s=1}^{t-1}\left\Vert \bg_{s}\right\Vert ^{2}}}$,
and $(c)$ holds because of $\sqrt{\sum_{s=1}^{t}\left\Vert \bg_{s}\right\Vert ^{2}}\leq\sqrt{\sum_{s=1}^{t-1}\left\Vert \bg_{s}\right\Vert ^{2}}+\left\Vert \bg_{t}\right\Vert $.
Note that (\ref{eq:DA-cvx-ada-3}) is also true for $t=1$ by (\ref{eq:DA-cvx-ada-2}).

Combine (\ref{eq:DA-cvx-ada-1}) and (\ref{eq:DA-cvx-ada-3}) and
use $\left\Vert \bx-\bx_{1}\right\Vert \leq D$ to obtain
\[
\sum_{t=1}^{T}\left\langle \bg_{t},\bx_{t}-\bx\right\rangle \le\frac{D^{2}}{2\eta_{T}}+\sum_{t=1}^{T}\frac{2D\left\Vert \bg_{t}\right\Vert ^{2}}{\sqrt{\sum_{s=1}^{t}\left\Vert \bg_{s}\right\Vert ^{2}}}=\frac{D^{2}}{2\eta_{T}}+\sum_{t=1}^{T}\eta_{t}\left\Vert \bg_{t}\right\Vert ^{2},
\]
which only differs from (\ref{eq:AdaGrad-cvx-1}) by a constant. Hence,
by a similar proof for (\ref{eq:AdaGrad-cvx-3}), there is
\[
\sum_{t=1}^{T}\left\langle \bg_{t},\bx_{t}-\bx\right\rangle \lesssim D\left[\sqrt{\sum_{t=1}^{T}\left\Vert \nabla\ell_{t}(\bx_{t})\right\Vert ^{2}}+\left(\sum_{t=1}^{T}\left\Vert \err_{t}\right\Vert ^{\p}\right)^{\frac{1}{\p}}\right],
\]
implying
\[
\E\left[\reg_{T}^{\DA}(\bx)\right]\lesssim GD\sqrt{T}+\sigma DT^{1/\p}.
\]
\end{proof}

\section{Missing Proofs for $\protect\Ada$\label{sec:AdaGrad}}

This section provides missing proofs for regret bounds of $\Ada$.

\subsection{Proof of Theorem \ref{thm:main-AdaGrad}}

\begin{proof}
As mentioned, $\Ada$ can be viewed as $\OGD$ with a stepsize $\eta_{t}=\frac{\eta}{\sqrt{V_{t}}}=\frac{\eta}{\sqrt{\sum_{s=1}^{t}\left\Vert \bg_{s}\right\Vert ^{2}}}$.
Therefore, we can use (\ref{eq:OGD-1}) for $\Ada$ to know for any
$\bx\in\X$,
\[
\left\langle \bg_{t},\bx_{t}-\bx\right\rangle \leq\frac{\left\Vert \bx_{t}-\bx\right\Vert ^{2}-\left\Vert \bx_{t+1}-\bx\right\Vert ^{2}}{2\eta_{t}}+\frac{\eta_{t}\left\Vert \bg_{t}\right\Vert ^{2}}{2}.
\]
Sum up the above inequality from $t=1$ to $T$ and drop the term
$-\frac{\left\Vert \bx_{T+1}-\bx\right\Vert ^{2}}{2\eta_{T}}$ to
have
\begin{align}
\sum_{t=1}^{T}\left\langle \bg_{t},\bx_{t}-\bx\right\rangle  & \leq\frac{\left\Vert \bx_{1}-\bx\right\Vert ^{2}}{2\eta_{1}}+\sum_{t=1}^{T-1}\left(\frac{1}{\eta_{t+1}}-\frac{1}{\eta_{t}}\right)\frac{\left\Vert \bx_{t+1}-\bx\right\Vert ^{2}}{2}+\sum_{t=1}^{T}\frac{\eta_{t}\left\Vert \bg_{t}\right\Vert ^{2}}{2}\nonumber \\
 & \leq\frac{D^{2}}{2\eta_{T}}+\sum_{t=1}^{T}\frac{\eta_{t}\left\Vert \bg_{t}\right\Vert ^{2}}{2},\label{eq:AdaGrad-cvx-1}
\end{align}
where the last step is by $\left\Vert \bx_{t}-\bx\right\Vert \leq D,\forall t\in\left[T\right]$
and $\eta_{t+1}\leq\eta_{t},\forall t\in\left[T-1\right]$.

Next, observe that for any $t\in\left[T\right]$,
\[
\left\Vert \bg_{t}\right\Vert ^{2}=\frac{\eta^{2}}{\eta_{t}^{2}}-\frac{\eta^{2}}{\eta_{t-1}^{2}}=\eta^{2}\left(\frac{1}{\eta_{t}}-\frac{1}{\eta_{t-1}}\right)\left(\frac{1}{\eta_{t}}+\frac{1}{\eta_{t-1}}\right)\leq\frac{2\eta^{2}}{\eta_{t}}\left(\frac{1}{\eta_{t}}-\frac{1}{\eta_{t-1}}\right),
\]
where $1/\eta_{0}$ should be read as $0$. The above inequality implies
\begin{equation}
\sum_{t=1}^{T}\frac{\eta_{t}\left\Vert \bg_{t}\right\Vert ^{2}}{2}\leq\eta^{2}\sum_{t=1}^{T}\frac{1}{\eta_{t}}-\frac{1}{\eta_{t-1}}=\frac{\eta^{2}}{\eta_{T}}.\label{eq:AdaGrad-cvx-2}
\end{equation}
Combine (\ref{eq:AdaGrad-cvx-1}) and (\ref{eq:AdaGrad-cvx-2}) to
have
\[
\sum_{t=1}^{T}\left\langle \bg_{t},\bx_{t}-\bx\right\rangle \leq\frac{D^{2}}{2\eta_{T}}+\frac{\eta^{2}}{\eta_{T}}=\left(\frac{D^{2}}{2\eta}+\eta\right)\sqrt{\sum_{t=1}^{T}\left\Vert \bg_{t}\right\Vert ^{2}}.
\]
Note that there is
\begin{align*}
\sqrt{\sum_{t=1}^{T}\left\Vert \bg_{t}\right\Vert ^{2}} & \leq\sqrt{\sum_{t=1}^{T}2\left\Vert \nabla\ell_{t}(\bx_{t})\right\Vert ^{2}+2\left\Vert \err_{t}\right\Vert ^{2}}\leq\sqrt{2\sum_{t=1}^{T}\left\Vert \nabla\ell_{t}(\bx_{t})\right\Vert ^{2}}+\sqrt{2\sum_{t=1}^{T}\left\Vert \err_{t}\right\Vert ^{2}}\\
 & \leq\sqrt{2\sum_{t=1}^{T}\left\Vert \nabla\ell_{t}(\bx_{t})\right\Vert ^{2}}+\sqrt{2}\left(\sum_{t=1}^{T}\left\Vert \err_{t}\right\Vert ^{\p}\right)^{\frac{1}{\p}},
\end{align*}
where the last step is due to $\left\Vert \cdot\right\Vert _{2}\leq\left\Vert \cdot\right\Vert _{\p}$
for any $\p\in\left[1,2\right]$. Hence, we obtain
\begin{equation}
\sum_{t=1}^{T}\left\langle \bg_{t},\bx_{t}-\bx\right\rangle \leq\sqrt{2}\left(\frac{D^{2}}{2\eta}+\eta\right)\left[\sqrt{\sum_{t=1}^{T}\left\Vert \nabla\ell_{t}(\bx_{t})\right\Vert ^{2}}+\left(\sum_{t=1}^{T}\left\Vert \err_{t}\right\Vert ^{\p}\right)^{\frac{1}{\p}}\right].\label{eq:AdaGrad-cvx-3}
\end{equation}
We take expectations on both sides of (\ref{eq:AdaGrad-cvx-3}), then
apply H\"{o}lder's inequality to have
\[
\E\left[\left(\sum_{t=1}^{T}\left\Vert \err_{t}\right\Vert ^{\p}\right)^{\frac{1}{\p}}\right]\leq\left(\sum_{t=1}^{T}\E\left[\left\Vert \err_{t}\right\Vert ^{\p}\right]\right)^{\frac{1}{\p}}\leq\sigma T^{\frac{1}{\p}},
\]
and finally plug in $\eta=D/\sqrt{2}$ to conclude
\[
\E\left[\reg_{T}^{\Ada}(\bx)\right]\lesssim GD\sqrt{T}+\sigma DT^{1/\p}.
\]
\end{proof}

\section{Missing Proofs for Applications: Nonsmooth Convex Optimization\label{sec:cvx}}

In this section, we prove the last-iterate convergence for $\SGD$
(i.e., $\OGD$ for stochastic optimization) under heavy-tailed noise.
Following the existing literature \citep{lan2020first,liu2024revisiting},
we will move beyond the standard $G$-Lipschitz convex function considered
before and instead focus on a broader class of nonsmooth functions,
i.e., $\exists\mu\geq0$, $G\geq0$, and $H\geq0$, such that $G+H>0$
and
\begin{equation}
\frac{\mu}{2}\left\Vert \bx-\by\right\Vert ^{2}\leq\ell_{t}(\bx)-\ell_{t}(\by)-\left\langle \nabla\ell_{t}(\by),\bx-\by\right\rangle \leq2G\left\Vert \bx-\by\right\Vert +\frac{H}{2}\left\Vert \bx-\by\right\Vert ^{2},\forall\bx,\by\in\X,\nabla\ell_{t}(\by)\in\partial\ell_{t}(\by).\label{eq:curvature}
\end{equation}
Notably, such a condition encompasses various commonly studied objectives
simultaneously (i.e., strong convexity and smoothness).

Theorem \ref{thm:OGD-last-core} first shows the general last-iterate
convergence result, the proof of which is inspired by \citep{liu2024revisiting,doi:10.1137/24M1717762}.
\begin{thm}
\label{thm:OGD-last-core}Under Assumption \ref{assu:OCO} (without
the third point) and additionally assuming (\ref{eq:curvature}) for
$\ell_{t}(\bx)=F(\bx)$, for any stepsize $0<\eta_{t}\leq\frac{1}{2H\lor\mu}$
in $\OGD$ (Algorithm \ref{alg:OGD}), let $\gamma_{t}\defeq\frac{\eta_{t}}{\prod_{s=2}^{t}(1-\mu\eta_{s})}$,
we have
\[
\E\left[F(\bx_{T})-F(\bx)\right]\lesssim\frac{(1-\mu\eta_{1})D^{2}}{\sum_{t=1}^{T-1}\gamma_{t}}+G^{2}\sum_{t=1}^{T-1}\frac{\gamma_{t}\eta_{t}}{\sum_{s=t}^{T-1}\gamma_{s}}+\sigma^{\p}D^{2-\p}\sum_{t=1}^{T-1}\frac{\gamma_{t}\eta_{t}^{\p-1}}{\sum_{s=t}^{T-1}\gamma_{s}},\forall\bx\in\X.
\]
\end{thm}
\begin{proof}
Given $\bx\in\X$, we recursively define
\begin{eqnarray}
\by_{0}\defeq\bx & \text{and} & \by_{t}\defeq\left(1-\frac{w_{t-1}}{w_{t}}\right)\bx_{t}+\frac{w_{t-1}}{w_{t}}\by_{t-1},\forall t\in\left[T\right],\label{eq:OGD-last-core-y}
\end{eqnarray}
in which
\begin{eqnarray}
w_{t}\defeq\frac{\gamma_{T}}{\sum_{s=t}^{T}\gamma_{s}},\forall t\in\left[T\right] & \text{and} & w_{0}\defeq w_{1}.\label{eq:OGD-last-core-w}
\end{eqnarray}
Equivalently, $\by_{t}$ can be written into a convex combination
of $\bx,\bx_{1},\mydots,\bx_{t}$ as
\begin{equation}
\by_{t}=\frac{w_{0}}{w_{t}}\bx+\sum_{s=1}^{t}\frac{w_{s}-w_{s-1}}{w_{t}}\bx_{s},\forall t\in\left\{ 0\right\} \cup\left[T\right].\label{eq:OGD-last-core-y-equiv}
\end{equation}
Therefore, $\by_{t}$ also falls into $\X$ and satisfies $\by_{t}\in\F_{t-1}$.

We invoke (\ref{eq:OGD-2}) for $\by_{t}$ to obtain
\[
\left\langle \bg_{t},\bx_{t}-\by_{t}\right\rangle \leq\frac{\left\Vert \bx_{t}-\by_{t}\right\Vert ^{2}-\left\Vert \bx_{t+1}-\by_{t}\right\Vert ^{2}}{2\eta_{t}}+\left\langle \bg_{t},\bx_{t}-\bx_{t+1}\right\rangle -\frac{\left\Vert \bx_{t}-\bx_{t+1}\right\Vert ^{2}}{2\eta_{t}},
\]
and then bound
\begin{align*}
 & \left\langle \bg_{t},\bx_{t}-\bx_{t+1}\right\rangle =\left\langle \nabla F(\bx_{t}),\bx_{t}-\bx_{t+1}\right\rangle +\left\langle \err_{t},\bx_{t}-\bx_{t+1}\right\rangle \\
\leq & F(\bx_{t})-F(\bx_{t+1})+8\eta_{t}G^{2}+2^{\p-1}\C(\p)\eta_{t}^{\p-1}\left\Vert \err_{t}\right\Vert ^{\p}D^{2-\p}+\frac{\left\Vert \bx_{t}-\bx_{t+1}\right\Vert ^{2}}{2\eta_{t}},
\end{align*}
where, in the second step, we use (\ref{eq:OGD-5}) to bound $\left\langle \err_{t},\bx_{t}-\bx_{t+1}\right\rangle $
(up to change $\eta_{t}$ to $2\eta_{t}$) and the following step
to bound $\left\langle \nabla F(\bx_{t}),\bx_{t}-\bx_{t+1}\right\rangle $,
\begin{align*}
\left\langle \nabla F(\bx_{t}),\bx_{t}-\bx_{t+1}\right\rangle  & \overset{(\ref{eq:curvature})}{\leq}F(\bx_{t})-F(\bx_{t+1})+2G\left\Vert \bx_{t}-\bx_{t+1}\right\Vert +\frac{H}{2}\left\Vert \bx_{t}-\bx_{t+1}\right\Vert ^{2}\\
 & \leq F(\bx_{t})-F(\bx_{t+1})+8\eta_{t}G^{2}+\frac{\left\Vert \bx_{t}-\bx_{t+1}\right\Vert ^{2}}{8\eta_{t}}+\frac{H}{2}\left\Vert \bx_{t}-\bx_{t+1}\right\Vert ^{2}\\
 & \overset{\eta_{t}\leq\frac{1}{2H}}{\leq}F(\bx_{t})-F(\bx_{t+1})+8\eta_{t}G^{2}+\frac{3\left\Vert \bx_{t}-\bx_{t+1}\right\Vert ^{2}}{8\eta_{t}}.
\end{align*}
Hence, we know
\begin{equation}
\left\langle \bg_{t},\bx_{t}-\by_{t}\right\rangle \leq F(\bx_{t})-F(\bx_{t+1})+\frac{\left\Vert \bx_{t}-\by_{t}\right\Vert ^{2}-\left\Vert \bx_{t+1}-\by_{t}\right\Vert ^{2}}{2\eta_{t}}+8\eta_{t}G^{2}+2^{\p-1}\C(\p)\eta_{t}^{\p-1}\left\Vert \err_{t}\right\Vert ^{\p}D^{2-\p}.\label{eq:OGD-last-core-1}
\end{equation}
Since $\bx_{t},\by_{t}\in\F_{t-1}$, there is
\[
\E\left[\left\langle \bg_{t},\bx_{t}-\by_{t}\right\rangle \right]=\E\left[\left\langle \E\left[\bg_{t}\mid\F_{t-1}\right],\bx_{t}-\by_{t}\right\rangle \right]=\E\left[\left\langle \nabla F(\bx_{t}),\bx_{t}-\by_{t}\right\rangle \right]\overset{(\ref{eq:curvature})}{\geq}\E\left[F(\bx_{t})-F(\by_{t})+\frac{\mu}{2}\left\Vert \bx_{t}-\by_{t}\right\Vert ^{2}\right].
\]
As such, we can take expectations on both sides of (\ref{eq:OGD-last-core-1})
and rearrange terms to have
\begin{align}
 & \E\left[F(\bx_{t+1})-F(\by_{t})\right]\leq\frac{(1-\mu\eta_{t})\E\left[\left\Vert \bx_{t}-\by_{t}\right\Vert ^{2}\right]-\E\left[\left\Vert \bx_{t+1}-\by_{t}\right\Vert ^{2}\right]}{2\eta_{t}}+8\eta_{t}G^{2}+2^{\p-1}\C(\p)\eta_{t}^{\p-1}\left\Vert \err_{t}\right\Vert ^{\p}D^{2-\p}\nonumber \\
\leq & \frac{(1-\mu\eta_{t})\E\left[\frac{w_{t-1}}{w_{t}}\left\Vert \bx_{t}-\by_{t-1}\right\Vert ^{2}\right]-\E\left[\left\Vert \bx_{t+1}-\by_{t}\right\Vert ^{2}\right]}{2\eta_{t}}+8\eta_{t}G^{2}+2^{\p-1}\C(\p)\eta_{t}^{\p-1}\left\Vert \err_{t}\right\Vert ^{\p}D^{2-\p},\label{eq:OGD-last-core-2}
\end{align}
where the second step is due to $\left\Vert \bx_{t}-\by_{t}\right\Vert ^{2}\leq\left(1-\frac{w_{t-1}}{w_{t}}\right)\left\Vert \bx_{t}-\bx_{t}\right\Vert ^{2}+\frac{w_{t-1}}{w_{t}}\left\Vert \bx_{t}-\by_{t-1}\right\Vert ^{2}=\frac{w_{t-1}}{w_{t}}\left\Vert \bx_{t}-\by_{t-1}\right\Vert ^{2}$
by (\ref{eq:OGD-last-core-y}) and the convexity of $\left\Vert \bx_{t}-\cdot\right\Vert ^{2}$.
Multiply both sides of (\ref{eq:OGD-last-core-2}) by $w_{t}\gamma_{t}$
and sum up from $t=1$ to $T$ to obtain (note that $\frac{\gamma_{t}}{\eta_{t}}=\frac{1}{\prod_{s=2}^{t}(1-\mu\eta_{s})}=\frac{\gamma_{t+1}(1-\mu\eta_{t+1})}{\eta_{t+1}},\forall t\in\left[T-1\right]$)
\begin{align}
 & \E\left[\sum_{t=1}^{T}w_{t}\gamma_{t}\left(F(\bx_{t+1})-F(\by_{t})\right)\right]\nonumber \\
\leq & \frac{\gamma_{1}(1-\mu\eta_{1})w_{0}\left\Vert \bx_{1}-\by_{0}\right\Vert ^{2}}{2\eta_{1}}-\frac{\gamma_{T}\E\left[w_{T}\left\Vert \bx_{T+1}-\by_{T}\right\Vert ^{2}\right]}{2\eta_{T}}+\sum_{t=1}^{T}8w_{t}\gamma_{t}\eta_{t}G^{2}+2^{\p-1}\C(\p)w_{t}\gamma_{t}\eta_{t}^{\p-1}\sigma^{\p}D^{2-\p}\nonumber \\
\leq & \frac{(1-\mu\eta_{1})w_{0}D^{2}}{2}+\sum_{t=1}^{T}8w_{t}\gamma_{t}\eta_{t}G^{2}+2^{\p-1}\C(\p)w_{t}\gamma_{t}\eta_{t}^{\p-1}\sigma^{\p}D^{2-\p}.\label{eq:OGD-last-core-3}
\end{align}

Now observe that
\[
F(\by_{t})-F(\bx)\overset{(\ref{eq:OGD-last-core-y-equiv})}{\leq}\frac{w_{0}}{w_{t}}\left(F(\bx)-F(\bx)\right)+\sum_{s=1}^{t}\frac{w_{s}-w_{s-1}}{w_{t}}\left(F(\bx_{s})-F(\bx)\right)=\sum_{s=1}^{t}\frac{w_{s}-w_{s-1}}{w_{t}}\left(F(\bx_{s})-F(\bx)\right),
\]
which implies that
\begin{align*}
\sum_{t=1}^{T}w_{t}\gamma_{t}\left(F(\by_{t})-F(\bx)\right) & \leq\sum_{t=1}^{T}\sum_{s=1}^{t}\left(w_{s}-w_{s-1}\right)\gamma_{t}\left(F(\bx_{s})-F(\bx)\right)\\
 & =\sum_{t=1}^{T}\left(w_{t}-w_{t-1}\right)\left(\sum_{s=t}^{T}\gamma_{s}\right)\left(F(\bx_{t})-F(\bx)\right).
\end{align*}
Thus, we can lower bound the L.H.S. of (\ref{eq:OGD-last-core-3})
by
\begin{align}
\sum_{t=1}^{T}w_{t}\gamma_{t}\left(F(\bx_{t+1})-F(\by_{t})\right) & =\sum_{t=1}^{T}w_{t}\gamma_{t}\left(F(\bx_{t+1})-F(\bx)\right)-w_{t}\gamma_{t}\left(F(\by_{t})-F(\bx)\right)\nonumber \\
 & \geq w_{T}\gamma_{T}\left(F(\bx_{T+1})-F(\bx)\right)-\left(w_{1}-w_{0}\right)\left(\sum_{s=1}^{T}\gamma_{1}\right)\left(F(\bx_{1})-F(\bx)\right)\nonumber \\
 & +\sum_{t=2}^{T}\left[w_{t-1}\gamma_{t-1}-\left(w_{t}-w_{t-1}\right)\left(\sum_{s=t}^{T}\gamma_{s}\right)\right]\left(F(\bx_{t})-F(\bx)\right)\nonumber \\
 & =w_{T}\gamma_{T}\left(F(\bx_{T+1})-F(\bx)\right),\label{eq:OGD-last-core-4}
\end{align}
where the last step is due to, for $t\geq2$,
\begin{align*}
w_{t-1}\gamma_{t-1}-\left(w_{t}-w_{t-1}\right)\left(\sum_{s=t}^{T}\gamma_{s}\right) & \overset{(\ref{eq:OGD-last-core-w})}{=}\frac{\gamma_{T}}{\sum_{s=t-1}^{T}\gamma_{s}}\cdot\gamma_{t-1}-\left(\frac{\gamma_{T}}{\sum_{s=t}^{T}\gamma_{s}}-\frac{\gamma_{T}}{\sum_{s=t-1}^{T}\gamma_{s}}\right)\left(\sum_{s=t}^{T}\gamma_{s}\right)\\
 & =\frac{\gamma_{T}}{\sum_{s=t-1}^{T}\gamma_{s}}\cdot\gamma_{t-1}-\frac{\gamma_{T}}{\sum_{s=t-1}^{T}\gamma_{s}}\cdot\gamma_{t-1}=0,
\end{align*}
and $w_{1}\overset{(\ref{eq:OGD-last-core-w})}{=}w_{0}$.

We plug (\ref{eq:OGD-last-core-4}) back into (\ref{eq:OGD-last-core-3})
and divide both sides by $w_{T}\gamma_{T}$ to obtain
\begin{align*}
\E\left[F(\bx_{T+1})-F(\bx)\right] & \leq\frac{(1-\mu\eta_{1})w_{0}D^{2}}{2w_{T}\gamma_{T}}+\sum_{t=1}^{T}\frac{8w_{t}\gamma_{t}\eta_{t}}{w_{T}\gamma_{T}}G^{2}+\frac{2^{\p-1}\C(\p)w_{t}\gamma_{t}\eta_{t}^{\p-1}}{w_{T}\gamma_{T}}\sigma^{\p}D^{2-\p}\\
 & \overset{(\ref{eq:OGD-last-core-w})}{\lesssim}\frac{(1-\mu\eta_{1})D^{2}}{\sum_{t=1}^{T}\gamma_{t}}+G^{2}\sum_{t=1}^{T}\frac{\gamma_{t}\eta_{t}}{\sum_{s=t}^{T}\gamma_{s}}+\sigma^{\p}D^{2-\p}\sum_{t=1}^{T}\frac{\gamma_{t}\eta_{t}^{\p-1}}{\sum_{s=t}^{T}\gamma_{s}}.
\end{align*}
Finally, relabel $T$ by $T-1$ to complete the proof.
\end{proof}

Equipped with Theorem \ref{thm:OGD-last-core}, we show the following
last-iterate convergence rate for $\SGD$/$\OGD$. As far as we know,
this is the first and only provable result demonstrating that the
last iterate of $\SGD$ can converge in heavy-tailed stochastic optimization
without gradient clipping. When $T$ is unknown, we only incur an
extra logarithmic factor compared to the best possible rate. In particular,
the first stepsize degenerates to $\eta_{t}=\frac{D}{G\sqrt{t}}\land\frac{D}{\sigma t^{1/\p}}$
employed in Theorem \ref{thm:main-OGD} if $H=0$. When $T$ is known,
we employ the linear decay rate proposed by \citep{doi:10.1137/24M1717762}
to give an optimal rate.
\begin{cor}
\label{cor:OGD-cvx-last}Under Assumption \ref{assu:OCO} (without
the third point) and additionally assuming (\ref{eq:curvature}) (with
$\mu=0$) for $\ell_{t}(\bx)=F(\bx)$:
\begin{itemize}
\item taking $\eta_{t}=\frac{1}{2H}\land\frac{D}{G\sqrt{t}}\land\frac{D}{\sigma t^{1/\p}}$
in $\OGD$ (Algorithm \ref{alg:OGD}), we have
\[
\E\left[F(\bx_{T})-F(\bx)\right]\lesssim\frac{HD^{2}\left(1+\log T\right)}{T}+\frac{GD\left(1+\log T\right)}{\sqrt{T}}+\frac{\sigma D\left(1+\log T\right)}{T^{1-\frac{1}{\p}}},\forall\bx\in\X.
\]
\item taking $\eta_{t}=\frac{T-t}{2HT}\land\frac{D(T-t)}{GT^{3/2}}\land\frac{D(T-t)}{\sigma T^{1+1/\p}}$
in $\OGD$ (Algorithm \ref{alg:OGD}), we have
\[
\E\left[F(\bx_{T})-F(\bx)\right]\lesssim\frac{HD^{2}}{T}+\frac{GD}{\sqrt{T}}+\frac{\sigma D}{T^{1-\frac{1}{\p}}},\forall\bx\in\X.
\]
\end{itemize}
\end{cor}
\begin{proof}
Since both kinds of stepsize satisfy $\eta_{t}\leq\frac{1}{2H\lor\mu}\overset{\mu=0}{=}\frac{1}{2H}$,
by Theorem \ref{thm:OGD-last-core}, we have
\begin{align}
\E\left[F(\bx_{T})-F(\bx)\right] & \lesssim\frac{(1-\mu\eta_{1})D^{2}}{\sum_{t=1}^{T-1}\gamma_{t}}+G^{2}\sum_{t=1}^{T-1}\frac{\gamma_{t}\eta_{t}}{\sum_{s=t}^{T-1}\gamma_{s}}+\sigma^{\p}D^{2-\p}\sum_{t=1}^{T-1}\frac{\gamma_{t}\eta_{t}^{\p-1}}{\sum_{s=t}^{T-1}\gamma_{s}}\nonumber \\
 & =\frac{D^{2}}{\sum_{t=1}^{T-1}\eta_{t}}+G^{2}\sum_{t=1}^{T-1}\frac{\eta_{t}^{2}}{\sum_{s=t}^{T-1}\eta_{s}}+\sigma^{\p}D^{2-\p}\sum_{t=1}^{T-1}\frac{\eta_{t}^{\p}}{\sum_{s=t}^{T-1}\eta_{s}},\label{eq:OGD-cvx-last-1}
\end{align}
where the second step is due $\gamma_{t}=\frac{\eta_{t}}{\prod_{s=2}^{t}(1-\mu\eta_{s})}=\eta_{t}$
when $\mu=0$.

For $\eta_{t}=\frac{1}{2H}\land\frac{D}{G\sqrt{t}}\land\frac{D}{\sigma t^{1/\p}}$,
we observe that by Cauchy-Schwarz inequality
\[
(T-t)^{2}\leq\left(\sum_{s=t}^{T-1}\frac{1}{\eta_{s}}\right)\left(\sum_{s=t}^{T-1}\eta_{s}\right)\Rightarrow\frac{1}{\sum_{s=t}^{T-1}\eta_{s}}\leq\frac{\sum_{s=t}^{T-1}\frac{1}{\eta_{s}}}{(T-t)^{2}}.
\]
Thus, there is
\begin{equation}
\E\left[F(\bx_{T})-F(\bx)\right]\overset{(\ref{eq:OGD-cvx-last-1})}{\lesssim}\frac{D^{2}}{(T-1)^{2}}\sum_{t=1}^{T-1}\frac{1}{\eta_{t}}+G^{2}\sum_{t=1}^{T-1}\frac{\eta_{t}^{2}\sum_{s=t}^{T-1}\frac{1}{\eta_{s}}}{(T-t)^{2}}+\sigma^{\p}D^{2-\p}\sum_{t=1}^{T-1}\frac{\eta_{t}^{\p}\sum_{s=t}^{T-1}\frac{1}{\eta_{s}}}{(T-t)^{2}}.\label{eq:OGD-cvx-last-2}
\end{equation}
We first bound
\[
\sum_{t=1}^{T-1}\frac{1}{\eta_{t}}=\sum_{t=1}^{T-1}2H\lor\frac{G\sqrt{t}}{D}\lor\frac{\sigma t^{1/\p}}{D}\leq\sum_{t=1}^{T-1}2H+\frac{G\sqrt{t}}{D}+\frac{\sigma t^{1/\p}}{D}\lesssim HT+\frac{G}{D}T^{3/2}+\frac{\sigma}{D}T^{1+1/\p},
\]
which implies
\begin{equation}
\frac{D^{2}}{(T-1)^{2}}\sum_{t=1}^{T-1}\frac{1}{\eta_{t}}\lesssim\frac{HD^{2}}{T}+\frac{GD}{\sqrt{T}}+\frac{\sigma D}{T^{1-\frac{1}{\p}}}.\label{eq:OGD-cvx-last-3}
\end{equation}
Next, we know
\begin{align*}
\sum_{t=1}^{T-1}\frac{\eta_{t}^{2}\sum_{s=t}^{T-1}\frac{1}{\eta_{s}}}{(T-t)^{2}} & \overset{(a)}{\leq}\sum_{t=1}^{T-1}\left[\frac{2HD^{2}}{G^{2}}\cdot\frac{1}{t(T-t)}+\frac{D}{G}\cdot\frac{\sum_{s=t}^{T-1}\sqrt{s}}{t(T-t)^{2}}+\frac{\sigma D}{G^{2}}\cdot\frac{\sum_{s=t}^{T-1}s^{1/\p}}{t(T-t)^{2}}\right]\\
 & \overset{\text{Fact }\ref{fact:last-ineq}}{\lesssim}\frac{HD^{2}\left(1+\log T\right)}{G^{2}T}+\frac{D\left(1+\log T\right)}{G\sqrt{T}}+\frac{\sigma D\left(1+\log T\right)}{G^{2}T^{1-\frac{1}{\p}}},
\end{align*}
where $(a)$ is by $\eta_{t}\leq\frac{D}{G\sqrt{t}}$ and $\frac{1}{\eta_{s}}\leq2H+\frac{G\sqrt{s}}{D}+\frac{\sigma s^{1/\p}}{D}$.
Hence, there is
\begin{equation}
G^{2}\sum_{t=1}^{T-1}\frac{\eta_{t}^{2}\sum_{s=t}^{T-1}\frac{1}{\eta_{s}}}{(T-t)^{2}}\lesssim\frac{HD^{2}\left(1+\log T\right)}{T}+\frac{GD\left(1+\log T\right)}{\sqrt{T}}+\frac{\sigma D\left(1+\log T\right)}{T^{1-\frac{1}{\p}}}.\label{eq:OGD-cvx-last-4}
\end{equation}
Similarly, we can bound
\begin{equation}
\sigma^{\p}D^{2-\p}\sum_{t=1}^{T-1}\frac{\eta_{t}^{\p}\sum_{s=t}^{T-1}\frac{1}{\eta_{s}}}{(T-t)^{2}}\lesssim\frac{HD^{2}\left(1+\log T\right)}{T}+\frac{GD\left(1+\log T\right)}{\sqrt{T}}+\frac{\sigma D\left(1+\log T\right)}{T^{1-\frac{1}{\p}}}.\label{eq:OGD-cvx-last-5}
\end{equation}
Finally, we plug (\ref{eq:OGD-cvx-last-3}), (\ref{eq:OGD-cvx-last-4})
and (\ref{eq:OGD-cvx-last-5}) back into (\ref{eq:OGD-cvx-last-2})
to conclude.

For $\eta_{t}=\frac{T-t}{2HT}\land\frac{D(T-t)}{GT^{3/2}}\land\frac{D(T-t)}{\sigma T^{1+1/\p}}$,
we denote by $\eta_{t}=\eta(T-t)$ for $\eta\defeq\frac{1}{2HT}\land\frac{D}{GT^{3/2}}\land\frac{D}{\sigma T^{1+1/\p}}$
to have
\begin{align*}
\E\left[F(\bx_{T})-F(\bx)\right] & \overset{(\ref{eq:OGD-cvx-last-1})}{\lesssim}\frac{D^{2}}{\eta\sum_{t=1}^{T-1}T-t}+\eta G^{2}\sum_{t=1}^{T-1}\frac{(T-t)^{2}}{\sum_{s=t}^{T-1}T-s}+\eta^{\p-1}\sigma^{\p}D^{2-\p}\sum_{t=1}^{T-1}\frac{(T-t)^{\p}}{\sum_{s=t}^{T-1}T-s}\\
 & =\frac{2D^{2}}{\eta T(T-1)}+2\eta G^{2}\sum_{t=1}^{T-1}\frac{T-t}{T-t+1}+2\eta^{\p-1}\sigma^{\p}D^{2-\p}\sum_{t=1}^{T-1}\frac{(T-t)^{\p-1}}{T-t+1}\\
 & \lesssim\frac{D^{2}}{\eta T^{2}}+\eta G^{2}T+\eta^{\p-1}\sigma^{\p}D^{2-\p}T^{\p-1}\lesssim\frac{HD^{2}}{T}+\frac{GD}{\sqrt{T}}+\frac{\sigma D}{T^{1-\frac{1}{\p}}}.
\end{align*}
\end{proof}

Next, Corollary \ref{cor:OGD-str-last} provides the first last-iterate
convergence results in the strongly convex case. In particular, when
$H=0\Rightarrow\kappa=0$, the first stepsize degenerates to $\eta_{t}=\frac{1}{\mu t}$,
as used in Theorem \ref{thm:OGD-str}, and yields a last-iterate rate
$\frac{G^{2}\left(1+\log T\right)}{\mu T}+\frac{\sigma^{\p}D^{2-\p}\left(1+\log T\right)}{\mu^{\p-1}T^{\p-1}}$,
matching the average-iterate rate (up to an extra logarithmic factor)
implied by the online-to-batch conversion for Theorem \ref{thm:OGD-str}.
\begin{cor}
\label{cor:OGD-str-last}Under Assumption \ref{assu:OCO} (without
the third point) and additionally assuming (\ref{eq:curvature}) (with
$\mu>0$) for $\ell_{t}(\bx)=F(\bx)$, let $\kappa\defeq\frac{H}{\mu}$:
\begin{itemize}
\item taking $\eta_{t}=\frac{1}{\mu(t+2\kappa)}$ in $\OGD$ (Algorithm
\ref{alg:OGD}), we have
\[
\E\left[F(\bx_{T})-F(\bx)\right]\lesssim\frac{HD^{2}}{T}+\frac{G^{2}\left(1+\log T\right)}{\mu\left(T+\kappa\right)}+\frac{\sigma^{\p}D^{2-\p}\left(1+\log T\right)}{\mu^{\p-1}\left(T+\kappa\right)^{\p-1}},\forall\bx\in\X.
\]
\item taking $\eta_{t}=\begin{cases}
\frac{1}{\mu(1+2\kappa)} & t\leq\tau\\
\frac{2}{\mu(t-\tau+2+4\kappa)} & t\geq\tau+1
\end{cases}$ where $\tau\defeq\left\lceil \frac{T}{2}\right\rceil $ in $\OGD$
(Algorithm \ref{alg:OGD}), we have
\[
\E\left[F(\bx_{T})-F(\bx)\right]\lesssim\frac{HD^{2}}{\exp\left(\frac{T}{2+4\kappa}\right)}+\frac{G^{2}\left(1+\log T\right)}{\mu\left(T+\kappa\right)}+\frac{\sigma^{\p}D^{2-\p}\left(1+\log T\right)}{\mu^{\p-1}\left(T+\kappa\right)^{\p-1}},\forall\bx\in\X.
\]
\end{itemize}
\end{cor}
\begin{proof}
Since both kinds of stepsize satisfy $\eta_{t}\leq\frac{1}{2H\lor\mu}$,
by Theorem \ref{thm:OGD-last-core}, we have
\begin{equation}
\E\left[F(\bx_{T})-F(\bx)\right]\lesssim\frac{(1-\mu\eta_{1})D^{2}}{\sum_{t=1}^{T-1}\gamma_{t}}+G^{2}\sum_{t=1}^{T-1}\frac{\gamma_{t}\eta_{t}}{\sum_{s=t}^{T-1}\gamma_{s}}+\sigma^{\p}D^{2-\p}\sum_{t=1}^{T-1}\frac{\gamma_{t}\eta_{t}^{\p-1}}{\sum_{s=t}^{T-1}\gamma_{s}}.\label{eq:OGD-str-last-1}
\end{equation}

For $\eta_{t}=\frac{1}{\mu(t+2\kappa)}$, we can find
\begin{equation}
\gamma_{t}=\frac{\eta_{t}}{\prod_{s=2}^{t}(1-\mu\eta_{s})}=\frac{1}{\mu(t+2\kappa)\prod_{s=2}^{t}\frac{s-1+2\kappa}{s+2\kappa}}=\frac{1}{\mu(1+2\kappa)}=\eta_{1}.\label{eq:OGD-str-last-gamma-1}
\end{equation}
Hence, there is
\begin{align}
\E\left[F(\bx_{T})-F(\bx)\right] & \overset{(\ref{eq:OGD-str-last-1}),(\ref{eq:OGD-str-last-gamma-1})}{\lesssim}\frac{(\eta_{1}^{-1}-\mu)D^{2}}{T-1}+G^{2}\sum_{t=1}^{T-1}\frac{\eta_{t}}{T-t}+\sigma^{\p}D^{2-\p}\sum_{t=1}^{T-1}\frac{\eta_{t}^{\p-1}}{T-t}\nonumber \\
 & =\frac{2HD^{2}}{T-1}+\frac{G^{2}}{\mu}\sum_{t=1}^{T-1}\frac{1}{(T-t)(t+2\kappa)}+\frac{\sigma^{\p}D^{2-\p}}{\mu^{\p-1}}\sum_{t=1}^{T-1}\frac{1}{(T-t)(t+2\kappa)^{\p-1}}.\label{eq:OGD-str-last-2}
\end{align}
Given $a\in\left(0,1\right]$, we can bound
\[
\sum_{t=1}^{T-1}\frac{1}{(T-t)(t+2\kappa)^{a}}=\frac{1}{T+2\kappa}\sum_{t=1}^{T-1}\frac{(t+2\kappa)^{1-a}}{T-t}+\frac{1}{(t+2\kappa)^{a}}\lesssim\frac{1+\log T}{(T+2\kappa)^{a}},
\]
where the last step is due to
\begin{eqnarray*}
\sum_{t=1}^{T-1}\frac{(t+2\kappa)^{1-a}}{T-t}\lesssim(T+2\kappa)^{1-a}(1+\log T) & \text{and} & \sum_{t=1}^{T-1}\frac{1}{(t+2\kappa)^{a}}\lesssim\begin{cases}
T^{1-a} & a\in\left(0,1\right)\\
1+\log T & a=1
\end{cases}.
\end{eqnarray*}
Finally, we apply the above inequality to (\ref{eq:OGD-str-last-2})
with $a=1$ and $a=\p-1$ to obtain
\[
\E\left[F(\bx_{T})-F(\bx)\right]\lesssim\frac{HD^{2}}{T}+\frac{G^{2}\left(1+\log T\right)}{\mu\left(T+\kappa\right)}+\frac{\sigma^{\p}D^{2-\p}\left(1+\log T\right)}{\mu^{\p-1}\left(T+\kappa\right)^{\p-1}}.
\]

For $\eta_{t}=\begin{cases}
\frac{1}{\mu(1+2\kappa)} & t\leq\tau\\
\frac{2}{\mu(t-\tau+2+4\kappa)} & t\geq\tau+1
\end{cases}$, let $\eta\defeq\frac{1}{\mu(1+2\kappa)}$, we can find
\begin{equation}
\gamma_{t}=\frac{\eta_{t}}{\prod_{s=2}^{t}(1-\mu\eta_{s})}=\begin{cases}
\frac{\eta}{(1-\mu\eta)^{t-1}} & t\leq\tau\\
\frac{\eta(t-\tau+1+4\kappa)}{(1-\mu\eta)^{\tau-1}(1+4\kappa)} & t\geq\tau+1
\end{cases}.\label{eq:OGD-str-last-gamma-2}
\end{equation}
Hence, there is
\begin{align*}
\frac{1-\mu\eta_{1}}{\sum_{t=1}^{T-1}\gamma_{t}} & \leq\frac{1-\mu\eta_{1}}{\sum_{t=1}^{\tau}\gamma_{t}}=\frac{1-\mu\eta}{\sum_{t=1}^{\tau}\frac{\eta}{(1-\mu\eta)^{t-1}}}=\frac{(1-\mu\eta)^{\tau}}{\eta\sum_{t=1}^{\tau}(1-\mu\eta)^{\tau-t}}=\frac{\mu(1-\mu\eta)^{\tau}}{1-(1-\mu\eta)^{\tau}}\\
 & \overset{(1-\mu\eta)^{\tau}\leq1-\mu\eta}{\leq}\frac{(1-\mu\eta)^{\tau}}{\eta}=2H(1-\mu\eta)^{\tau-1}\leq2H\exp\left(-\mu\eta(\tau-1)\right)\\
 & =2H\exp\left(\frac{-\tau+1}{1+2\kappa}\right)\overset{\kappa\geq0}{\leq}2eH\exp\left(\frac{-\tau}{1+2\kappa}\right)\overset{\tau\geq\frac{T}{2}}{\lesssim}\frac{H}{\exp\left(\frac{T}{2+4\kappa}\right)},
\end{align*}
implying that
\begin{equation}
\E\left[F(\bx_{T})-F(\bx)\right]\overset{(\ref{eq:OGD-str-last-1})}{\lesssim}\frac{HD^{2}}{\exp\left(\frac{T}{2+4\kappa}\right)}+G^{2}\sum_{t=1}^{T-1}\frac{\gamma_{t}\eta_{t}}{\sum_{s=t}^{T-1}\gamma_{s}}+\sigma^{\p}D^{2-\p}\sum_{t=1}^{T-1}\frac{\gamma_{t}\eta_{t}^{\p-1}}{\sum_{s=t}^{T-1}\gamma_{s}}.\label{eq:OGD-str-last-3}
\end{equation}
Given $a\in\left(0,1\right]$, we can write
\begin{equation}
\sum_{t=1}^{T-1}\frac{\gamma_{t}\eta_{t}^{a}}{\sum_{s=t}^{T-1}\gamma_{s}}=\sum_{t=1}^{\tau}\frac{\gamma_{t}\eta_{t}^{a}}{\sum_{s=t}^{T-1}\gamma_{s}}+\sum_{t=\tau+1}^{T-1}\frac{\gamma_{t}\eta_{t}^{a}}{\sum_{s=t}^{T-1}\gamma_{s}}=\sum_{t=1}^{\tau}\frac{\gamma_{t}\eta^{a}}{\sum_{s=t}^{T-1}\gamma_{s}}+\sum_{t=\tau+1}^{T-1}\frac{\gamma_{t}\eta_{t}^{a}}{\sum_{s=t}^{T-1}\gamma_{s}}.\label{eq:OGD-str-last-4}
\end{equation}
When $t\leq\tau$, we know
\[
\frac{\gamma_{t}}{\sum_{s=t}^{T-1}\gamma_{s}}\leq\frac{\gamma_{t}}{\sum_{s=\tau}^{T-1}\gamma_{s}}\overset{(\ref{eq:OGD-str-last-gamma-2})}{=}\frac{(1+4\kappa)(1-\mu\eta)^{\tau-t}}{\sum_{s=\tau}^{T-1}s-\tau+1+4\kappa}=\frac{2(1+4\kappa)(1-\mu\eta)^{\tau-t}}{(T-\tau)(T-\tau+1+8\kappa)}\leq\frac{2(1+4\kappa)}{(T-\tau)(T-\tau+1+8\kappa)},
\]
which implies
\begin{equation}
\sum_{t=1}^{\tau}\frac{\gamma_{t}\eta^{a}}{\sum_{s=t}^{T-1}\gamma_{s}}\leq\frac{2(1+4\kappa)\tau}{\mu^{a}(1+2\kappa)^{a}(T-\tau)(T-\tau+1+8\kappa)}\lesssim\frac{(1+\kappa)^{1-a}}{\mu^{a}(T+\kappa)}\leq\frac{1}{\mu^{a}(T+\kappa)^{a}}.\label{eq:OGD-str-last-5}
\end{equation}
When $t\geq\tau+1$, we know
\begin{align*}
\frac{\gamma_{t}}{\sum_{s=t}^{T-1}\gamma_{s}} & \overset{(\ref{eq:OGD-str-last-gamma-2})}{\leq}\frac{t-\tau+1+4\kappa}{\sum_{s=t}^{T-1}s-\tau+1+4\kappa}=\frac{2(t-\tau+1+4\kappa)}{(T-t)(T+t-2\tau+1+8\kappa)}\\
 & =\frac{2}{2T-2\tau+1+8\kappa}\left(\frac{t-\tau+1+4\kappa}{T-t}+\frac{t-\tau+1+4\kappa}{T+t-2\tau+1+8\kappa}\right)\\
 & \overset{\tau\leq\frac{T+1}{2},t\geq\tau+1}{\lesssim}\frac{1}{T+\kappa}\left(\frac{t-\tau+1+4\kappa}{T-t}+\frac{t-\tau+1+4\kappa}{T-\tau+2+8\kappa}\right),
\end{align*}
which implies
\begin{align}
\sum_{t=\tau+1}^{T-1}\frac{\gamma_{t}\eta_{t}^{a}}{\sum_{s=t}^{T-1}\gamma_{s}} & \lesssim\frac{1}{\mu^{a}(T+\kappa)}\sum_{t=\tau+1}^{T-1}\frac{(t-\tau+1+4\kappa)^{1-a}}{T-t}+\frac{(t-\tau+1+4\kappa)^{1-a}}{T-\tau+2+8\kappa}\nonumber \\
 & \lesssim\frac{1}{\mu^{a}(T+\kappa)}\left((T+\kappa)^{1-a}(1+\log T)+(T+\kappa)^{1-a}\right)\lesssim\frac{1+\log T}{\mu^{a}(T+\kappa)^{a}}.\label{eq:OGD-str-last-6}
\end{align}
Plug (\ref{eq:OGD-str-last-5}) and (\ref{eq:OGD-str-last-6}) back
into (\ref{eq:OGD-str-last-4}) to have
\[
\sum_{t=1}^{T-1}\frac{\gamma_{t}\eta_{t}^{a}}{\sum_{s=t}^{T-1}\gamma_{s}}\lesssim\frac{1+\log T}{\mu^{a}(T+\kappa)^{a}}.
\]
Applying the above inequality (for $a=1$ and $a=\p-1$) to (\ref{eq:OGD-str-last-3}),
we conclude
\[
\E\left[F(\bx_{T})-F(\bx)\right]\lesssim\frac{HD^{2}}{\exp\left(\frac{T}{2+4\kappa}\right)}+\frac{G^{2}\left(1+\log T\right)}{\mu\left(T+\kappa\right)}+\frac{\sigma^{\p}D^{2-\p}\left(1+\log T\right)}{\mu^{\p-1}\left(T+\kappa\right)^{\p-1}}.
\]
\end{proof}

\section{Missing Proofs for Applications: Nonsmooth Nonconvex Optimization\label{sec:ncvx}}

\subsection{$(\delta,\epsilon)$-Stationary Points}
\begin{defn}[Definition 4 of \citep{pmlr-v202-cutkosky23a}]
\label{def:delta-eps-stationary}A point $\bx\in\R^{d}$ is a $(\delta,\epsilon)$-stationary
point of an almost-everywhere differentiable function $F$ if there
is a finite subset $S\subset\B^{d}(\bx,\delta)$ such that for $\by$
selected uniformly at random from $S$, $\E\left[\by\right]=\bx$
and $\left\Vert \E\left[\nabla F(\by)\right]\right\Vert \leq\epsilon$.
\end{defn}
The concept of the $(\delta,\epsilon)$-stationary point presented
here is due to \citep{pmlr-v202-cutkosky23a}, which is mildly more
stringent than the notion of \citep{pmlr-v119-zhang20p}, since the
latter does not require $\E\left[\by\right]=\bx$. For more discussions,
see Section 2.1 of \citep{pmlr-v202-cutkosky23a}.

\subsection{Proof of Theorem \ref{thm:main-ncvx-core}}

In this section, our ultimate goal is to prove Theorem \ref{thm:main-ncvx-core}
for the $\otnc$ algorithm, extending Theorem 8 of \citep{pmlr-v202-cutkosky23a}
from $\p=2$ to any $\p\in\left(1,2\right]$. Notably, our new result
does not require any modification to the $\otnc$ method, but is obtained
only from a more careful analysis, indicating that $\otnc$ is a robust
and powerful algorithmic framework.

We begin with Lemma \ref{lem:ncvx-basic}, which lies as the cornerstone
for establishing the convergence of $\otnc$.
\begin{lem}[Theorem 7 of \citep{pmlr-v202-cutkosky23a}]
\label{lem:ncvx-basic}Under Assumption \ref{assu:ncvx} (only need
the second point and the unbiased part in the fourth point), for any
sequence of vectors $\bu_{1},\mydots,\bu_{KT}\in\R^{d}$, $\otnc$
(Algorithm \ref{alg:O2NC}) guarantees
\begin{equation}
\E\left[F(\by_{KT})\right]=F(\by_{0})+\E\left[\sum_{n=1}^{KT}\left\langle \bg_{n},\bx_{n}-\bu_{n}\right\rangle \right]+\E\left[\sum_{n=1}^{KT}\left\langle \bg_{n},\bu_{n}\right\rangle \right].\label{eq:ncvx-basic}
\end{equation}
\end{lem}
To relate Lemma \ref{lem:ncvx-basic} to the concept of $K$-shifting
regret introduced before (see (\ref{eq:ncvx-shift})), suppose now
a sequence of vectors $\bv_{1},\mydots,\bv_{K}$ is given, if we set
$\bu_{n}=\bv_{k}$ for all $n\in\left\{ (k-1)T+1,\mydots,kT\right\} $
and $k\in\left[K\right]$, then the second term on the R.H.S. of (\ref{eq:ncvx-basic})
can be written as $\E\left[\reg_{T}^{\A}(\bv_{1},\mydots,\bv_{K})\right]$,
and the third term can be simplified into $\sum_{k=1}^{K}\E\left[\left\langle \sum_{n=(k-1)T+1}^{kT}\bg_{n},\bv_{k}\right\rangle \right]$. 

Same as \citep{pmlr-v202-cutkosky23a}, we pick $\bv_{k}\defeq-D\frac{\sum_{n=(k-1)T+1}^{kT}\nabla F(\bz_{n})}{\left\Vert \sum_{n=(k-1)T+1}^{kT}\nabla F(\bz_{n})\right\Vert }$
for some constant $D>0$, which gives us
\begin{align*}
\E\left[\left\langle \sum_{n=(k-1)T+1}^{kT}\bg_{n},\bv_{k}\right\rangle \right] & =\E\left[\left\langle \sum_{n=(k-1)T+1}^{kT}\err_{n},\bv_{k}\right\rangle \right]-D\E\left[\left\Vert \sum_{n=(k-1)T+1}^{kT}\nabla F(\bz_{n})\right\Vert \right]\\
 & \leq D\E\left[\left\Vert \sum_{n=(k-1)T+1}^{kT}\err_{n}\right\Vert \right]-D\E\left[\left\Vert \sum_{n=(k-1)T+1}^{kT}\nabla F(\bz_{n})\right\Vert \right].
\end{align*}

If $\err_{n}$ has a finite variance (i.e., $\p=2$), then like \citep{pmlr-v202-cutkosky23a},
one can invoke H\"{o}lder's inequality and use the fact $\E\left[\left\langle \err_{m},\err_{n}\right\rangle \right]=0,\forall m\neq n\in\left[KT\right]$
to obtain for any $k\in\left[K\right]$,
\[
\E\left[\left\Vert \sum_{n=(k-1)T+1}^{kT}\err_{n}\right\Vert \right]\leq\sqrt{\E\left[\left\Vert \sum_{n=(k-1)T+1}^{kT}\err_{n}\right\Vert ^{2}\right]}=\sqrt{\sum_{n=(k-1)T+1}^{kT}\E\left[\left\Vert \err_{n}\right\Vert ^{2}\right]}\leq\sigma\sqrt{T}.
\]
However, this argument immediately fails when $\p<2$ as $\E\left[\left\Vert \err_{n}\right\Vert ^{2}\right]$
can be $+\infty$. To handle this potential issue, we require the
following Lemma \ref{lem:ncvx-martingale}.
\begin{lem}[Lemma 4.3 of \citep{liu2025nonconvex}]
\label{lem:ncvx-martingale}Given a vector-valued martingale difference
sequence $\bw_{1},\mydots,\bw_{T}$, there is
\[
\E\left[\left\Vert \sum_{t=1}^{T}\bw_{t}\right\Vert \right]\le2\sqrt{2}\E\left[\left(\sum_{t=1}^{T}\left\Vert \bw_{t}\right\Vert ^{\p}\right)^{\frac{1}{\p}}\right],\forall\p\in\left[1,2\right].
\]
\end{lem}
Equipped with Lemmas \ref{lem:ncvx-basic} and \ref{lem:ncvx-martingale},
we are ready to formally prove Theorem \ref{thm:main-ncvx-core},
demonstrating that the $\otnc$ framework provably works under heavy-tailed
noise.

\begin{proof}[Proof of Theorem \ref{thm:main-ncvx-core}]
We invoke Lemma \ref{lem:ncvx-basic} with $\bu_{n}=\bv_{\left\lceil n/T\right\rceil },\forall n\in\left[KT\right]$
(equivalently, $\bu_{n}=\bv_{k}$ if $n\in\left\{ (k-1)T+1,\mydots,kT\right\} $)
and use the definition of $K$-shifting regret (see (\ref{eq:ncvx-shift}))
to obtain
\begin{equation}
\E\left[F(\by_{KT})\right]=F(\by_{0})+\E\left[\reg_{T}^{\A}\left(\bv_{1},\mydots,\bv_{K}\right)\right]+\sum_{k=1}^{K}\E\left[\left\langle \sum_{n=(k-1)T+1}^{kT}\bg_{n},\bv_{k}\right\rangle \right].\label{eq:ncvx-core-1}
\end{equation}
Recall that $\bg_{n}=\nabla F(\bz_{n})+\err_{n}$, which implies for
any $k\in\left[K\right]$,
\begin{align}
\E\left[\left\langle \sum_{n=(k-1)T+1}^{kT}\bg_{n},\bv_{k}\right\rangle \right] & =\E\left[\left\langle \sum_{n=(k-1)T+1}^{kT}\err_{n},\bv_{k}\right\rangle \right]+\E\left[\left\langle \sum_{n=(k-1)T+1}^{kT}\nabla F(\bz_{n}),\bv_{k}\right\rangle \right]\nonumber \\
 & \leq\E\left[\left\Vert \sum_{n=(k-1)T+1}^{kT}\err_{n}\right\Vert \left\Vert \bv_{k}\right\Vert \right]+\E\left[\left\langle \sum_{n=(k-1)T+1}^{kT}\nabla F(\bz_{n}),\bv_{k}\right\rangle \right]\nonumber \\
 & =D\E\left[\left\Vert \sum_{n=(k-1)T+1}^{kT}\err_{n}\right\Vert \right]-D\E\left[\left\Vert \sum_{n=(k-1)T+1}^{kT}\nabla F(\bz_{n})\right\Vert \right],\label{eq:ncvx-core-2}
\end{align}
where the second step is by Cauchy-Schwarz inequality and the last
equation holds due to 
\begin{equation}
\bv_{k}=-D\frac{\sum_{n=(k-1)T+1}^{kT}\nabla F(\bz_{n})}{\left\Vert \sum_{n=(k-1)T+1}^{kT}\nabla F(\bz_{n})\right\Vert },\forall k\in\left[K\right].\label{eq:ncvx-core-v}
\end{equation}

Combine (\ref{eq:ncvx-core-1}) and (\ref{eq:ncvx-core-2}), apply
$F(\by_{KT})\geq F_{\star}$, and rearrange terms to have
\begin{align}
 & \E\left[\sum_{k=1}^{K}\frac{1}{K}\left\Vert \frac{1}{T}\sum_{n=(k-1)T+1}^{kT}\nabla F(\bz_{n})\right\Vert \right]\nonumber \\
\leq & \frac{F(\by_{0})-F_{\star}}{DKT}+\frac{\E\left[R_{T}^{\A}\left(\bv_{1},\mydots,\bv_{K}\right)\right]}{DKT}+\frac{\sum_{k=1}^{K}\E\left[\left\Vert \sum_{n=(k-1)T+1}^{kT}\err_{n}\right\Vert \right]}{KT}.\label{eq:ncvx-core-3}
\end{align}
For any fixed $k\in\left[K\right]$, we apply Lemma \ref{lem:ncvx-martingale}
with $\bw_{t}=\err_{(k-1)T+t},\forall t\in\left[T\right]$ to know
\begin{align}
\E\left[\left\Vert \sum_{n=(k-1)T+1}^{kT}\err_{n}\right\Vert \right] & \leq2\sqrt{2}\E\left[\left(\sum_{n=(k-1)T+1}^{kT}\left\Vert \err_{n}\right\Vert ^{\p}\right)^{\frac{1}{\p}}\right]\nonumber \\
 & \leq2\sqrt{2}\left(\sum_{n=(k-1)T+1}^{kT}\E\left[\left\Vert \err_{n}\right\Vert ^{\p}\right]\right)^{\frac{1}{\p}}\leq2\sqrt{2}\sigma T^{\frac{1}{\p}},\label{eq:ncvx-core-4}
\end{align}
where the second step is by H\"{o}lder's inequality (note that $\p>1$).
Finally, we conclude the proof after plugging (\ref{eq:ncvx-core-4})
back into (\ref{eq:ncvx-core-3}).
\end{proof}

\subsection{Proof of Theorem \ref{thm:main-ncvx-general}}

\begin{proof}
By Theorem \ref{thm:main-ncvx-core}, there is
\begin{equation}
\E\left[\sum_{k=1}^{K}\frac{1}{K}\left\Vert \frac{1}{T}\sum_{n=(k-1)T+1}^{kT}\nabla F(\bz_{n})\right\Vert \right]\lesssim\frac{F(\by_{0})-F_{\star}}{DKT}+\frac{\E\left[\reg_{T}^{\A}(\bv_{1},\mydots,\bv_{K})\right]}{DKT}+\frac{\sigma}{T^{1-\frac{1}{\p}}}.\label{eq:ncvx-general-1}
\end{equation}
Note that $\A$ has the domain $\X=\B^{d}(D)$ and $s_{n}\sim\uni\left[0,1\right]$.
Thus, for any $n\in\left[KT\right]$,
\begin{eqnarray}
\left\Vert \bx_{n}\right\Vert \leq D & \text{and} & s_{n}\in\left[0,1\right].\label{eq:ncvx-general-facts}
\end{eqnarray}

We first lower bound the L.H.S. of (\ref{eq:ncvx-general-1}). Given
$k\in\left[K\right]$, for any $m<n\in\left\{ (k-1)T+1,\mydots,kT\right\} $,
observe that
\begin{align*}
\left\Vert \bz_{n}-\bz_{m}\right\Vert  & =\left\Vert \by_{n-1}+s_{n}\bx_{n}-\by_{m-1}-s_{m}\bx_{m}\right\Vert =\left\Vert s_{n}\bx_{n}-s_{m}\bx_{m}+\sum_{i=m}^{n-1}\bx_{i}\right\Vert \\
 & \leq s_{n}\left\Vert \bx_{n}\right\Vert +\left(1-s_{m}\right)\left\Vert \bx_{m}\right\Vert +\sum_{i=m+1}^{n-1}\left\Vert \bx_{i}\right\Vert \overset{(\ref{eq:ncvx-general-facts})}{\leq}\left(n-m+1\right)D\leq DT.
\end{align*}
Recall that $\bar{\bz}_{k}=\frac{1}{T}\sum_{n=(k-1)T+1}^{kT}\bz_{n}$
and $D=\delta/T$ now, then the above inequality implies
\begin{equation}
\left\Vert \bz_{n}-\bar{\bz}_{k}\right\Vert \leq DT=\delta,\forall n\in\left\{ (k-1)T+1,\mydots,kT\right\} ,\label{eq:ncvx-general-ball}
\end{equation}
which means
\[
\bz_{n}\in\B^{d}(\bar{\bz}_{k},\delta),\forall n\in\left\{ (k-1)T+1,\mydots,kT\right\} .
\]
By the definition of $\left\Vert \nabla F(\bar{\bz}_{k})\right\Vert _{\delta}$
(see Definition \ref{def:delta-norm}), there is
\begin{equation}
\left\Vert \nabla F(\bar{\bz}_{k})\right\Vert _{\delta}\leq\left\Vert \frac{1}{T}\sum_{n=(k-1)T+1}^{kT}\nabla F(\bz_{n})\right\Vert .\label{eq:ncvx-general-lhs}
\end{equation}

Next, we upper bound the R.H.S. of (\ref{eq:ncvx-general-1}). By
the definition of $K$-shifting regret (see (\ref{eq:ncvx-shift})),
there is
\[
\E\left[\reg_{T}^{\A}(\bv_{1},\mydots,\bv_{K})\right]=\sum_{k=1}^{K}\E\left[\sum_{n=(k-1)T+1}^{kT}\left\langle \bg_{n},\bx_{n}-\bv_{k}\right\rangle \right].
\]
Note that we reset the stepsize in $\A$ after every $T$ iterations
and $\bv_{k}\in\B^{d}(D)$ by its definition (see (\ref{eq:ncvx-core-v})).
Then for any $\A\in\left\{ \OGD,\DA,\Ada\right\} $, we can invoke
its regret bound\footnote{A minor point here is that the current function $\ell_{n}(\bx)=\left\langle \bg_{n},\bx\right\rangle $
does not entirely fit Assumption \ref{assu:OCO}. We clarify that
one does not need to worry about it, since all results proved in Section
\ref{sec:OCO} hold under this change. For example, in the proof of
Theorem \ref{thm:main-OGD}, we can safely replace the L.H.S. of (\ref{eq:OGD-cvx-3})
with $\E\left[\sum_{t=1}^{T}\left\langle \bg_{t},\bx_{t}-\bx\right\rangle \right]$.} (i.e., Theorems \ref{thm:main-OGD}, \ref{thm:main-DA} and \ref{thm:main-AdaGrad})
to obtain
\[
\E\left[\sum_{n=(k-1)T+1}^{kT}\left\langle \bg_{n},\bx_{n}-\bv_{k}\right\rangle \right]\lesssim GD\sqrt{T}+\sigma DT^{1/\p},\forall k\in\left[K\right],
\]
which implies
\begin{equation}
\E\left[\reg_{T}^{\A}(\bv_{1},\mydots,\bv_{K})\right]\lesssim GDK\sqrt{T}+\sigma DKT^{1/\p}.\label{eq:ncvx-general-rhs}
\end{equation}

Finally, we plug (\ref{eq:ncvx-general-lhs}) and (\ref{eq:ncvx-general-rhs})
back into (\ref{eq:ncvx-general-1}), then use $D=\delta/T$ and $\Delta=F(\by_{0})-F_{\star}$
to have
\[
\E\left[\frac{1}{K}\sum_{k=1}^{K}\left\Vert \nabla F(\bar{\bz}_{k})\right\Vert _{\delta}\right]\lesssim\frac{\Delta}{\delta K}+\frac{G}{\sqrt{T}}+\frac{\sigma}{T^{1-\frac{1}{\p}}}.
\]
\end{proof}

\subsection{Proof of Corollary \ref{cor:main-ncvx-dep}}

\begin{proof}
Recall that we pick
\begin{eqnarray*}
K=\left\lfloor \frac{N}{T}\right\rfloor  & \text{and} & T=\left\lceil \frac{N}{2}\right\rceil \land\left(\left\lceil \left(\frac{\delta GN}{\Delta}\right)^{\frac{2}{3}}\right\rceil \lor\left\lceil \left(\frac{\delta\sigma N}{\Delta}\right)^{\frac{\p}{2\p-1}}\right\rceil \right),
\end{eqnarray*}
where $\Delta=F(\by_{0})-F_{\star}$. We invoke Theorem \ref{thm:main-ncvx-general}
and use $KT\geq N/4$ (see Fact \ref{fact:N/2}) to obtain
\[
\E\left[\frac{1}{K}\sum_{k=1}^{K}\left\Vert \nabla F(\bar{\bz}_{k})\right\Vert _{\delta}\right]\lesssim\frac{\Delta T}{\delta N}+\frac{G}{\sqrt{T}}+\frac{\sigma}{T^{1-\frac{1}{\p}}}.
\]
By the definition of $T$, we know
\[
\frac{\Delta T}{\delta N}\lesssim\frac{\Delta}{\delta N}\left[1+\left(\frac{\delta GN}{\Delta}\right)^{\frac{2}{3}}+\left(\frac{\delta\sigma N}{\Delta}\right)^{\frac{\p}{2\p-1}}\right]=\frac{\Delta}{\delta N}+\frac{G^{\frac{2}{3}}\Delta^{\frac{1}{3}}}{(\delta N)^{\frac{1}{3}}}+\frac{\sigma^{\frac{\p}{2\p-1}}\Delta^{\frac{\p-1}{2\p-1}}}{(\delta N)^{\frac{\p-1}{2\p-1}}},
\]
and
\begin{eqnarray*}
\frac{G}{\sqrt{T}}\lesssim\frac{G}{\sqrt{N}}+\frac{G^{\frac{2}{3}}\Delta^{\frac{1}{3}}}{(\delta N)^{\frac{1}{3}}}, &  & \frac{\sigma}{T^{1-\frac{1}{\p}}}\lesssim\frac{\sigma}{N^{1-\frac{1}{\p}}}+\frac{\sigma^{\frac{\p}{2\p-1}}\Delta^{\frac{\p-1}{2\p-1}}}{(\delta N)^{\frac{\p-1}{2\p-1}}}.
\end{eqnarray*}
Therefore, there is
\[
\E\left[\frac{1}{K}\sum_{k=1}^{K}\left\Vert \nabla F(\bar{\bz}_{k})\right\Vert _{\delta}\right]\lesssim\frac{G}{\sqrt{N}}+\frac{\sigma}{N^{1-\frac{1}{\p}}}+\frac{\Delta}{\delta N}+\frac{G^{\frac{2}{3}}\Delta^{\frac{1}{3}}}{(\delta N)^{\frac{1}{3}}}+\frac{\sigma^{\frac{\p}{2\p-1}}\Delta^{\frac{\p-1}{2\p-1}}}{(\delta N)^{\frac{\p-1}{2\p-1}}}.
\]
\end{proof}

\subsection{Extension to the Case of Unknown Problem-Dependent Parameters}

In Corollary \ref{cor:ncvx-free}, we show how to set $K$ and $T$
when all problem-dependent parameters are unknown. It is particularly
meaningful for $\Ada$. As in that case, the rate is achieved without
knowing any problem-dependent parameter. This kind of result is the
first to appear for nonsmooth nonconvex optimization with heavy tails.
However, the rate is not as good as Corollary \ref{cor:main-ncvx-dep}.
It is currently unclear whether the same bound $1/(\delta N)^{\frac{\p-1}{2\p-1}}$
as in Corollary \ref{cor:main-ncvx-dep} can be obtained when no information
about the problem is known.
\begin{cor}
\label{cor:ncvx-free}Under the same setting of Theorem \ref{thm:main-ncvx-general},
suppose we have $N\geq2$ stochastic gradient budgets, taking $K=\left\lfloor N/T\right\rfloor $
and $T=\left\lceil N/2\right\rceil \land\left\lceil (\delta N)^{\frac{2}{3}}\right\rceil $,
we have
\[
\E\left[\frac{1}{K}\sum_{k=1}^{K}\left\Vert \nabla F(\bar{\bz}_{k})\right\Vert _{\delta}\right]\lesssim\frac{\Delta}{(\delta N)\land(\delta N)^{\frac{1}{3}}}+\frac{G}{\sqrt{N}\land(\delta N)^{\frac{1}{3}}}+\frac{\sigma}{N^{1-\frac{1}{\p}}\land(\delta N)^{\frac{2(\p-1)}{3\p}}}.
\]
\end{cor}
\begin{proof}
We invoke Theorem \ref{thm:main-ncvx-general} and use $KT\geq N/4$
(see Fact \ref{fact:N/2}) to obtain
\[
\E\left[\frac{1}{K}\sum_{k=1}^{K}\left\Vert \nabla F(\bar{\bz}_{k})\right\Vert _{\delta}\right]\lesssim\frac{\Delta T}{\delta N}+\frac{G}{\sqrt{T}}+\frac{\sigma}{T^{1-\frac{1}{\p}}}.
\]
By the definition of $T$, we know
\[
\frac{\Delta T}{\delta N}\lesssim\frac{\Delta}{\delta N}\left[1+(\delta N)^{\frac{2}{3}}\right]\lesssim\frac{\Delta}{(\delta N)\land(\delta N)^{\frac{1}{3}}}.
\]
and
\begin{align*}
\frac{G}{\sqrt{T}} & \lesssim\frac{G}{\sqrt{N}}+\frac{G}{(\delta N)^{\frac{1}{3}}}\lesssim\frac{G}{\sqrt{N}\land(\delta N)^{\frac{1}{3}}},\\
\frac{\sigma}{T^{1-\frac{1}{\p}}} & \lesssim\frac{\sigma}{N^{1-\frac{1}{\p}}}+\frac{\sigma}{(\delta N)^{\frac{2(\p-1)}{3\p}}}\lesssim\frac{\sigma}{N^{1-\frac{1}{\p}}\land(\delta N)^{\frac{2(\p-1)}{3\p}}}.
\end{align*}
Therefore, there is
\[
\E\left[\frac{1}{K}\sum_{k=1}^{K}\left\Vert \nabla F(\bar{\bz}_{k})\right\Vert _{\delta}\right]\lesssim\frac{\Delta}{(\delta N)\land(\delta N)^{\frac{1}{3}}}+\frac{G}{\sqrt{N}\land(\delta N)^{\frac{1}{3}}}+\frac{\sigma}{N^{1-\frac{1}{\p}}\land(\delta N)^{\frac{2(\p-1)}{3\p}}}.
\]
\end{proof}

\subsection{Proof of Theorem \ref{thm:main-ncvx-lb}}

In this section, our ultimate goal is to prove Theorem \ref{thm:main-ncvx-lb}.
The analysis follows the framework first established in \citep{arjevani2023lower}
and later developed by \citep{pmlr-v202-cutkosky23a} but with some
necessary (though minor) variation to make it compatible with heavy-tailed
noise. In the following, we will slightly abuse the notation $\A$
to denote any possible randomized first-order algorithm instead of
an online learning algorithm.

\subsubsection{Basic Definitions}

To begin with, we introduce some basic definitions given in \citep{arjevani2023lower}.
\begin{defn}[stochastic first-order oracle]
Given a differentiable function $F:\R^{d}\to\R$, a tuple $(\bg,\calR,\P_{r})$
is called a stochastic first-order oracle of $F$ if $\P_{r}$ is
a probability distribution on the measurable space $\calR$ and $\bg:\R^{d}\times\calR\to\R^{d}$
satisfies $\E_{r\sim\P_{r}}\left[\bg(\bx,r)\right]=\nabla F(\bx),\forall\bx\in\R^{d}$.
\end{defn}
\begin{rem}
When the context is clear, we will omit $F$, $\calR$ and $\P_{r}$,
and simply call $\bg$ as a stochastic first-order oracle.
\end{rem}
\begin{defn}[randomized algorithm]
\label{def:randomized-algorithm}A randomized algorithm $\A$ consists
of a probability distribution $\P_{s}$ over a measurable space $\S$
and a sequence of measurable mappings $\A_{t},\forall t\in\N$ such
that every $\A_{t}$ takes a common random seed $s\in\S$ and the
first $t-1$ oracle responses to produce the $t$-th query. Concretely,
given a differentiable function $F$ equipped with a stochastic first-order
oracle $(\bg,\calR,\P_{r})$, the sequence $\bx_{t},\forall t\in\N$
produced by $\A$ to optimize $F$ is recursively defined as 
\[
\bx_{t}=\A_{t}(s,\bg(\bx_{t-1},r_{t-1}),\mydots,\bg(\bx_{1},r_{1})),\forall t\in\N,
\]
where $s\sim\P_{s}$ is drawn a single time at the beginning of the
algorithm and $r_{t}\sim\P_{r},\forall t\in\N$ is a sequence of i.i.d.
random variables. Moreover, $\A_{\rand}$ denotes the set containing
all randomized algorithms.
\end{defn}
Next, we require a useful concept named probability-$q$ zero-chain.
Before formally stating what it is, we need some notations. Given
$\bx\in\R^{d}$ and $\alpha\in\left[0,1\right]$, $\prog_{\alpha}(\bx)$
denotes the largest index whose entry is $\alpha$-far from $0$,
i.e.,
\begin{eqnarray*}
\prog_{\alpha}(\bx)\defeq\max\left\{ i\in\left[d\right]:\left|\bx\left[i\right]\right|>\alpha\right\}  & \text{where} & \max\varnothing\defeq0.
\end{eqnarray*}
In addition, for any $j\in\left\{ 0\right\} \cup\N$, let $\bx_{\leq j}\left[i\right]\defeq\bx\left[i\right]\1\left[i\leq j\right],\forall i\in\left[d\right]$
be the truncated version of $\bx$. Now we are ready to provide the
definition of the probability-$q$ zero-chain.
\begin{defn}[probability-$q$ zero-chain]
A stochastic first-order oracle $(\bg,\calR,\P_{r})$ is called
a probability-$q$ zero-chain if and only if
\[
\P_{r}\left[\forall\bx\in\R^{d}:\prog_{0}(\bg(\bx,r))\leq\prog_{1/4}(\bx)\text{ and }\bg(\bx,r)=\bg(\bx_{\leq\prog_{1/4}(\bx)},r)\right]\geq1-q,
\]
and
\[
\P_{r}\left[\forall\bx\in\R^{d}:\prog_{0}(\bg(\bx,r))\leq\prog_{1/4}(\bx)+1\text{ and }\bg(\bx,r)=\bg(\bx_{\leq\prog_{1/4}(\bx)+1},r)\right]=1.
\]
\end{defn}

\subsubsection{Useful Existing Results}

In this part, we list some useful existing results from \citep{arjevani2023lower}
and \citep{pmlr-v202-cutkosky23a}.

Given $d\geq T\in\N$ and a differentiable function $F_{T}:\R^{T}\to\R$
with a stochastic first-order oracle $\bg_{T}$ that is a probability-$q$
zero-chain, their rotated variants parametrized by a matrix $U\in\ortho(d,T)\defeq\left\{ U\in\R^{d\times T}:U^{\top}U=I_{T}\right\} $
(where $I_{T}$ is the $T$-dimensional identity matrix) are defined
as
\begin{eqnarray}
\bar{F}_{T,U}(\bx)\defeq F_{T}(U^{\top}\bx) & \text{and} & \bar{\bg}_{T,U}(\bx,r)\defeq U\bg_{T}(U^{\top}\bx,r).\label{eq:bar-F-g}
\end{eqnarray}
Clearly, $\bar{\bg}_{T,U}$ is a stochastic first-order oracle of
$\bar{F}_{T,U}$. In addition, we emphasize that the input variable
$\bx$ is from $\R^{d}$ instead of $\R^{T}$ now.

With $\bar{F}_{T,U}$ and $\bar{\bg}_{T,U}$, we state the following
lemma due to \citep{arjevani2023lower}.
\begin{lem}[Lemma 5 of \citep{arjevani2023lower}]
\label{lem:ncvx-lb-bounded} Given $q,\iota\in\left(0,1\right)$,
$R>0$, $T\in\N$, $d\in\N$ satisfying $d\geq T+32R^{2}\log\frac{2T^{2}}{q\iota}$
and a randomized algorithm $\A\in\A_{\rand}$, suppose the output
of $\A$ always has a norm bounded by $R$, let $\bx_{t},\forall t\in\N$
be the trajectory produced by applying $\A$ to optimize $\bar{F}_{T,U}$
interacting with the stochastic first-order oracle $\bar{\bg}_{T,U}$,
where $U$ is drawn from $\ortho(d,T)$ uniformly, then there is
\[
\Pr\left[\prog_{1/4}\left(U^{\top}\bx_{t}\right)<T,\forall t\leq\frac{T-\log(2/\iota)}{2q}\right]\geq1-\iota.
\]
Here $\Pr$ takes into account all randomness over $\A$, $\bg_{T}$,
and $U$.
\end{lem}
Note that Lemma \ref{lem:ncvx-lb-bounded} can be only applied to
a randomized algorithm with bounded outputs. To overcome this issue,
we need the following variants of $\bar{F}_{T,U}$ and $\bar{\bg}_{T,U}$
introduced by \citep{pmlr-v202-cutkosky23a}:
\begin{align}
\widehat{F}_{T,U}(\bx) & \defeq\bar{F}_{T,U}(\brho_{R,d}(\bx))+\eta\bx^{\top}\brho_{R,d}(\bx),\label{eq:hat-F}\\
\widehat{\bg}_{T,U}(\bx,r) & \defeq\J_{R,d}(\bx)^{\top}\bar{\bg}_{T,U}(\brho_{R,d}(\bx),r)+\eta\nabla(\bx^{\top}\brho_{R,d}(\bx)),\label{eq:hat-g}
\end{align}
where $\brho_{R,d}(\bx)\defeq\frac{\bx}{\sqrt{1+\left\Vert \bx\right\Vert ^{2}/R^{2}}}$
is a bijection from $\R^{d}$ to $\interior\B^{d}(R)$\footnote{As one can check, $\brho_{R,d}^{-1}(\bx):\interior\B^{d}(R)\to\R^{d}$
exists and equals $\frac{\bx}{\sqrt{1-\left\Vert \bx\right\Vert ^{2}/R^{2}}}$.}, $\J_{R,d}$ denotes the Jacobian of $\brho_{R,d}$, and $R,\eta>0$
are two constants being determined later.

Before moving on, we state some useful properties of $\brho_{R,d}$
here.
\begin{lem}[Lemma 13 of \citep{arjevani2023lower} and Proposition 29 of \citep{pmlr-v202-cutkosky23a}]
\label{lem:ncvx-lb-rho} Let $\left\Vert \cdot\right\Vert _{\op}$
denotes the operator norm, then for $\brho_{R,d}:\R^{d}\to\R^{d},\bx\mapsto\frac{\bx}{\sqrt{1+\left\Vert \bx\right\Vert ^{2}/R^{2}}}$,
there are
\begin{enumerate}
\item $\nabla(\bx^{\top}\brho_{R,d}(\bx))=\left(2-\left\Vert \brho_{R,d}(\bx)\right\Vert ^{2}/R^{2}\right)\brho_{R,d}(\bx)$,
\item $\left\Vert \brho_{R,d}(\bx)-\brho_{R,d}(\by)\right\Vert \leq\left\Vert \bx-\by\right\Vert $,
\item $\J_{R,d}(\bx)=\frac{I_{d}-\brho_{R,d}(\bx)\brho_{R,d}(\bx)^{\top}/R^{2}}{\sqrt{1+\left\Vert \bx\right\Vert ^{2}/R^{2}}}$,
\item $\left\Vert \J_{R,d}(\bx)\right\Vert _{\op}=\frac{1}{\sqrt{1+\left\Vert \bx\right\Vert ^{2}/R^{2}}}\leq1$,
\item $\left\Vert \J_{R,d}(\bx)-\J_{R,d}(\by)\right\Vert _{\op}\leq\frac{3}{R}\left\Vert \bx-\by\right\Vert $.
\end{enumerate}
\end{lem}
We then can combine (\ref{eq:bar-F-g}), (\ref{eq:hat-F}), (\ref{eq:hat-g})
and Lemma \ref{lem:ncvx-lb-rho} to have
\begin{align}
\widehat{F}_{T,U}(\bx) & =\bar{F}_{T,U}(\brho_{R,d}(\bx))+\eta\bx^{\top}\brho_{R,d}(\bx)=F_{T}(U^{\top}\brho_{R,d}(\bx))+\eta\bx^{\top}\brho_{R,d}(\bx),\label{eq:hat-F-simple}\\
\nabla\widehat{F}_{T,U}(\bx) & =\J_{R,d}(\bx)^{\top}\nabla\bar{F}_{T,U}(\brho_{R,d}(\bx))+\eta\left(2-\left\Vert \brho_{R,d}(\bx)\right\Vert ^{2}/R^{2}\right)\brho_{R,d}(\bx)\label{eq:hat-grad-F-simple-1}\\
 & =\J_{R,d}(\bx)^{\top}U\nabla F_{T}(U^{\top}\brho_{R,d}(\bx))+\eta\left(2-\left\Vert \brho_{R,d}(\bx)\right\Vert ^{2}/R^{2}\right)\brho_{R,d}(\bx),\label{eq:hat-grad-F-simple-2}\\
\widehat{\bg}_{T,U}(\bx,r) & =\J_{R,d}(\bx)^{\top}\bar{\bg}_{T,U}(\brho_{R,d}(\bx),r)+\eta\left(2-\left\Vert \brho_{R,d}(\bx)\right\Vert ^{2}/R^{2}\right)\brho_{R,d}(\bx)\label{eq:hat-g-simple-1}\\
 & =\J_{R,d}(\bx)^{\top}U\bg_{T}(U^{\top}\brho_{R,d}(\bx),r)+\eta\left(2-\left\Vert \brho_{R,d}(\bx)\right\Vert ^{2}/R^{2}\right)\brho_{R,d}(\bx).\label{eq:hat-g-simple-2}
\end{align}

Now we are ready to give the following Lemma \ref{lem:ncvx-lb-general},
which is almost identical to Lemma 25 in \citep{pmlr-v202-cutkosky23a}
by differing up to numerical constants. Since Lemma \ref{lem:ncvx-lb-general}
is particularly important, we therefore provide its full proof here.
\begin{lem}[Lemma 25 of \citep{pmlr-v202-cutkosky23a}]
\label{lem:ncvx-lb-general} For $T\in\N$, suppose that $F_{T}$
satisfies the following two requirements:
\begin{enumerate}
\item For all $\bx\in\R^{T}$, $\left\Vert \nabla F_{T}(\bx)\right\Vert \leq\gamma\sqrt{T}$
where $\gamma>0$ is a constant satisfying $\frac{1}{\sqrt{5}}-\frac{39}{140}\geq\frac{10}{63}+\frac{5\sqrt{5}}{126\gamma}\Leftrightarrow\gamma\geq\frac{50\sqrt{5}}{252\sqrt{5}-551}\approx8.95$.
\item For all $\bx\in\R^{T}$, if $\prog_{1}(\bx)<T$, then $\left\Vert \nabla F_{T}(\bx)\right\Vert \geq\left|\nabla F_{T}(\bx)\left[\prog_{1}(\bx)+1\right]\right|>1$.
\end{enumerate}
Given $q,\iota\in\left(0,1\right)$, $\eta=\frac{1}{\sqrt{5}}-\frac{39}{140}$,
$R=7\gamma\sqrt{T}$, $d\in\N$ satisfying $d\geq T+32R^{2}\log\frac{2T^{2}}{q\iota}$
and a randomized algorithm $\A\in\A_{\rand}$, let $\bx_{t},\forall t\in\N$
be the trajectory produced by applying $\A$ to optimize $\widehat{F}_{T,U}$
interacting with the stochastic first-order oracle $\widehat{\bg}_{T,U}$,
where $U$ is drawn from $\ortho(d,T)$ uniformly, then there is
\[
\Pr\left[\left\Vert \nabla\widehat{F}_{T,U}(\bx_{t})\right\Vert \geq\frac{1}{2},\forall t\leq\frac{T-\log(2/\iota)}{2q}\right]\geq1-\iota.
\]
Here $\Pr$ takes into account all randomness over $\A$, $\bg_{T}$
and $U$.
\end{lem}
\begin{proof}
By definition of randomized algorithms (see Definition \ref{def:randomized-algorithm}),
we have
\begin{equation}
\bx_{t}=\A_{t}\left(s,\widehat{\bg}_{T,U}(\bx_{t-1},r_{t-1}),\mydots,\widehat{\bg}_{T,U}(\bx_{1},r_{1})\right),\forall t\in\N.\label{eq:ncvx-lb-general-x}
\end{equation}
Now let us consider another sequence
\begin{equation}
\by_{t}\defeq\brho_{R,d}(\bx_{t})=\frac{\bx_{t}}{\sqrt{1+\left\Vert \bx_{t}\right\Vert /R^{2}}}\in\interior\B^{d}(R),\forall t\in\N.\label{eq:ncvx-lb-general-y}
\end{equation}
Note that
\begin{align}
\widehat{\bg}_{T,U}(\bx_{t},r_{t}) & \overset{(\ref{eq:hat-g-simple-1})}{=}\J_{R,d}(\bx_{t})^{\top}\bar{\bg}_{T,U}(\brho_{R,d}(\bx_{t}),r_{t})+\left(2-\left\Vert \brho_{R,d}(\bx_{t})\right\Vert ^{2}/R^{2}\right)\brho_{R,d}(\bx_{t})\nonumber \\
 & \overset{(\ref{eq:ncvx-lb-general-y})}{=}\J_{R,d}(\brho_{R,d}^{-1}(\by_{t}))^{\top}\bar{\bg}_{T,U}(\by_{t},r_{t})+\eta\left(2-\left\Vert \by_{t}\right\Vert ^{2}/R^{2}\right)\by_{t}\nonumber \\
 & =\G(\by_{t},\bar{\bg}_{T,U}(\by_{t},r_{t})),\label{eq:ncvx-lb-general-G}
\end{align}
where $\G(\bu,\bv):\interior\B^{d}(R)\times\R^{d}\to\R^{d}$ is a
measurable mapping defined as
\[
\G(\bu,\bv)\defeq\J_{R,d}(\brho_{R,d}^{-1}(\bu))^{\top}\bv+\eta\left(2-\left\Vert \bu\right\Vert ^{2}/R^{2}\right)\bu.
\]
Thus, we can write
\[
\by_{t}\overset{(\ref{eq:ncvx-lb-general-x}),(\ref{eq:ncvx-lb-general-y}),(\ref{eq:ncvx-lb-general-G})}{=}\brho_{R,d}\circ\A_{t}\left(s,\G(\by_{t-1},\bar{\bg}_{T,U}(\by_{t-1},r_{t-1})),\mydots,\G(\by_{1},\bar{\bg}_{T,U}(\by_{1},r_{1}))\right),\forall t\in\N.
\]
By a simple induction, there exists a sequence of measurable mappings
$\A_{t}^{\by},\forall t\in\N$ such that
\[
\by_{t}=\A_{t}^{\by}(s,\bar{\bg}_{T,U}(\by_{t-1},r_{t-1}),\mydots,\bar{\bg}_{T,U}(\by_{1},r_{1})),\forall t\in\N.
\]
The above reformulation implies $\by_{t},\forall t\in\N$ can be viewed
as a sequence produced by a randomized algorithm $\A^{\by}\in\A_{\rand}$
interacting with stochastic first-order oracle $\bar{\bg}_{T,U}$.
Note that $d\geq T+32R^{2}\log\frac{2T^{2}}{q\iota}$ and $\left\Vert \by_{t}\right\Vert \overset{(\ref{eq:ncvx-lb-general-y})}{\leq}R$,
we hence have by Lemma \ref{lem:ncvx-lb-bounded},
\begin{equation}
\Pr\left[\prog_{1/4}\left(U^{\top}\by_{t}\right)<T,\forall t\leq\frac{T-\log(2/\iota)}{2q}\right]\geq1-\iota.\label{eq:ncvx-lb-general-core}
\end{equation}

Now we recall
\begin{align}
\nabla\widehat{F}_{T,U}(\bx_{t}) & \overset{(\ref{eq:hat-grad-F-simple-1})}{=}\J_{R,d}(\bx_{t})^{\top}\nabla\bar{F}_{T,U}(\brho_{R,d}(\bx_{t}))+\eta\left(2-\left\Vert \brho_{R,d}(\bx_{t})\right\Vert ^{2}/R^{2}\right)\brho_{R,d}(\bx_{t})\nonumber \\
 & \overset{(\ref{eq:ncvx-lb-general-y})}{=}\J_{R,d}(\bx_{t})^{\top}\nabla\bar{F}_{T,U}(\by_{t})+\eta\left(2-\left\Vert \by_{t}\right\Vert ^{2}/R^{2}\right)\by_{t}.\label{eq:ncvx-lb-general-1}
\end{align}
Moreover, by the first requirement on $F_{T}$, there is
\begin{equation}
\left\Vert \nabla\bar{F}_{T,U}(\by_{t})\right\Vert \overset{(\ref{eq:bar-F-g})}{=}\left\Vert U\nabla F_{T}(U^{\top}\by_{t})\right\Vert \overset{U^{\top}U=I_{T}}{=}\left\Vert \nabla F_{T}(U^{\top}\by_{t})\right\Vert \leq\gamma\sqrt{T}.\label{eq:ncvx-lb-general-2}
\end{equation}
We fix $t\leq\frac{T-\log(2/\iota)}{2q}$ and consider the following
two cases:
\begin{casenv}
\item $\left\Vert \bx_{t}\right\Vert \ge\frac{R}{2}$. In this case, we
first have
\[
\left\Vert \nabla\widehat{F}_{T,U}(\bx_{t})\right\Vert \overset{(\ref{eq:ncvx-lb-general-1})}{\geq}\eta\left(2-\left\Vert \by_{t}\right\Vert ^{2}/R^{2}\right)\left\Vert \by_{t}\right\Vert -\left\Vert \J_{R,d}(\bx_{t})^{\top}\nabla\bar{F}_{T,U}(\by_{t})\right\Vert .
\]
Note that $\frac{\left\Vert \by_{t}\right\Vert }{R}\overset{(\ref{eq:ncvx-lb-general-y})}{=}\frac{\left\Vert \bx_{t}\right\Vert /R}{\sqrt{1+\left\Vert \bx_{t}\right\Vert ^{2}/R^{2}}}\in\left[\frac{1}{\sqrt{5}},1\right]$
when $\left\Vert \bx_{t}\right\Vert \geq\frac{R}{2}$, implying
\[
\eta\left(2-\left\Vert \by_{t}\right\Vert ^{2}/R^{2}\right)\left\Vert \by_{t}\right\Vert \geq\eta R\min_{\frac{1}{\sqrt{5}}\leq a\leq1}(2-a^{2})a=\frac{9\eta R}{5\sqrt{5}}.
\]
In addition, there is
\begin{align*}
\left\Vert \J_{R,d}(\bx_{t})^{\top}\nabla\bar{F}_{T,U}(\by_{t})\right\Vert  & \leq\left\Vert \J_{R,d}(\bx_{t})\right\Vert _{\op}\left\Vert \nabla\bar{F}_{T,U}(\by_{t})\right\Vert \overset{\text{Lemma }\ref{lem:ncvx-lb-rho}}{=}\frac{\left\Vert \nabla\bar{F}_{T,U}(\by_{t})\right\Vert }{\sqrt{1+\left\Vert \bx_{t}\right\Vert ^{2}/R^{2}}}\\
 & \overset{\left\Vert \bx_{t}\right\Vert \ge\frac{R}{2}}{\leq}\frac{2\left\Vert \nabla\bar{F}_{T,U}(\by_{t})\right\Vert }{\sqrt{5}}\overset{(\ref{eq:ncvx-lb-general-2})}{\leq}\frac{2\gamma\sqrt{T}}{\sqrt{5}}.
\end{align*}
As such, by our choices of $\eta$ and $R$,
\[
\left\Vert \nabla\widehat{F}_{T,U}(\bx_{t})\right\Vert \geq\frac{9\eta R}{5\sqrt{5}}-\frac{2\gamma\sqrt{T}}{\sqrt{5}}\geq\frac{1}{2}.
\]
\item $\left\Vert \bx_{t}\right\Vert <\frac{R}{2}$ . In this case, we introduce
\[
j_{t}\defeq\prog_{1}\left(U^{\top}\by_{t}\right)+1\leq\prog_{1/4}\left(U^{\top}\by_{t}\right)+1\overset{(\ref{eq:ncvx-lb-general-core})}{\leq}T.
\]
Let $\bu_{j_{t}}\in\R^{d}$ denotes the $j_{t}$-th column of $U$.
By the definition of $j_{t}$, there is
\begin{equation}
\left|\left\langle \bu_{j_{t}},\by_{t}\right\rangle \right|=\left|\left(U^{\top}\by_{t}\right)\left[j_{t}\right]\right|<1.\label{eq:lb-general-3}
\end{equation}
In addition, we have
\begin{equation}
\left|\left\langle \bu_{j_{t}},\nabla\bar{F}_{T,U}(\by_{t})\right\rangle \right|=\left|\left(U^{\top}\nabla\bar{F}_{T,U}(\by_{t})\right)\left[j_{t}\right]\right|\overset{(\ref{eq:bar-F-g}),U^{\top}U=I_{T}}{=}\left|\nabla F_{T}(U^{\top}\by_{t})\left[j_{t}\right]\right|>1,\label{eq:ncvx-lb-general-4}
\end{equation}
where the last step is by the second requirement on $F_{T}$. Now
we compute
\begin{align*}
 & \left\Vert \nabla\widehat{F}_{T,U}(\bx_{t})\right\Vert \overset{\left\Vert \bu_{j_{t}}\right\Vert =1}{\geq}\left|\left\langle \bu_{j_{t}},\nabla\widehat{F}_{T,U}(\bx_{t})\right\rangle \right|\\
\overset{(\ref{eq:ncvx-lb-general-1})}{=} & \left|\left\langle \bu_{j_{t}},\J_{R,d}(\bx_{t})^{\top}\nabla\bar{F}_{T,U}(\by_{t})\right\rangle +\eta\left(2-\left\Vert \by_{t}\right\Vert ^{2}/R^{2}\right)\left\langle \bu_{j_{t}},\by_{t}\right\rangle \right|\\
\overset{\text{Lemma }\ref{lem:ncvx-lb-rho}}{=} & \left|\frac{\left\langle \bu_{j_{t}},\nabla\bar{F}_{T,U}(\by_{t})\right\rangle }{\sqrt{1+\left\Vert \bx_{t}\right\Vert ^{2}/R^{2}}}-\frac{\left\langle \bu_{j_{t}},\by_{t}\right\rangle \left\langle \by_{t},\nabla\bar{F}_{T,U}(\by_{t})\right\rangle }{R^{2}\sqrt{1+\left\Vert \bx_{t}\right\Vert ^{2}/R^{2}}}+\eta\left(2-\left\Vert \by_{t}\right\Vert ^{2}/R^{2}\right)\left\langle \bu_{j_{t}},\by_{t}\right\rangle \right|\\
\geq & \frac{\left|\left\langle \bu_{j_{t}},\nabla\bar{F}_{T,U}(\by_{t})\right\rangle \right|}{\sqrt{1+\left\Vert \bx_{t}\right\Vert ^{2}/R^{2}}}-\left|\left\langle \bu_{j_{t}},\by_{t}\right\rangle \right|\left[\frac{\left|\left\langle \by_{t},\nabla\bar{F}_{T,U}(\by_{t})\right\rangle \right|}{R^{2}\sqrt{1+\left\Vert \bx_{t}\right\Vert ^{2}/R^{2}}}+\eta\left(2-\left\Vert \by_{t}\right\Vert ^{2}/R^{2}\right)\right]\\
\overset{(a)}{\geq} & \frac{1}{\sqrt{1+\left\Vert \bx_{t}\right\Vert ^{2}/R^{2}}}-\frac{\left\Vert \by_{t}\right\Vert \left\Vert \nabla\bar{F}_{T,U}(\by_{t})\right\Vert }{R^{2}\sqrt{1+\left\Vert \bx_{t}\right\Vert ^{2}/R^{2}}}-2\eta\overset{(b)}{\geq}\frac{2}{\sqrt{5}}-\frac{2\gamma\sqrt{T}}{5R}-2\eta\overset{(c)}{=}\frac{1}{2},
\end{align*}
where $(a)$ is due to (\ref{eq:lb-general-3}), (\ref{eq:ncvx-lb-general-4}),
Cauchy-Schwarz inequality and $\eta\geq0$, $(b)$ is by $\frac{1}{\sqrt{1+\left\Vert \bx_{t}\right\Vert ^{2}/R^{2}}}>\frac{2}{\sqrt{5}}$
and $\frac{\left\Vert \by_{t}\right\Vert }{\sqrt{1+\left\Vert \bx_{t}\right\Vert ^{2}/R^{2}}}\overset{(\ref{eq:ncvx-lb-general-y})}{=}\frac{\left\Vert \bx_{t}\right\Vert }{1+\left\Vert \bx_{t}\right\Vert ^{2}/R^{2}}<\frac{2R}{5}$
when $\left\Vert \bx_{t}\right\Vert <\frac{R}{2}$ and (\ref{eq:ncvx-lb-general-2}),
and $(c)$ holds under our choices of $\eta$ and $R$.
\end{casenv}
We combine two cases to conclude
\[
\Pr\left[\left\Vert \nabla\widehat{F}_{T,U}(\bx_{t})\right\Vert \geq\frac{1}{2},\forall t\leq\frac{T-\log(2/\iota)}{2q}\right]\geq1-\iota.
\]
\end{proof}

Lastly, we recall the following hard instance and its stochastic first-order
oracle studied in \citep{arjevani2023lower}.
\begin{lem}[Lemma 2 of \citep{arjevani2023lower}]
\label{lem:ncvx-lb-F} For any $T\in\N$, there exists a differentiable
function $F_{T}:\R^{T}\to\R$ satisfying the following properties:
\begin{enumerate}
\item $F_{T}(\bzero)=0$ and $\inf_{\bx\in\R^{T}}F_{T}(\bx)\geq-fT$, where
$f=12$.
\item $F_{T}$ is $\ell$-smooth, where $\ell=152$.
\item For all $\bx\in\R^{T}$, $\left\Vert \nabla F_{T}(\bx)\right\Vert _{\infty}\leq\gamma$,
where $\gamma=23$.
\item For all $\bx\in\R^{T}$, $\prog_{0}(\nabla F_{T}(\bx))\leq\prog_{1/2}(\bx)+1$.
\item For all $\bx\in\R^{T}$ and $i\defeq\prog_{1/2}(\bx)$, $\nabla F_{T}(\bx)=\nabla F_{T}(\bx_{\leq i+1})$
and $\left[\nabla F_{T}(\bx)\right]_{\leq i}=\left[\nabla F_{T}(\bx_{\leq i})\right]_{\leq i}$.
\item For all $\bx\in\R^{T}$, if $\prog_{1}(\bx)<T$, then $\left\Vert \nabla F_{T}(\bx)\right\Vert \geq\left|\nabla F_{T}(\bx)\left[\prog_{1}(\bx)+1\right]\right|>1$.
\end{enumerate}
\end{lem}
\begin{lem}[Lemma 3 of \citep{arjevani2023lower}]
\label{lem:ncvx-lb-g} For any $T\in\N$ and $F_{T}$ in Lemma \ref{lem:ncvx-lb-F},
the following $\bg_{T}:\R^{T}\times\calR\to\R^{T}$, where $\calR\defeq\left\{ 0,1\right\} $,
is a stochastic first-order oracle of $F_{T}$:
\[
\bg_{T}(\bx,r)\left[i\right]\defeq\begin{cases}
\nabla F_{T}(\bx)\left[i\right] & i\neq\prog_{1/4}(\bx)+1\\
\frac{r}{q}\nabla F_{T}(\bx)\left[i\right] & i=\prog_{1/4}(\bx)+1
\end{cases},\forall i\in\left[T\right],
\]
where $r=\mathsf{Bernoulli}(q)$ for some $q\in\left(0,1\right)$.
Moreover, $\bg_{T}$ is a probability-$q$ zero-chain.
\end{lem}

\subsubsection{Analysis under Heavy-Tailed Noise}

From now on, we need to diverge from \citep{arjevani2023lower} and
\citep{pmlr-v202-cutkosky23a}, since both of which are under the
finite variance case (i.e., $\p=2$) instead of heavy-tailed noise
(i.e., $\p\in\left(1,2\right]$).
\begin{lem}[variation of Lemma 7 in \citep{arjevani2023lower} and Lemma 26 in
\citep{pmlr-v202-cutkosky23a}]
\label{lem:lb-instance-property}The instances $\widehat{F}_{T,U}$
and $\widehat{\bg}_{T,U}$ (see (\ref{eq:hat-F}) and (\ref{eq:hat-g}))
constructed based on $F_{T}$ in Lemma \ref{lem:ncvx-lb-F} and $\bg_{T}$
in Lemma \ref{lem:ncvx-lb-g} under $\eta=\frac{1}{\sqrt{5}}-\frac{39}{140}$
and $R=7\gamma\sqrt{T}$ for $\gamma=23$ satisfy the following properties:
\begin{enumerate}
\item $\widehat{F}_{T,U}(\bzero)-\inf_{\bx\in\R^{d}}\widehat{F}_{T,U}(\bx)\leq\widehat{f}T$,
where $\widehat{f}=12$.
\item $\widehat{F}_{T,U}$ is $\widehat{\ell}$-smooth, where $\widehat{\ell}=154$.
\item For all $\bx\in\R^{d}$, $\left\Vert \nabla\widehat{F}_{T,U}(\bx)\right\Vert \leq\widehat{\gamma}\sqrt{T}$,
where $\widehat{\gamma}=53$.
\item For all $\bx\in\R^{d}$ and $\p\in\left(1,2\right]$, $\E_{r}\left[\left\Vert \nabla\widehat{F}_{T,U}(\bx)-\widehat{\bg}_{T,U}(\bx,r)\right\Vert ^{\p}\right]\leq\frac{\left(2\gamma\right)^{\p}\left(1-q^{\p-1}\right)}{(\p-1)q^{\p-1}}$.
\end{enumerate}
\end{lem}
\begin{proof}
First, we know
\begin{align*}
\widehat{F}_{T,U}(\bzero)-\inf_{\bx\in\R^{d}}\widehat{F}_{T,U}(\bx) & \overset{(\ref{eq:hat-F-simple})}{=}F_{T}(\bzero)-\inf_{\bx\in\R^{d}}\left(F_{T}(U^{\top}\brho_{R,d}(\bx))+\eta\bx^{\top}\brho_{R,d}(\bx)\right)\\
 & \overset{\bx^{\top}\brho_{R,d}(\bx)\geq0}{\leq}F_{T}(\bzero)-\inf_{\bx\in\R^{T}}F_{T}(\bx)\overset{\text{Lemma }\ref{lem:ncvx-lb-F}}{\leq}fT=\widehat{f}T.
\end{align*}

Next, for any $\bx\in\R^{d}$,
\[
\nabla\widehat{F}_{T,U}(\bx)\overset{(\ref{eq:hat-grad-F-simple-2})}{=}\J_{R,d}(\bx)^{\top}U\nabla F_{T}(U^{\top}\brho_{R,d}(\bx))+\eta\left(2-\left\Vert \brho_{R,d}(\bx)\right\Vert ^{2}/R^{2}\right)\brho_{R,d}(\bx).
\]
By Lemma 14 in \citep{arjevani2023lower} and Lemma \ref{lem:ncvx-lb-F},
$\J_{R,d}(\bx)^{\top}U\nabla F_{T}(U^{\top}\brho_{R,d}(\bx))$ is
$\ell+\frac{3\gamma\sqrt{T}}{R}\overset{R=7\gamma\sqrt{T}}{=}\ell+\frac{3}{7}$-Lipschitz
for $\ell=152$. Moreover, we have
\begin{align*}
 & \left\Vert \left(2-\left\Vert \brho_{R,d}(\bx)\right\Vert ^{2}/R^{2}\right)\brho_{R,d}(\bx)-\left(2-\left\Vert \brho_{R,d}(\by)\right\Vert ^{2}/R^{2}\right)\brho_{R,d}(\by)\right\Vert \\
\leq & 2\left\Vert \brho_{R,d}(\bx)-\brho_{R,d}(\by)\right\Vert +\frac{\left\Vert \brho_{R,d}(\by)\right\Vert ^{2}}{R^{2}}\left\Vert \brho_{R,d}(\by)-\brho_{R,d}(\bx)\right\Vert \\
 & +\frac{\left|\left\Vert \brho_{R,d}(\by)\right\Vert ^{2}-\left\Vert \brho_{R,d}(\bx)\right\Vert ^{2}\right|}{R^{2}}\left\Vert \brho_{R,d}(\bx)\right\Vert \\
\overset{(a)}{\leq} & 2\left\Vert \bx-\by\right\Vert +\left\Vert \by-\bx\right\Vert +2\left\Vert \by-\bx\right\Vert =5\left\Vert \bx-\by\right\Vert ,
\end{align*}
where $(a)$ is by $\left\Vert \brho_{R,d}(\bx)-\brho_{R,d}(\by)\right\Vert \overset{\text{Lemma }\ref{lem:ncvx-lb-rho}}{\leq}\left\Vert \bx-\by\right\Vert $
and $\left\Vert \brho_{R,d}(\cdot)\right\Vert \leq R$. So $\widehat{F}_{T,U}(\bx)$
is $\ell+\frac{3}{7}+5\eta\overset{\eta=\frac{1}{\sqrt{5}}-\frac{39}{140}}{=}\frac{4229}{28}+\sqrt{5}\leq154=\widehat{\ell}$-smooth.

Moreover, we observe that
\[
\left\Vert \nabla\widehat{F}_{T,U}(\bx)\right\Vert \overset{(\ref{eq:hat-grad-F-simple-2})}{\leq}\left\Vert \J_{R,d}(\bx)\right\Vert _{\op}\left\Vert U\nabla F_{T}(U^{\top}\brho_{R,d}(\bx))\right\Vert +\eta\left(2-\left\Vert \brho_{R,d}(\bx)\right\Vert ^{2}/R^{2}\right)\left\Vert \brho_{R,d}(\bx)\right\Vert .
\]
Note that for $\gamma=23$,
\begin{align*}
\left\Vert \J_{R,d}(\bx)\right\Vert _{\op}\left\Vert U\nabla F_{T}(U^{\top}\brho_{R,d}(\bx))\right\Vert  & \overset{\text{Lemma }\ref{lem:ncvx-lb-rho}}{\leq}\left\Vert U\nabla F_{T}(U^{\top}\brho_{R,d}(\bx))\right\Vert \\
 & \overset{U^{\top}U=I_{T}}{=}\left\Vert \nabla F_{T}(U^{\top}\brho_{R,d}(\bx))\right\Vert \overset{\text{Lemma }\ref{lem:ncvx-lb-F}}{\leq}\gamma\sqrt{T},
\end{align*}
and
\begin{align*}
\eta\left(2-\left\Vert \brho_{R,d}(\bx)\right\Vert ^{2}/R^{2}\right)\left\Vert \brho_{R,d}(\bx)\right\Vert  & \overset{\left\Vert \brho_{R,d}(\cdot)\right\Vert \leq R}{\leq}\eta R\max_{0\leq a\leq1}(2-a^{2})a=\eta R\cdot\frac{4}{3}\sqrt{\frac{2}{3}}\\
 & =\left(\frac{1}{\sqrt{5}}-\frac{39}{140}\right)\cdot7\cdot\frac{4}{3}\sqrt{\frac{2}{3}}\cdot\gamma\sqrt{T}\leq1.3\gamma\sqrt{T}.
\end{align*}
We hence have $\left\Vert \nabla\widehat{F}_{T,U}(\bx)\right\Vert \leq2.3\gamma\sqrt{T}\leq53\sqrt{T}=\widehat{\gamma}\sqrt{T}$.

Lastly, still for $\gamma=23$, we compute
\begin{align*}
\left\Vert \nabla\widehat{F}_{T,U}(\bx)-\widehat{\bg}_{T,U}(\bx,r)\right\Vert  & \overset{(\ref{eq:hat-grad-F-simple-2}),(\ref{eq:hat-g-simple-2})}{=}\left\Vert \J_{R,d}(\bx)^{\top}U\nabla F_{T}(U^{\top}\brho_{R,d}(\bx))-\J_{R,d}(\bx)^{\top}U\bg_{T}(U^{\top}\brho_{R,d}(\bx),r)\right\Vert \\
 & \leq\left\Vert \J_{R,d}(\bx)\right\Vert _{\op}\left\Vert U\left(\nabla F_{T}(U^{\top}\brho_{R,d}(\bx))-\bg_{T}(U^{\top}\brho_{R,d}(\bx),r)\right)\right\Vert \\
 & \overset{\text{Lemma }\ref{lem:ncvx-lb-rho}}{\leq}\left\Vert U\left(\nabla F_{T}(U^{\top}\brho_{R,d}(\bx))-\bg_{T}(U^{\top}\brho_{R,d}(\bx),r)\right)\right\Vert \\
 & \overset{U^{\top}U=I_{T}}{=}\left\Vert \nabla F_{T}(U^{\top}\brho_{R,d}(\bx))-\bg_{T}(U^{\top}\brho_{R,d}(\bx),r)\right\Vert \\
 & \overset{\text{Lemma }\ref{lem:ncvx-lb-g}}{=}\left|1-\frac{r}{q}\right|\left|\nabla F_{T}(\bx)\left[\prog_{1/4}(\bx)+1\right]\right|\overset{\text{Lemma }\ref{lem:ncvx-lb-F}}{\leq}\gamma\left|1-\frac{r}{q}\right|,
\end{align*}
which implies
\begin{align*}
\E_{r}\left[\left\Vert \nabla\widehat{F}_{T,U}(\bx)-\widehat{\bg}_{T,U}(\bx,r)\right\Vert ^{\p}\right] & \leq\gamma^{\p}\left(1-q+\frac{(1-q)^{\p}}{q^{\p-1}}\right)=\gamma^{\p}(1-q)\left(\frac{q^{\p-1}+(1-q)^{\p-1}}{q^{\p-1}}\right)\\
 & \overset{\text{Fact }\ref{fact:p-q}}{\leq}\frac{\gamma^{\p}2^{2-\p}\left(1-q^{\p-1}\right)}{(\p-1)q^{\p-1}}\overset{\p>1}{\leq}\frac{\left(2\gamma\right)^{\p}\left(1-q^{\p-1}\right)}{(\p-1)q^{\p-1}}.
\end{align*}
\end{proof}

Next, inspired by \citep{pmlr-v202-cutkosky23a}, we first prove a
lower bound for heavy-tailed smooth nonconvex optimization, as presented
in the following Theorem \ref{thm:ncvx-lb-smooth}.
\begin{thm}
\label{thm:ncvx-lb-smooth}For any $\Delta>0$, $H>0$, $\p\in\left(1,2\right]$,
$\sigma\geq0$, $0<\epsilon\leq\sqrt{\frac{\Delta H}{96\cdot12\cdot154}}$,
let $d$ be in the order of $\frac{\Delta H}{\epsilon^{2}}\log\left(\frac{\Delta H}{\epsilon^{2}}\left(1+\left(\frac{\sigma}{\epsilon}\right)^{\frac{\p}{2(\p-1)}}\right)\right)$
(see proof for the precise definition), there exists a distribution
over functions $F:\R^{d}\to\R$ and stochastic first-order oracles
$\bg$ such that with probability $1$, $F(\bzero)-F_{\star}\leq\Delta$,
$F$ is $H$-smooth and $\frac{3}{2}\sqrt{H\Delta}$-Lipschitz, and
$\bg$ has a finite $\p$-th centered moment $\sigma^{\p}$. Moreover,
for any randomized $\A\in\A_{\rand}$ employed to optimize a randomly
selected $F$ interacting with $\bg$, to output a point $\bx$ such
that $\E\left[\left\Vert \nabla F(\bx)\right\Vert \right]\leq\epsilon$,
the number of queries of $\bg$ by $\A$ satisfies $\gtrsim\Delta H\epsilon^{-2}+\Delta H\sigma^{\frac{\p}{\p-1}}\epsilon^{-\frac{3\p-2}{\p-1}}$.
\end{thm}
\begin{rem}
The reader familiar with the literature may find that a similar lower
bound (in fact, exactly the same order) has been established by \citep{liu2025nonconvex,NEURIPS2020_b05b57f6}
before, and hence may wonder about the difference. Here, we note that
the lower bounds in \citep{liu2025nonconvex,NEURIPS2020_b05b57f6}
are shown for a special algorithmic class known as zero-respecting
algorithms\footnote{A first-order algorithm is called zero-respecting if it satisfies
$\bx_{t}\in\cup_{s<t}\mathrm{support}(\bg_{s}),\forall t\in\N$. For
more details, see Definition 1 of \citep{arjevani2023lower}.}. However, our Theorem \ref{thm:ncvx-lb-smooth} is proved for a broader
family, i.e., randomized algorithms. This fact is important because
Algorithm \ref{alg:O2NC} is a randomized algorithm but not a zero-respecting
algorithm.
\end{rem}
\begin{proof}
For $T\in\N$ and $q\in\left(0,1\right)$ being determined later,
let $\iota=1/2$, $\eta=\frac{1}{\sqrt{5}}-\frac{39}{140}$, $R=7\gamma\sqrt{T}$
where $\gamma=23$ and $d=\left\lceil T+32R^{2}\log\frac{2T^{2}}{q\iota}\right\rceil =\left\lceil T+32R^{2}\log\frac{4T^{2}}{q}\right\rceil =\Theta(T\log\frac{T^{2}}{q})$,
we construct $\widehat{F}_{T,U}:\R^{d}\to\R$ and $\widehat{\bg}_{T,U}(\bx,r):\R^{d}\times\calR\to\R^{d}$
based on $F_{T}$ in Lemma \ref{lem:ncvx-lb-F} and $\bg_{T}$ in
Lemma \ref{lem:ncvx-lb-g}.

By Lemma \ref{lem:ncvx-lb-general} (note that $F_{T}$ satisfies
the requirements of Lemma \ref{lem:ncvx-lb-general} due to Lemma
\ref{lem:ncvx-lb-F}), for any $\A\in\A_{\rand}$ employed to optimize
$\widehat{F}_{T,U}$ interacting with the stochastic first-order oracle
$\widehat{\bg}_{T,U}$, where $U$ is drawn from $\ortho(d,T)$ uniformly,
we have
\begin{equation}
\Pr\left[\left\Vert \nabla\widehat{F}_{T,U}(\bx_{t})\right\Vert \geq\frac{1}{2},\forall t\leq\frac{T-\log4}{2q}\right]\geq\frac{1}{2}.\label{eq:lb-full-1}
\end{equation}
By Lemma \ref{lem:lb-instance-property}, $\widehat{F}_{T,U}(\bzero)-\inf_{\bx\in\R^{d}}\widehat{F}_{T,U}(\bx)\leq\widehat{f}T$
for $\widehat{f}=12$, $\widehat{F}_{T,U}$ is $\widehat{\ell}$-smooth
for $\widehat{\ell}=154$ and $\widehat{\gamma}\sqrt{T}$-Lipschitz
for $\widehat{\gamma}=53$, $\E_{r}\left[\left\Vert \nabla\widehat{F}_{T,U}(\bx)-\widehat{\bg}_{T,U}(\bx,r)\right\Vert ^{\p}\right]\leq\frac{\left(2\gamma\right)^{\p}\left(1-q^{\p-1}\right)}{(\p-1)q^{\p-1}},\forall\bx\in\R^{d}$.

Now we set
\begin{eqnarray}
\lambda\defeq\frac{4\widehat{\ell}\epsilon}{H}, & T\defeq\left\lfloor \frac{\widehat{\ell}\Delta}{\widehat{f}H\lambda^{2}}\right\rfloor , & q\defeq\frac{1}{\left[1+(\p-1)\left(\frac{\widehat{\ell}\sigma}{2\gamma H\lambda}\right)^{\p}\right]^{\frac{1}{\p-1}}}.\label{eq:lb-full-l-T-q}
\end{eqnarray}
For any $U\in\ortho(d,T)$, we introduce
\begin{eqnarray*}
F_{U}(\bx)\defeq\frac{H\lambda^{2}}{\widehat{\ell}}\widehat{F}_{T,U}(\bx/\lambda) & \text{and} & \bg_{U}(\bx)=\frac{H\lambda}{\widehat{\ell}}\widehat{\bg}_{T,U}(\bx/\lambda,r).
\end{eqnarray*}
We observe that
\[
F_{U}(\bzero)-\inf_{\bx\in\R^{d}}F_{U}(\bx)=\frac{H\lambda^{2}}{\widehat{\ell}}\left(\widehat{F}_{T,U}(\bzero)-\inf_{\bx\in\R^{d}}\widehat{F}_{T,U}(\bx/\lambda)\right)\leq\frac{H\lambda^{2}}{\widehat{\ell}}\widehat{f}T\overset{(\ref{eq:lb-full-l-T-q})}{\leq}\Delta.
\]
Moreover, $\nabla F_{U}(\bx)=\frac{H\lambda}{\widehat{\ell}}\nabla\widehat{F}_{T,U}(\bx/\lambda)$,
which implies $F_{U}$ is $H$-smooth and $\frac{H\lambda\widehat{\gamma}\sqrt{T}}{\widehat{\ell}}\overset{(\ref{eq:lb-full-l-T-q})}{\leq}\frac{3\sqrt{H\Delta}}{2}$-Lipschitz.
In addition, there is
\begin{align*}
\E_{r}\left[\left\Vert \nabla F_{U}(\bx)-\bg_{U}(\bx)\right\Vert ^{\p}\right] & =\left(\frac{H\lambda}{\widehat{\ell}}\right)^{\p}\E_{r}\left[\left\Vert \nabla\widehat{F}_{T,U}(\bx/\lambda)-\widehat{\bg}_{T,U}(\bx/\lambda,r)\right\Vert ^{\p}\right]\\
 & \leq\left(\frac{2\gamma H\lambda}{\widehat{\ell}}\right)^{\p}\frac{1-q^{\p-1}}{(\p-1)q^{\p-1}}\overset{(\ref{eq:lb-full-l-T-q})}{=}\sigma^{\p}.
\end{align*}

For any $\A\in\A_{\rand}$ used to optimize $F_{U}$ with $\bg_{U}$
when $U$ is drawn from $\ortho(d,T)$ uniformly, we can view as $\bx_{t}/\lambda$
as the output of another $\A^{\lambda}\in\A_{\rand}$ interacting
with $\widehat{F}_{T,U}$ and $\widehat{\bg}_{T,U}$ (a similar argument
is used in the proof of Lemma \ref{lem:ncvx-lb-general}). As such,
by (\ref{eq:lb-full-1}),
\[
\Pr\left[\left\Vert \nabla\widehat{F}_{T,U}(\bx_{t}/\lambda)\right\Vert \geq\frac{1}{2},\forall t\leq\frac{T-\log4}{2q}\right]\geq\frac{1}{2},
\]
which implies with probability at least $1/2$
\begin{align*}
\left\Vert \nabla F_{U}(\bx_{t})\right\Vert  & =\frac{H\lambda}{\widehat{\ell}}\left\Vert \nabla\widehat{F}_{T,U}(\bx_{t}/\lambda)\right\Vert \geq\frac{H\lambda}{2\widehat{\ell}},\forall t\leq\frac{T-\log4}{2q}.
\end{align*}
Therefore, we have
\[
\E\left[\left\Vert \nabla F_{U}(\bx_{t})\right\Vert \right]\geq\frac{H\lambda}{2\widehat{\ell}}\Pr\left[\left\Vert \nabla F_{U}(\bx_{t})\right\Vert \geq\frac{H\lambda}{2\widehat{\ell}}\right]\geq\frac{H\lambda}{4\widehat{\ell}}\overset{(\ref{eq:lb-full-l-T-q})}{=}\epsilon,\forall t\leq\frac{T-\log4}{2q}.
\]
Lastly, we compute
\begin{align*}
\frac{T-\log4}{2q} & \overset{(\ref{eq:lb-full-l-T-q})}{=}\frac{1}{2}\left(\left\lfloor \frac{\Delta H}{16\widehat{f}\widehat{\ell}\epsilon^{2}}\right\rfloor -\log4\right)\left[1+(\p-1)\left(\frac{\widehat{\ell}\sigma}{2\gamma H\lambda}\right)^{\p}\right]^{\frac{1}{\p-1}}\\
 & \overset{(a)}{\geq}\frac{1}{2}\left(\frac{\Delta H}{16\widehat{f}\widehat{\ell}\epsilon^{2}}-3\right)\left[1+(\p-1)\left(\frac{\widehat{\ell}\sigma}{2\gamma H\lambda}\right)^{\p}\right]^{\frac{1}{\p-1}}\\
 & \overset{(b)}{\geq}\frac{\Delta H}{64\widehat{f}\widehat{\ell}\epsilon^{2}}\left(1+(\p-1)^{\frac{1}{\p-1}}\left(\frac{\widehat{\ell}\sigma}{2\gamma H\lambda}\right)^{\frac{\p}{\p-1}}\right)\\
 & \overset{(\ref{eq:lb-full-l-T-q})}{\gtrsim}\Delta H\epsilon^{-2}+\Delta H\sigma^{\frac{\p}{\p-1}}\epsilon^{-\frac{3\p-2}{\p-1}},
\end{align*}
where the $(a)$ is by $\left\lfloor \cdot\right\rfloor \ge\cdot-1$
and $\log4\leq2$, and $(b)$ holds due to the condition $\epsilon^{2}\leq\frac{\Delta H}{96\cdot12\cdot154}=\frac{\Delta H}{96\widehat{f}\widehat{\ell}}$,
implying $\frac{\Delta H}{16\widehat{f}\widehat{\ell}\epsilon^{2}}-3\geq\frac{\Delta H}{32\widehat{f}\widehat{\ell}\epsilon^{2}}$,
and $(1+x)^{\frac{1}{\p-1}}\geq1+x^{\frac{1}{\p-1}}$ for $x\geq0$
when $\p\in\left(1,2\right]$. So to achieve $\E\left[\left\Vert \nabla F_{U}(\bx_{t})\right\Vert \right]<\epsilon$,
it requires at least $\Delta H\epsilon^{-2}+\Delta H\sigma^{\frac{\p}{\p-1}}\epsilon^{-\frac{3\p-2}{\p-1}}$
many iterations. In particular, we note that $d=\Theta(T\log\frac{T^{2}}{q})=\Theta\left(\frac{\Delta H}{\epsilon^{2}}\log\left(\frac{\Delta H}{\epsilon^{2}}\left(1+\left(\frac{\sigma}{\epsilon}\right)^{\frac{\p}{2(\p-1)}}\right)\right)\right)$.
\end{proof}

Equipped with Theorem \ref{thm:ncvx-lb-smooth}, we can show the lower
bound on the distributional complexity \citep{nemirovskij1983problem}
for nonsmooth nonconvex optimization with heavy tails in the following
Theorem \ref{thm:ncvx-lb-nonsmooth}.
\begin{thm}
\label{thm:ncvx-lb-nonsmooth}For any given $\Delta>0$, $G>0$, $\p\in\left(1,2\right]$,
$\sigma\geq0$, $\delta>0$ and $0<\epsilon\leq\frac{\Delta}{96\cdot12\cdot154\delta}\land\frac{4G^{2}\delta}{9\Delta}$,
let $d$ be in the order of $\frac{\Delta}{\delta\epsilon}\log\left(\frac{\Delta}{\delta\epsilon}\left(1+\left(\frac{\sigma}{\epsilon}\right)^{\frac{\p}{2(\p-1)}}\right)\right)$,
there exists a distribution over functions $F:\R^{d}\to\R$, and stochastic
first-order oracles $\bg$ such that with probability $1$, $F(\bzero)-F_{\star}\leq\Delta$,
$F$ is $G$-Lipschitz and $\bg$ has a finite $\p$-th centered moment
$\sigma^{\p}$. Moreover, for any randomized $\A\in\A_{\rand}$ employed
to optimize a randomly selected $F$ interacting with $\bg$, to output
a point $\bx$ such that $\E\left[\left\Vert \nabla F(\bx)\right\Vert _{\delta}\right]\leq\epsilon$,
the number of queries of $\bg$ by $\A$ satisfies $\gtrsim\Delta\delta^{-1}\epsilon^{-1}+\Delta\sigma^{\frac{\p}{\p-1}}\delta^{-1}\epsilon^{-\frac{2\p-1}{\p-1}}$.
\end{thm}
\begin{proof}
Let $H\defeq\epsilon/\delta$ in the following. Note that $\epsilon\leq\sqrt{\frac{\Delta H}{96\cdot12\cdot154}}$
due to our requirement on $\epsilon$. So we can consider the same
distribution on $F$ and $\bg$ as in Theorem \ref{thm:ncvx-lb-smooth}.
Importantly, $F$ is $H$-smooth and $\frac{3}{2}\sqrt{H\Delta}=\frac{3}{2}\sqrt{\frac{\epsilon\Delta}{\delta}}\leq G$-Lipschitz
(again due to the condition on $\epsilon$) with probability $1$. 

If $\A\in\A_{\rand}$ finds a point $\bx$ such that $\E\left[\left\Vert \nabla F(\bx)\right\Vert _{\delta}\right]\leq\epsilon$
(i.e., a $(\delta,\epsilon)$-stationary point of $F$), then by Proposition
14 of \citep{pmlr-v202-cutkosky23a} (or $\nu=1$ in Lemma \ref{lem:O2NC-holder}
given later), it also satisfies $\E\left[\left\Vert \nabla F(\bx)\right\Vert \right]\leq\epsilon+H\delta=2\epsilon$.
Therefore, Theorem \ref{thm:ncvx-lb-smooth} implies the number of
queries of $\bg$ by $\A$ is at least $\gtrsim\Delta H\epsilon^{-2}+\Delta H\sigma^{\frac{\p}{\p-1}}\epsilon^{-\frac{3\p-2}{\p-1}}=\Delta\delta^{-1}\epsilon^{-1}+\Delta\sigma^{\frac{\p}{\p-1}}\delta^{-1}\epsilon^{-\frac{2\p-1}{\p-1}}$. 
\end{proof}

Finally, we are able to prove Theorem \ref{thm:main-ncvx-lb}.

\begin{proof}[Proof of Theorem \ref{thm:main-ncvx-lb}]
We recall the fact that lower bounds on the distributional complexity
imply lower bounds on the minimax complexity \citep{nemirovskij1983problem}.
Thus, Theorem \ref{thm:main-ncvx-lb} can be directly concluded from
Theorem \ref{thm:ncvx-lb-nonsmooth}.
\end{proof}

\section{More Details about Further Extensions\label{sec:ext-details}}

This section contains more details of Section \ref{sec:extension}.

\subsection{Full Results for Smooth $\ell_{t}$}

We present the full version of Theorem \ref{thm:main-OCO-smooth}
and Corollary \ref{cor:main-cvx-smooth-avg}.
\begin{thm}[full version of Theorem \ref{thm:main-OCO-smooth}]
\label{thm:ext-OCO-smooth}Under Assumption \ref{assu:OCO} (with
replacing the third point by Condition \ref{cond:smooth}) and let
$S_{T}(\bx)\defeq\left(H\sum_{t=1}^{T}\ell_{t}(\bx)-\ell_{t}^{\inf}\right)\land\left(H\sum_{t=1}^{T}\ell_{t}(\bx)-\ell_{t}^{\star}+\sum_{t=1}^{T}\left\Vert \nabla\ell_{t}(\bx_{t}^{\star})\right\Vert ^{2}\right)$
where $\ell_{t}^{\inf}\defeq\inf_{\bx\in\R^{d}}\ell_{t}(\bx)$, $\bx_{t}^{\star}\defeq\argmin_{\bx\in\X}\ell_{t}(\bx)$,
and $\ell_{t}^{\star}\defeq\ell_{t}(\bx_{t}^{\star})$:
\begin{itemize}
\item taking $\eta_{t}=\frac{1}{4H}\land\gamma D\land\frac{D}{\sigma t^{1/\p}}$
for any $\gamma>0$ in $\A\in\left\{ \OGD,\DA\right\} $, we have
\[
\E\left[\reg_{T}^{\A}(\bx)\right]\lesssim HD^{2}+D\left(\frac{1}{\gamma}+\gamma S_{T}(\bx)\right)+\sigma DT^{1/\p},\forall\bx\in\X.
\]
\item taking $\eta=D/\sqrt{2}$ in $\A=\Ada$, we have
\[
\E\left[\reg_{T}^{\A}(\bx)\right]\lesssim HD^{2}+D\sqrt{S_{T}(\bx)}+\sigma DT^{1/\p},\forall\bx\in\X.
\]
\end{itemize}
\end{thm}
\begin{rem}
First, $\bx_{t}^{\star}$ must exist since $\X$ is a nonempty compact
convex set and $\ell_{t}$ is closed convex. Next, when $\ell_{t}\geq0$
on $\R^{d}$, we have $S_{T}(\bx)\leq H\sum_{t=1}^{T}\ell_{t}(\bx)-\ell_{t}^{\inf}\leq H\sum_{t=1}^{T}\ell_{t}(\bx)$.
This fact, together with the replacement of $\gamma$ by $\gamma/\sqrt{H}$
for $\OGD$ and $\DA$, recovers Theorem \ref{thm:main-OCO-smooth}.
\end{rem}
\begin{proof}
Under Condition \ref{cond:smooth}, the following famous inequality
holds (regardless of the convexity of $\ell_{t}$)
\begin{equation}
\left\Vert \nabla\ell_{t}(\bx)\right\Vert ^{2}\leq2H(\ell_{t}(\bx)-\ell_{t}^{\inf}),\forall\bx\in\R^{d}.\label{eq:smooth-global}
\end{equation}
For convex $\ell_{t}$, another similar inequality is also true
\[
\left\Vert \nabla\ell_{t}(\bx)-\nabla\ell_{t}(\bx_{t}^{\star})\right\Vert ^{2}\leq2H(\ell_{t}(\bx)-\ell_{t}^{\star}),\forall\bx\in\X,
\]
which implies
\begin{equation}
\left\Vert \nabla\ell_{t}(\bx)\right\Vert ^{2}\leq\left(\sqrt{2H(\ell_{t}(\bx)-\ell_{t}^{\star})}+\left\Vert \nabla\ell_{t}(\bx_{t}^{\star})\right\Vert \right)^{2}\leq3H(\ell_{t}(\bx)-\ell_{t}^{\star})+3\left\Vert \nabla\ell_{t}(\bx_{t}^{\star})\right\Vert ^{2},\forall\bx\in\X.\label{eq:smooth-local}
\end{equation}
By (\ref{eq:smooth-global}) and (\ref{eq:smooth-local}), there is
\begin{align}
\sum_{t=1}^{T}\left\Vert \nabla\ell_{t}(\bx_{t})\right\Vert ^{2} & \leq3\sum_{t=1}^{T}\left[H(\ell_{t}(\bx_{t})-\ell_{t}^{\inf})\land\left(H(\ell_{t}(\bx_{t})-\ell_{t}^{\star})+\left\Vert \nabla\ell_{t}(\bx_{t}^{\star})\right\Vert ^{2}\right)\right]\nonumber \\
 & =3H\reg_{T}^{\A}(\bx)+3\sum_{t=1}^{T}\left[H(\ell_{t}(\bx)-\ell_{t}^{\inf})\land\left(H(\ell_{t}(\bx)-\ell_{t}^{\star})+\left\Vert \nabla\ell_{t}(\bx_{t}^{\star})\right\Vert ^{2}\right)\right]\nonumber \\
 & \leq3H\reg_{T}^{\A}(\bx)+3S_{T}(\bx).\label{eq:smooth-1}
\end{align}
With (\ref{eq:smooth-1}) on hand, we can extend the regret bound
of $\A\in\left\{ \OGD,\DA,\Ada\right\} $ to smooth $\ell_{t}$.

For $\OGD$, note that (\ref{eq:OGD-cvx-2}) still holds up to the
change of $G$ by $\left\Vert \nabla\ell_{t}(\bx_{t})\right\Vert $,
meaning that for any $\bx\in\X$,
\begin{equation}
\sum_{t=1}^{T}\left\langle \bg_{t},\bx_{t}-\bx\right\rangle \leq\frac{D^{2}}{\eta_{T}}+\sum_{t=1}^{T}\eta_{t}\left\Vert \nabla\ell_{t}(\bx_{t})\right\Vert ^{2}+\C(\p)\eta_{t}^{\p-1}\left\Vert \err_{t}\right\Vert ^{\p}D^{2-\p}.\label{eq:old-smooth-OGD}
\end{equation}
For $\DA$, note that (\ref{eq:DA-cvx-2}) still holds up to the change
of $G$ by $\left\Vert \nabla\ell_{t}(\bx_{t})\right\Vert $, meaning
that for any $\bx\in\X$,
\begin{equation}
\sum_{t=1}^{T}\left\langle \bg_{t},\bx_{t}-\bx\right\rangle \leq\frac{D^{2}}{\eta_{T}}+\sum_{t=1}^{T}\eta_{t-1}\left\Vert \nabla\ell_{t}(\bx_{t})\right\Vert ^{2}+\C(\p)\eta_{t-1}^{\p-1}\left\Vert \err_{t}\right\Vert ^{\p}D^{2-\p},\label{eq:old-smooth-DA}
\end{equation}
where $\eta_{0}=\eta_{1}$. Now we write $\eta_{t}=\eta\land\frac{D}{\sigma t^{1/\p}}$
for $\eta=\frac{1}{4H}\land\gamma D$. For $\A\in\left\{ \OGD,\DA\right\} $,
taking expectations on both sides of (\ref{eq:old-smooth-OGD}) or
(\ref{eq:old-smooth-OGD}) implies
\begin{align*}
\E\left[\reg_{T}^{\A}(\bx)\right] & \lesssim\frac{D^{2}}{\eta}+\sigma DT^{1/\p}+\eta\E\left[\sum_{t=1}^{T}\left\Vert \nabla\ell_{t}(\bx_{t})\right\Vert ^{2}\right]\\
 & \overset{(\ref{eq:smooth-1})}{\lesssim}\frac{D^{2}}{\eta}+\sigma DT^{1/\p}+3\eta H\E\left[\reg_{T}^{\A}(\bx)\right]+3\eta S_{T}(\bx)
\end{align*}
Since $\eta=\frac{1}{4H}\land\gamma D\Rightarrow1-3\eta H\geq\frac{1}{4}$,
$\eta\leq\gamma D$ and $\frac{1}{\eta}\leq4H+\frac{1}{\gamma D}$,
we hence conclude
\[
\E\left[\reg_{T}^{\A}(\bx)\right]\lesssim\frac{D^{2}}{\eta(1-3\eta H)}+\frac{3\eta S_{T}(\bx)}{1-3\eta H}+\frac{\sigma DT^{1/\p}}{1-3\eta H}\lesssim HD^{2}+D\left(\frac{1}{\gamma}+\gamma S_{T}(\bx)\right)+\sigma DT^{1/\p}.
\]

For $\A=\Ada$, note that (\ref{eq:AdaGrad-cvx-3}) still holds, i.e.,
for any $\bx\in\X$,
\begin{equation}
\sum_{t=1}^{T}\left\langle \bg_{t},\bx_{t}-\bx\right\rangle \leq\sqrt{2}\left(\frac{D^{2}}{2\eta}+\eta\right)\left[\sqrt{\sum_{t=1}^{T}\left\Vert \nabla\ell_{t}(\bx_{t})\right\Vert ^{2}}+\left(\sum_{t=1}^{T}\left\Vert \err_{t}\right\Vert ^{\p}\right)^{\frac{1}{\p}}\right].\label{eq:old-smooth-AdaGrad}
\end{equation}
We take expectations on both sides of (\ref{eq:old-smooth-AdaGrad}),
then apply H\"{o}lder's inequality to have $\E\left[\left(\sum_{t=1}^{T}\left\Vert \err_{t}\right\Vert ^{\p}\right)^{\frac{1}{\p}}\right]\leq\left(\sum_{t=1}^{T}\E\left[\left\Vert \err_{t}\right\Vert ^{\p}\right]\right)^{\frac{1}{\p}}\leq\sigma T^{\frac{1}{\p}}$,
and plug in $\eta=D/\sqrt{2}$ to obtain
\[
\E\left[\reg_{T}^{\A}(\bx)\right]\lesssim D\E\left[\sqrt{\sum_{t=1}^{T}\left\Vert \nabla\ell_{t}(\bx_{t})\right\Vert ^{2}}\right]+\sigma DT^{1/\p}\leq D\sqrt{\E\left[\sum_{t=1}^{T}\left\Vert \nabla\ell_{t}(\bx_{t})\right\Vert ^{2}\right]}+\sigma DT^{1/\p},
\]
where the last step is again by H\"{o}lder's inequality. Combine
the above inequality with (\ref{eq:smooth-1}) to conclude
\begin{align*}
\E\left[\reg_{T}^{\A}(\bx)\right] & \lesssim D\sqrt{H\E\left[\reg_{T}^{\A}(\bx)\right]+S_{T}(\bx)}+\sigma DT^{1/\p}\\
\Rightarrow\E\left[\reg_{T}^{\A}(\bx)\right] & \lesssim HD^{2}+D\sqrt{S_{T}(\bx)}+\sigma DT^{1/\p}.
\end{align*}
\end{proof}
 
\begin{cor}[full version of Corollary \ref{cor:main-cvx-smooth-avg}]
\label{cor:ext-cvx-smooth-avg}Under Assumption \ref{assu:OCO} (with
replacing the third point by Condition \ref{cond:smooth}) for $\ell_{t}(\bx)=F(\bx)$
and let $\bx^{\star}\defeq\argmin_{\bx\in\X}F(\bx)$ and $\bar{\bx}_{T}\defeq\frac{1}{T}\sum_{t=1}^{T}\bx_{t}$,
taking $\eta_{t}=\frac{1}{4H}\land\frac{D}{\left\Vert \nabla F(\bx^{\star})\right\Vert \sqrt{T}}\land\frac{D}{\sigma t^{1/\p}}$
in $\A\in\left\{ \OGD,\DA\right\} $ or $\eta=D/\sqrt{2}$ in $\A=\Ada$,
we have 
\[
\E\left[F(\bar{\bx}_{T})-F(\bx^{\star})\right]\leq\frac{\E\left[\reg_{T}^{\A}(\bx^{\star})\right]}{T}\lesssim\frac{HD^{2}}{T}+\frac{\left\Vert \nabla F(\bx^{\star})\right\Vert D}{\sqrt{T}}+\frac{\sigma D}{T^{1-\frac{1}{\p}}}.
\]
\end{cor}
\begin{rem}
When $\bx^{\star}\in\X$ is also a global minimizer (i.e., $\bx^{\star}\in\mathrm{arginf}_{\bx\in\R^{d}}F(\bx)$),
we have $\left\Vert \nabla F(\bx^{\star})\right\Vert =0$ since $F$
is differentiable. This fact recovers Corollary \ref{cor:main-cvx-smooth-avg}.
\end{rem}
\begin{proof}
By convexity, $F(\bar{\bx}_{T})-F(\bx^{\star})\leq\frac{\sum_{t=1}^{T}F(\bx_{t})-F(\bx^{\star})}{T}=\frac{\reg_{T}^{\A}(\bx^{\star})}{T}$
is valid for any OCO algorithm $\A$. We conclude from invoking Theorem
\ref{thm:ext-OCO-smooth} with $\gamma=\frac{1}{\left\Vert \nabla F(\bx^{\star})\right\Vert \sqrt{T}}$
for $\A\in\left\{ \OGD,\DA\right\} $ and $\eta=D/\sqrt{2}$ for $\A=\Ada$
and combining the fact $S_{T}(\bx^{\star})\leq H\sum_{t=1}^{T}F(\bx^{\star})-F(\bx^{\star})+\sum_{t=1}^{T}\left\Vert \nabla F(\bx^{\star})\right\Vert ^{2}=\left\Vert \nabla F(\bx^{\star})\right\Vert ^{2}T$
in this case.
\end{proof}

\subsection{An Optimistic Algorithm under Heavy Tails}

As discussed in Section \ref{sec:extension}, our goal is to handle
broader cases with optimistic algorithms. To do so, we first introduce
the following new Assumption \ref{assu:ext-OCO}:
\begin{assumption}
\label{assu:ext-OCO}We consider the following series of assumptions:
\begin{itemize}
\item $\X\subset\R^{d}$ is a nonempty closed convex set bounded by $D$,
i.e., $\sup_{\bx,\by\in\X}\left\Vert \bx-\by\right\Vert \leq D$.
\item $\ell_{t}:\X\to\R$ is closed convex for all $t\in\left[T\right]$.
\item $\ell_{t}$ is $(G_{t},H_{t},\nu)$-general nonsmooth on $\X$, i.e.,
there exists $G_{t}\geq0$, $H_{t}\geq0$ and $\nu\in\left(0,1\right]$
such that $G_{t}+H_{t}>0$ and $\left\Vert \nabla\ell_{t}(\bx)-\nabla\ell_{t}(\by)\right\Vert \leq2G_{t}+H_{t}\left\Vert \bx-\by\right\Vert ^{\nu},\forall\bz\in\X,\nabla\ell_{t}(\bz)\in\partial\ell_{t}(\bz),\bz\in\left\{ \bx,\by\right\} $,
for all $t\in\left[T\right]$.
\item Given a point $\bx_{t}\in\X$ at the $t$-th iteration, one can query
$\bg_{t}\in\R^{d}$ satisfying $\nabla\ell_{t}(\bx_{t})\defeq\E\left[\bg_{t}\mid\F_{t-1}\right]\in\partial\ell_{t}(\bx_{t})$
and $\E\left[\left\Vert \err_{t}\right\Vert ^{\p}\right]\leq\sigma_{t}^{\p}$
for some $\p\in\left(1,2\right]$ and $\sigma_{t}\geq0$, where $\F_{t}\defeq\sigma(\bg_{1},\mydots,\bg_{t})$
denotes the natural filtration and $\err_{t}\defeq\bg_{t}-\nabla\ell_{t}(\bx_{t})$
is the stochastic noise.
\end{itemize}
\end{assumption}
Assumption \ref{assu:ext-OCO} generalizes Assumption \ref{assu:OCO}
since the latter is a special case of the former, i.e., when $(G_{t},H_{t},\nu)=(G,0,\nu)$
(for any $\nu\in\left(0,1\right]$) and $\sigma_{t}=\sigma$. When
$G_{t}=0$, this new assumption means that each $\ell_{t}$ is locally
H\"{o}lder smooth with time-varying parameters $(H_{t},\nu)$ on
$\X$. If we further let $\nu=1$, then each $\ell_{t}$ is standard
smooth with a parameter $H_{t}$.

\subsubsection{Optimistic $\protect\Ada$}

\begin{algorithm}[H]
\caption{\label{alg:OAda}Optimistic $\protect\Ada$ ($\protect\OAda$)}

\textbf{Input:} initial point $\bx_{1}\in\X$, initial hint $\bh_{1}\in\R^{d}$,
stepsize $\eta>0$ and $\gamma_{t}>0$

\textbf{for} $t=1$ \textbf{to} $T$ \textbf{do}

$\quad$Query a hint $\bh_{t+1}$

$\quad$$\eta_{t}=\eta V_{t}^{-1/2}\land\gamma_{t}$ where $V_{t}=\sum_{s=1}^{t}\left\Vert \bg_{s}-\bh_{s}\right\Vert ^{2}$

$\quad$$\bx_{t+1}=\argmin_{\bx\in\X}\left\langle \bg_{t}+\bh_{t+1}-\bh_{t},\bx\right\rangle +\frac{\left\Vert \bx-\bx_{t}\right\Vert ^{2}}{2\eta_{t}}$

\textbf{end for}
\end{algorithm}

Following the famous idea of optimism, we consider an optimistic
version of $\Ada$ named $\OAda$ in Algorithm \ref{alg:OAda}. To
start with, we give the following Lemma \ref{lem:OAda-core}. As a
sanity check, one can take $\bh_{t}=\bzero$ and $\gamma_{t}=+\infty$,
then $\OAda$ reduces to the standard $\Ada$ method, and Lemma \ref{lem:OAda-core}
recovers the path-wise regret in (\ref{eq:AdaGrad-path}). Therefore,
the following result can be viewed as a further extension of $\Ada$,
meaning that we can apply our idea described before (see the paragraph
under (\ref{eq:AdaGrad-path}) or the proof around (\ref{eq:AdaGrad-cvx-3}))
to overcome heavy-tailed noise.
\begin{lem}
\label{lem:OAda-core}Under Assumption \ref{assu:ext-OCO}, taking
any hint sequence $\bh_{t}\in\R^{d}$, $\eta=D/\sqrt{2}$, and any
nonincreasing stepsize $\gamma_{t}$ in $\OAda$ (Algorithm \ref{alg:OAda}),
we have
\[
\sum_{t=1}^{T}\left\langle \bg_{t},\bx_{t}-\bx\right\rangle \lesssim\left\langle \bh_{1},\bx_{1}-\bx\right\rangle +\frac{D^{2}}{\gamma_{T}}+D\sqrt{\sum_{t=1}^{T}\left\Vert \bg_{t}-\bh_{t}\right\Vert ^{2}}-\sum_{t=1}^{T-1}\frac{\left\Vert \bx_{t+1}-\bx_{t}\right\Vert ^{2}}{4\gamma_{t}},\forall\bx\in\X.
\]
\end{lem}
\begin{proof}
Given $t\in\left[T\right]$, by the optimality condition of the update
rule in Algorithm \ref{alg:OAda}, for any $\bx\in\X$,
\[
\left\langle \bg_{t}+\bh_{t+1}-\bh_{t}+\frac{\bx_{t+1}-\bx_{t}}{\eta_{t}},\bx_{t+1}-\bx\right\rangle \leq0,
\]
which implies
\begin{align*}
\left\langle \bg_{t}+\bh_{t+1}-\bh_{t},\bx_{t+1}-\bx\right\rangle  & \leq\frac{\left\langle \bx_{t}-\bx_{t+1},\bx_{t+1}-\bx\right\rangle }{\eta_{t}}\\
 & =\frac{\left\Vert \bx-\bx_{t}\right\Vert ^{2}-\left\Vert \bx-\bx_{t+1}\right\Vert ^{2}-\left\Vert \bx_{t+1}-\bx_{t}\right\Vert ^{2}}{2\eta_{t}},
\end{align*}
Therefore, we know
\begin{align*}
\left\langle \bg_{t},\bx_{t}-\bx\right\rangle = & \left\langle \bg_{t}+\bh_{t+1}-\bh_{t},\bx_{t+1}-\bx\right\rangle +\left\langle \bg_{t}-\bh_{t},\bx_{t}-\bx_{t+1}\right\rangle +\left\langle \bh_{t},\bx_{t}-\bx\right\rangle -\left\langle \bh_{t+1},\bx_{t+1}-\bx\right\rangle \\
\leq & \frac{\left\Vert \bx-\bx_{t}\right\Vert ^{2}-\left\Vert \bx-\bx_{t+1}\right\Vert ^{2}}{2\eta_{t}}+\left\langle \bg_{t}-\bh_{t},\bx_{t}-\bx_{t+1}\right\rangle -\frac{\left\Vert \bx_{t+1}-\bx_{t}\right\Vert ^{2}}{2\eta_{t}}\\
 & +\left\langle \bh_{t},\bx_{t}-\bx\right\rangle -\left\langle \bh_{t+1},\bx_{t+1}-\bx\right\rangle \\
\leq & \frac{\left\Vert \bx-\bx_{t}\right\Vert ^{2}-\left\Vert \bx-\bx_{t+1}\right\Vert ^{2}}{2\eta_{t}}+\eta_{t}\left\Vert \bg_{t}-\bh_{t}\right\Vert ^{2}-\frac{\left\Vert \bx_{t+1}-\bx_{t}\right\Vert ^{2}}{4\eta_{t}}\\
 & +\left\langle \bh_{t},\bx_{t}-\bx\right\rangle -\left\langle \bh_{t+1},\bx_{t+1}-\bx\right\rangle ,
\end{align*}
sum up which from $t=1$ to $T$ and drop the term $-\frac{\left\Vert \bx_{T+1}-\bx\right\Vert ^{2}}{2\eta_{T}}$
to have
\begin{align*}
\sum_{t=1}^{T}\left\langle \bg_{t},\bx_{t}-\bx\right\rangle \leq & \frac{\left\Vert \bx-\bx_{1}\right\Vert ^{2}}{2\eta_{1}}+\sum_{t=1}^{T-1}\left(\frac{1}{\eta_{t+1}}-\frac{1}{\eta_{t}}\right)\frac{\left\Vert \bx_{t+1}-\bx\right\Vert ^{2}}{2}+\sum_{t=1}^{T}\eta_{t}\left\Vert \bg_{t}-\bh_{t}\right\Vert ^{2}-\frac{\left\Vert \bx_{t+1}-\bx_{t}\right\Vert ^{2}}{4\eta_{t}}\\
 & +\left\langle \bh_{1},\bx_{1}-\bx\right\rangle -\left\langle \bh_{T+1},\bx_{T+1}-\bx\right\rangle \\
\overset{(a)}{\leq} & \frac{D^{2}}{2\eta_{T}}+\sum_{t=1}^{T}\eta_{t}\left\Vert \bg_{t}-\bh_{t}\right\Vert ^{2}-\frac{\left\Vert \bx_{t+1}-\bx_{t}\right\Vert ^{2}}{4\eta_{t}}+\left\langle \bh_{1},\bx_{1}-\bx\right\rangle -\left\langle \bh_{T+1},\bx_{T+1}-\bx\right\rangle \\
\overset{(b)}{\leq} & \frac{D^{2}}{2\gamma_{T}}+\frac{D^{2}\sqrt{V_{T}}}{2\eta}+\sum_{t=1}^{T}\frac{\eta\left\Vert \bg_{t}-\bh_{t}\right\Vert ^{2}}{\sqrt{V_{t}}}-\frac{\left\Vert \bx_{t+1}-\bx_{t}\right\Vert ^{2}}{4\gamma_{t}}+\left\langle \bh_{1},\bx_{1}-\bx\right\rangle -\left\langle \bh_{T+1},\bx_{T+1}-\bx\right\rangle \\
\overset{(c)}{\leq} & \frac{D^{2}}{2\gamma_{T}}+\left(\frac{D^{2}}{2\eta}+2\eta\right)\sqrt{\sum_{t=1}^{T}\left\Vert \bg_{t}-\bh_{t}\right\Vert ^{2}}-\sum_{t=1}^{T}\frac{\left\Vert \bx_{t+1}-\bx_{t}\right\Vert ^{2}}{4\gamma_{t}}+\left\langle \bh_{1},\bx_{1}-\bx\right\rangle -\left\langle \bh_{T+1},\bx_{T+1}-\bx\right\rangle ,
\end{align*}
where $(a)$ is by $\left\Vert \bx_{t}-\bx\right\Vert \leq D,\forall t\in\left[T\right]$
and $\eta_{t+1}\leq\eta_{t},\forall t\in\left[T-1\right]$ (since
$\gamma_{t}$ is assumed to be nonincreasing), $(b)$ is due to $\eta_{t}=\frac{\eta}{\sqrt{V_{t}}}\land\gamma_{t}\Rightarrow\frac{1}{\eta_{T}}=\frac{\sqrt{V_{T}}}{\eta}\lor\frac{1}{\gamma_{T}}\leq\frac{\sqrt{V_{T}}}{\eta}+\frac{1}{\gamma_{T}}$,
and $(c)$ follows a similar step as proving (\ref{eq:AdaGrad-cvx-2}).
Finally, we drop the term $-\frac{\left\Vert \bx_{T+1}-\bx_{T}\right\Vert ^{2}}{4\gamma_{T}}$,
use $\eta=D/\sqrt{2}$, and set $\bh_{T+1}=\bzero$ to obtain the
desired bound (this step is without loss of generality, otherwise,
one can simply change the $\bx_{T+1}$ used above to be $\argmin_{\bx\in\X}\left\langle \bg_{T}-\bh_{T},\bx\right\rangle +\frac{\left\Vert \bx-\bx_{T}\right\Vert ^{2}}{2\eta_{T}}$).
\end{proof}

\subsubsection{New Regret for $\protect\OAda$}

Equipped with Lemma \ref{lem:OAda-core}, we first prove the following
Theorem \ref{thm:OAda}, which establishes the regret of $\OAda$
under heavy tails.
\begin{thm}
\label{thm:OAda}Under Assumption \ref{assu:ext-OCO}, taking $\bh_{t}=\bg_{t-1}$
where $\bg_{0}\defeq\bzero$, $\eta=D/\sqrt{2}$, and $\gamma_{t}=\frac{D^{1-\nu}}{\max_{s\in\left[t\right]}H_{s}}$
in $\OAda$ (Algorithm \ref{alg:OAda}), we have
\[
\E\left[\reg_{T}^{\OAda}(\bx)\right]\lesssim D^{1+\nu}\left(\sum_{t=1}^{T}H_{t}^{\frac{2-\nu}{1-\nu}}\right)^{\frac{1-\nu}{2-\nu}}+D\sqrt{A_{T}+\sum_{t=1}^{T}G_{t}^{2}}+D\left(\sum_{t=1}^{T}\sigma_{t}^{\p}\right)^{\frac{1}{\p}},
\]
where $A_{T}\defeq\left\Vert \nabla\ell_{1}(\bx_{1})\right\Vert +\sum_{t=2}^{T}\sup_{\bx\in\X}\left\Vert \nabla\ell_{t}(\bx)-\nabla\ell_{t-1}(\bx)\right\Vert ^{2}$
and $\ell_{0}\defeq0$.
\end{thm}
We briefly discuss Theorem \ref{thm:OAda} before proving it. First,
the quantity $A_{T}$ is standard and well-known in the literature
as gradient variation \citep{pmlr-v23-chiang12}. Next, let us consider
the case $H_{t}=0$, then the bound degenerates to $D\sqrt{A_{T}+\sum_{t=1}^{T}G_{t}^{2}}+D\left(\sum_{t=1}^{T}\sigma_{t}^{\p}\right)^{1/\p}$,
which further reduces to the optimal regret $GD\sqrt{T}+\sigma DT^{1/\p}$
under Assumption \ref{assu:OCO}. In the case $G_{t}=0$, i.e., locally
H\"{o}lder smooth $\ell_{t}$, it gives a regret $D^{1+\nu}\left(\sum_{t=1}^{T}H_{t}^{\frac{2-\nu}{1-\nu}}\right)^{\frac{1-\nu}{2-\nu}}+D\sqrt{A_{T}}+D\left(\sum_{t=1}^{T}\sigma_{t}^{\p}\right)^{1/\p}$.
In the special situation, deterministic OCO under standard smoothness
(i.e., $\sigma_{t}=0$, $H_{t}=H$, and $\nu=1$), this matches the
classical result $HD^{2}+\sqrt{A_{T}}D$ \citep{pmlr-v23-chiang12}.
But as one can see, our Theorem \ref{thm:OAda} is more general and,
as far as we know, is the first bound containing gradient variation
for heavy-tailed OCO.

\begin{proof}
By Lemma \ref{lem:OAda-core} with $\bh_{t}=\bg_{t-1}$, we have
\begin{equation}
\sum_{t=1}^{T}\left\langle \bg_{t},\bx_{t}-\bx\right\rangle \lesssim\frac{D^{2}}{\gamma_{T}}+D\sqrt{\sum_{t=1}^{T}\left\Vert \bg_{t}-\bg_{t-1}\right\Vert ^{2}}-\sum_{t=1}^{T-1}\frac{\left\Vert \bx_{t+1}-\bx_{t}\right\Vert ^{2}}{4\gamma_{t}}.\label{eq:OAda-1}
\end{equation}
Let $\err_{0}\defeq\bzero$, we can find
\begin{align}
\sum_{t=1}^{T}\left\Vert \bg_{t}-\bg_{t-1}\right\Vert ^{2} & =\sum_{t=1}^{T}\left\Vert \nabla\ell_{t}(\bx_{t})-\nabla\ell_{t-1}(\bx_{t})+\nabla\ell_{t-1}(\bx_{t})-\nabla\ell_{t-1}(\bx_{t-1})+\err_{t}-\err_{t-1}\right\Vert ^{2}\nonumber \\
 & \lesssim\sum_{t=1}^{T}\left\Vert \nabla\ell_{t}(\bx_{t})-\nabla\ell_{t-1}(\bx_{t})\right\Vert ^{2}+\sum_{t=1}^{T}\left\Vert \nabla\ell_{t-1}(\bx_{t})-\nabla\ell_{t-1}(\bx_{t-1})\right\Vert ^{2}+\sum_{t=1}^{T}\left\Vert \err_{t}\right\Vert ^{2}\nonumber \\
 & \leq A_{T}+\sum_{t=2}^{T}\left\Vert \nabla\ell_{t-1}(\bx_{t})-\nabla\ell_{t-1}(\bx_{t-1})\right\Vert ^{2}+\sum_{t=1}^{T}\left\Vert \err_{t}\right\Vert ^{2},\label{eq:OAda-2}
\end{align}
where the last step is by the definition of $A_{T}$ and $\ell_{0}=0$.
Now by Assumption \ref{assu:ext-OCO}, $\left\Vert \nabla\ell_{t-1}(\bx_{t})-\nabla\ell_{t-1}(\bx_{t-1})\right\Vert ^{2}\lesssim G_{t-1}^{2}+H_{t-1}^{2}\left\Vert \bx_{t}-\bx_{t-1}\right\Vert ^{2\nu}$,
which implies
\begin{equation}
\sum_{t=2}^{T}\left\Vert \nabla\ell_{t-1}(\bx_{t})-\nabla\ell_{t-1}(\bx_{t-1})\right\Vert ^{2}\lesssim\sum_{t=2}^{T}G_{t-1}^{2}+H_{t-1}^{2}\left\Vert \bx_{t}-\bx_{t-1}\right\Vert ^{2\nu}=\sum_{t=1}^{T-1}G_{t}^{2}+H_{t}^{2}\left\Vert \bx_{t+1}-\bx_{t}\right\Vert ^{2\nu}.\label{eq:OAda-3}
\end{equation}
Combine (\ref{eq:OAda-2}) and (\ref{eq:OAda-3}) and use $\sqrt{a+b}\le\sqrt{a}+\sqrt{b},\forall a,b\geq0$
to have
\begin{align}
\sqrt{\sum_{t=1}^{T}\left\Vert \bg_{t}-\bg_{t-1}\right\Vert ^{2}} & \lesssim\sqrt{A_{T}+\sum_{t=1}^{T-1}G_{t}^{2}}+\sqrt{\sum_{t=1}^{T-1}H_{t}^{2}\left\Vert \bx_{t+1}-\bx_{t}\right\Vert ^{2\nu}}+\sqrt{\sum_{t=1}^{T}\left\Vert \err_{t}\right\Vert ^{2}}\nonumber \\
 & \leq\sqrt{A_{T}+\sum_{t=1}^{T-1}G_{t}^{2}}+\sqrt{\sum_{t=1}^{T-1}H_{t}^{2}\left\Vert \bx_{t+1}-\bx_{t}\right\Vert ^{2\nu}}+\left(\sum_{t=1}^{T}\left\Vert \err_{t}\right\Vert ^{\p}\right)^{\frac{1}{\p}},\label{eq:OAda-4}
\end{align}
where the last step is due to $\left\Vert \cdot\right\Vert _{2}\leq\left\Vert \cdot\right\Vert _{\p}$
for any $\p\in\left[1,2\right]$.

Therefore, by (\ref{eq:OAda-1}) and (\ref{eq:OAda-4}), the following
inequality holds
\begin{align}
\sum_{t=1}^{T}\left\langle \bg_{t},\bx_{t}-\bx\right\rangle \lesssim & \frac{D^{2}}{\gamma_{T}}+D\sqrt{\sum_{t=1}^{T-1}H_{t}^{2}\left\Vert \bx_{t+1}-\bx_{t}\right\Vert ^{2\nu}}-\sum_{t=1}^{T-1}\frac{\left\Vert \bx_{t+1}-\bx_{t}\right\Vert ^{2}}{4\gamma_{t}}\nonumber \\
 & +D\sqrt{A_{T}+\sum_{t=1}^{T-1}G_{t}^{2}}+D\left(\sum_{t=1}^{T}\left\Vert \err_{t}\right\Vert ^{\p}\right)^{\frac{1}{\p}}.\label{eq:OAda-5}
\end{align}
We use H\"{o}lder's inequality to bound
\[
\sum_{t=1}^{T-1}H_{t}^{2}\left\Vert \bx_{t+1}-\bx_{t}\right\Vert ^{2\nu}=\sum_{t=1}^{T-1}H_{t}^{2}\left(2\nu\gamma_{t}\right)^{\nu}\frac{\left\Vert \bx_{t+1}-\bx_{t}\right\Vert }{\left(2\nu\gamma_{t}\right)^{\nu}}^{2\nu}\leq\left(\sum_{t=1}^{T-1}H_{t}^{\frac{2}{1-\nu}}\left(2\nu\gamma_{t}\right)^{\frac{\nu}{1-\nu}}\right)^{1-\nu}\left(\sum_{t=1}^{T-1}\frac{\left\Vert \bx_{t+1}-\bx_{t}\right\Vert }{2\nu\gamma_{t}}^{2}\right)^{\nu},
\]
which implies
\begin{align}
D\sqrt{\sum_{t=1}^{T-1}H_{t}^{2}\left\Vert \bx_{t+1}-\bx_{t}\right\Vert ^{2\nu}} & \leq D\left(\sum_{t=1}^{T-1}H_{t}^{\frac{2}{1-\nu}}\left(2\nu\gamma_{t}\right)^{\frac{\nu}{1-\nu}}\right)^{\frac{1-\nu}{2}}\left(\sum_{t=1}^{T-1}\frac{\left\Vert \bx_{t+1}-\bx_{t}\right\Vert }{2\nu\gamma_{t}}^{2}\right)^{\frac{\nu}{2}}\nonumber \\
 & \leq\frac{D^{\frac{2}{2-\nu}}\left(\sum_{t=1}^{T-1}H_{t}^{\frac{2}{1-\nu}}\left(2\nu\gamma_{t}\right)^{\frac{\nu}{1-\nu}}\right)^{\frac{1-\nu}{2-\nu}}}{2/(2-\nu)}+\frac{\sum_{t=1}^{T-1}\frac{\left\Vert \bx_{t+1}-\bx_{t}\right\Vert }{2\nu\gamma_{t}}^{2}}{2/\nu}\nonumber \\
 & \lesssim D^{\frac{2}{2-\nu}}\left(\sum_{t=1}^{T-1}H_{t}^{\frac{2}{1-\nu}}\gamma_{t}^{\frac{\nu}{1-\nu}}\right)^{\frac{1-\nu}{2-\nu}}+\sum_{t=1}^{T-1}\frac{\left\Vert \bx_{t+1}-\bx_{t}\right\Vert ^{2}}{4\gamma_{t}},\label{eq:OAda-6}
\end{align}
where the second step is due to Young's inequality. Plug (\ref{eq:OAda-6})
back into (\ref{eq:OAda-5}) to know
\[
\sum_{t=1}^{T}\left\langle \bg_{t},\bx_{t}-\bx\right\rangle \lesssim\frac{D^{2}}{\gamma_{T}}+D^{\frac{2}{2-\nu}}\left(\sum_{t=1}^{T-1}H_{t}^{\frac{2}{1-\nu}}\gamma_{t}^{\frac{\nu}{1-\nu}}\right)^{\frac{1-\nu}{2-\nu}}+D\sqrt{A_{T}+\sum_{t=1}^{T-1}G_{t}^{2}}+D\left(\sum_{t=1}^{T}\left\Vert \err_{t}\right\Vert ^{\p}\right)^{\frac{1}{\p}}.
\]
Recall that $\gamma_{t}=\frac{D^{1-\nu}}{\max_{s\in\left[t\right]}H_{s}}$,
we hence have
\[
\sum_{t=1}^{T}\left\langle \bg_{t},\bx_{t}-\bx\right\rangle \lesssim D^{1+\nu}\left(\sum_{t=1}^{T-1}H_{t}^{\frac{2-\nu}{1-\nu}}\right)^{\frac{1-\nu}{2-\nu}}+D\sqrt{A_{T}+\sum_{t=1}^{T-1}G_{t}^{2}}+D\left(\sum_{t=1}^{T}\left\Vert \err_{t}\right\Vert ^{\p}\right)^{\frac{1}{\p}}.
\]
Finally, we conclude by taking expectations on both sides of the above
inequality and using H\"{o}lder's inequality again to obtain
\[
\E\left[\left(\sum_{t=1}^{T}\left\Vert \err_{t}\right\Vert ^{\p}\right)^{\frac{1}{\p}}\right]\leq\left(\E\left[\sum_{t=1}^{T}\left\Vert \err_{t}\right\Vert ^{\p}\right]\right)^{\frac{1}{\p}}\leq D\left(\sum_{t=1}^{T}\sigma_{t}^{\p}\right)^{\frac{1}{\p}}.
\]
\end{proof}

An undesired point of the above Theorem \ref{thm:OAda} is requiring
the knowledge of both $\nu$ and $H_{t}$ to set the stepsize $\gamma_{t}$.
In the following, we show that under a slightly variant nonsmooth
notion, one can relax this requirement. To do so, we consider Condition
\ref{cond:star} to substitute the third point in Assumption \ref{assu:ext-OCO}.
\begin{condition}
\label{cond:star}$\ell_{t}$ is $(G_{t},H_{t},\nu,\star)$-general
nonsmooth on $\X$, i.e., there exists $G_{t}\geq0$, $H_{t}\geq0$
and $\nu\in\left(0,1\right]$ such that $G_{t}+H_{t}>0$ and $\left\Vert \nabla\ell_{t}(\bx)-\nabla\ell_{t}(\bx_{t}^{\star})\right\Vert \lesssim G_{t}+H_{t}^{\frac{1}{1+\nu}}(\ell_{t}(\bx)-\ell_{t}^{\star})^{\frac{\nu}{1+\nu}},\forall\bx\in\X,\nabla\ell_{t}(\bx)\in\partial\ell_{t}(\bx)$
where $\bx_{t}^{\star}\defeq\argmin_{\bx\in\X}\ell_{t}(\bx)$ and
$\ell_{t}^{\star}\defeq\ell_{t}(\bx_{t}^{\star})$, for all $t\in\left[T\right]$.
\end{condition}
To gain some intuition of the new Condition \ref{cond:star}, we first
let $H_{t}=0$ to have $\left\Vert \nabla\ell_{t}(\bx)-\nabla\ell_{t}(\bx_{t}^{\star})\right\Vert \lesssim G_{t}$,
which is similar to $\left\Vert \nabla\ell_{t}(\bx)-\nabla\ell_{t}(\by)\right\Vert \lesssim G_{t},\forall\bx,\by\in\X$
in Assumption \ref{assu:ext-OCO}. Next, we consider the case $G_{t}=0$,
where Condition \ref{cond:star} states an inequality $\left\Vert \nabla\ell_{t}(\bx)-\nabla\ell_{t}(\bx_{t}^{\star})\right\Vert \lesssim H_{t}^{\frac{1}{1+\nu}}(\ell_{t}(\bx)-\ell_{t}^{\star})^{\frac{\nu}{1+\nu}}$,
which is known to hold if $\ell_{t}$ is $(H_{t},\nu)$-H\"{o}lder
smooth on $\R^{d}$ and convex, which can be viewed as a generalization
of Assumption \ref{assu:ext-OCO} (though strictly speaking, we only
require $\ell_{t}$ to be H\"{o}lder smooth on $\X$ in Assumption
\ref{assu:ext-OCO}). For general $G_{t}$, $H_{t}$, and $\nu$,
whether Condition \ref{cond:star} is strictly general than the third
point in Assumption \ref{assu:ext-OCO} is unclear. Thus, we consider
it as a separate condition here.

Now, we prove a new regret for $\OAda$ under Condition \ref{cond:star},
which no longer needs to know $H_{t}$ and $\nu$.
\begin{thm}
\label{thm:OAda-star}Under Assumption \ref{assu:ext-OCO} (with replacing
the third point by Condition \ref{cond:star}), taking $\bh_{t}=\bg_{t-1}$
where $\bg_{0}\defeq\bzero$, $\eta=D/\sqrt{2}$, and $\gamma_{t}=+\infty$
in $\OAda$ (Algorithm \ref{alg:OAda}), we have
\[
\E\left[\reg_{T}^{\OAda}(\bx)\right]\lesssim D^{1+\nu}\left(\sum_{t=1}^{T}H_{t}^{\frac{2}{1-\nu}}\right)^{\frac{1-\nu}{2}}+D\left(\sum_{t=1}^{T}H_{t}^{\frac{2}{1-\nu}}\right)^{\frac{1-\nu}{2(1+\nu)}}\left[C_{T}(\bx)\right]^{\frac{\nu}{1+\nu}}+D\sqrt{B_{T}+\sum_{t=1}^{T}G_{t}^{2}}+D\left(\sum_{t=1}^{T}\sigma_{t}^{\p}\right)^{\frac{1}{\p}},
\]
where $B_{T}\defeq\sum_{t=1}^{T}\left\Vert \nabla\ell_{t}(\bx_{t}^{\star})-\nabla\ell_{t-1}(\bx_{t-1}^{\star})\right\Vert ^{2}$,
$C_{T}(\bx)\defeq\sum_{t=1}^{T}\ell_{t}(\bx)-\ell_{t}^{\star}$, and
$\ell_{0}\defeq0$.
\end{thm}
Note that Theorem \ref{thm:OAda-star} and the previous Theorem \ref{thm:OAda}
are not directly comparable due to different assumptions. But we can
consider some special cases to better understand the difference. For
example, let $H_{t}=0$, Theorem \ref{thm:OAda-star} degenerates
to $D\sqrt{B_{T}+\sum_{t=1}^{T}G_{t}^{2}}+D\left(\sum_{t=1}^{T}\sigma_{t}^{\p}\right)^{\frac{1}{\p}}$,
similar to the regret $D\sqrt{A_{T}+\sum_{t=1}^{T}G_{t}^{2}}+D\left(\sum_{t=1}^{T}\sigma_{t}^{\p}\right)^{1/\p}$
by Theorem \ref{thm:OAda} in this case (note that both $B_{T}$ and
$A_{T}$ are at most in the same order of $T$). Next, if $G_{t}=0$,
then these two bounds are hard to compare due to the extra term $D\left(\sum_{t=1}^{T}H_{t}^{\frac{2}{1-\nu}}\right)^{\frac{1-\nu}{2(1+\nu)}}\left[C_{T}(\bx)\right]^{\frac{\nu}{1+\nu}}$
in Theorem \ref{thm:OAda-star}. But as one will see later, in convex
optimization, Theorem \ref{thm:OAda-star} possibly leads to a better
rate than Theorem \ref{thm:OAda}.

\begin{proof}
By Lemma \ref{lem:OAda-core} with $\bh_{t}=\bg_{t-1}$ and $\gamma_{t}=+\infty$,
we have
\begin{equation}
\sum_{t=1}^{T}\left\langle \bg_{t},\bx_{t}-\bx\right\rangle \lesssim D\sqrt{\sum_{t=1}^{T}\left\Vert \bg_{t}-\bg_{t-1}\right\Vert ^{2}}.\label{eq:OAda-star-1}
\end{equation}
Let $\err_{0}\defeq\bzero$, we can find
\begin{align}
 & \sum_{t=1}^{T}\left\Vert \bg_{t}-\bg_{t-1}\right\Vert ^{2}\nonumber \\
= & \sum_{t=1}^{T}\left\Vert \nabla\ell_{t}(\bx_{t})-\nabla\ell_{t}(\bx_{t}^{\star})+\nabla\ell_{t}(\bx_{t}^{\star})-\nabla\ell_{t-1}(\bx_{t-1}^{\star})+\nabla\ell_{t-1}(\bx_{t-1}^{\star})-\nabla\ell_{t-1}(\bx_{t-1})+\err_{t}-\err_{t-1}\right\Vert ^{2}\nonumber \\
\lesssim & \sum_{t=1}^{T}\left\Vert \nabla\ell_{t}(\bx_{t}^{\star})-\nabla\ell_{t-1}(\bx_{t-1}^{\star})\right\Vert ^{2}+\sum_{t=1}^{T}\left\Vert \nabla\ell_{t}(\bx_{t})-\nabla\ell_{t}(\bx_{t}^{\star})\right\Vert ^{2}+\sum_{t=1}^{T}\left\Vert \err_{t}\right\Vert ^{2}\nonumber \\
= & B_{T}+\sum_{t=1}^{T}\left\Vert \nabla\ell_{t}(\bx_{t})-\nabla\ell_{t}(\bx_{t}^{\star})\right\Vert ^{2}+\sum_{t=1}^{T}\left\Vert \err_{t}\right\Vert ^{2}.\label{eq:OAda-star-2}
\end{align}
By Condition \ref{cond:star}, we know $\left\Vert \nabla\ell_{t}(\bx_{t})-\nabla\ell_{t}(\bx_{t}^{\star})\right\Vert ^{2}\lesssim G_{t}^{2}+H_{t}^{\frac{2}{1+\nu}}(\ell_{t}(\bx_{t})-\ell_{t}^{\star})^{\frac{2\nu}{1+\nu}}$,
which implies
\begin{equation}
\sum_{t=1}^{T}\left\Vert \nabla\ell_{t}(\bx_{t})-\nabla\ell_{t}(\bx_{t}^{\star})\right\Vert ^{2}\lesssim\sum_{t=1}^{T}G_{t}^{2}+H_{t}^{\frac{2}{1+\nu}}(\ell_{t}(\bx_{t})-\ell_{t}^{\star})^{\frac{2\nu}{1+\nu}}.\label{eq:OAda-star-3}
\end{equation}
Combine (\ref{eq:OAda-star-2}) and (\ref{eq:OAda-star-3}) and use
$\sqrt{a+b}\le\sqrt{a}+\sqrt{b},\forall a,b\geq0$ to have
\begin{align}
\sqrt{\sum_{t=1}^{T}\left\Vert \bg_{t}-\bg_{t-1}\right\Vert ^{2}} & \lesssim\sqrt{B_{T}+\sum_{t=1}^{T}G_{t}^{2}}+\sqrt{\sum_{t=1}^{T}H_{t}^{\frac{2}{1+\nu}}(\ell_{t}(\bx_{t})-\ell_{t}^{\star})^{\frac{2\nu}{1+\nu}}}+\sqrt{\sum_{t=1}^{T}\left\Vert \err_{t}\right\Vert ^{2}}\nonumber \\
 & \leq\sqrt{B_{T}+\sum_{t=1}^{T}G_{t}^{2}}+\sqrt{\sum_{t=1}^{T}H_{t}^{\frac{2}{1+\nu}}(\ell_{t}(\bx_{t})-\ell_{t}^{\star})^{\frac{2\nu}{1+\nu}}}+\left(\sum_{t=1}^{T}\left\Vert \err_{t}\right\Vert ^{\p}\right)^{\frac{1}{\p}},\label{eq:OAda-star-4}
\end{align}
where the last step is due to $\left\Vert \cdot\right\Vert _{2}\leq\left\Vert \cdot\right\Vert _{\p}$
for any $\p\in\left[1,2\right]$.

Therefore, by (\ref{eq:OAda-star-1}) and (\ref{eq:OAda-star-4}),
the following inequality holds
\[
\sum_{t=1}^{T}\left\langle \bg_{t},\bx_{t}-\bx\right\rangle \lesssim D\sqrt{\sum_{t=1}^{T}H_{t}^{\frac{2}{1+\nu}}(\ell_{t}(\bx_{t})-\ell_{t}^{\star})^{\frac{2\nu}{1+\nu}}}+D\sqrt{B_{T}+\sum_{t=1}^{T}G_{t}^{2}}+D\left(\sum_{t=1}^{T}\left\Vert \err_{t}\right\Vert ^{\p}\right)^{\frac{1}{\p}}.
\]
We use H\"{o}lder's inequality to bound
\[
\sum_{t=1}^{T}H_{t}^{\frac{2}{1+\nu}}(\ell_{t}(\bx_{t})-\ell_{t}^{\star})^{\frac{2\nu}{1+\nu}}\leq\left(\sum_{t=1}^{T}H_{t}^{\frac{2}{1-\nu}}\right)^{\frac{1-\nu}{1+\nu}}\left(\sum_{t=1}^{T}\ell_{t}(\bx_{t})-\ell_{t}^{\star}\right)^{\frac{2\nu}{1+\nu}},
\]
which implies
\[
\sum_{t=1}^{T}\left\langle \bg_{t},\bx_{t}-\bx\right\rangle \lesssim D\left(\sum_{t=1}^{T}H_{t}^{\frac{2}{1-\nu}}\right)^{\frac{1-\nu}{2(1+\nu)}}\left(\sum_{t=1}^{T}\ell_{t}(\bx_{t})-\ell_{t}^{\star}\right)^{\frac{\nu}{1+\nu}}+D\sqrt{B_{T}+\sum_{t=1}^{T}G_{t}^{2}}+D\left(\sum_{t=1}^{T}\left\Vert \err_{t}\right\Vert ^{\p}\right)^{\frac{1}{\p}}.
\]
Now take expectations on both sides of the above inequality and use
the fact $\E\left[\left(\sum_{t=1}^{T}\left\Vert \err_{t}\right\Vert ^{\p}\right)^{\frac{1}{\p}}\right]\leq\left(\E\left[\sum_{t=1}^{T}\left\Vert \err_{t}\right\Vert ^{\p}\right]\right)^{\frac{1}{\p}}\leq D\left(\sum_{t=1}^{T}\sigma_{t}^{\p}\right)^{\frac{1}{\p}}$
(due to H\"{o}lder's inequality) to find for $\A=\OAda$
\begin{align}
\E\left[\reg_{T}^{\A}(\bx)\right] & \lesssim D\left(\sum_{t=1}^{T}H_{t}^{\frac{2}{1-\nu}}\right)^{\frac{1-\nu}{2(1+\nu)}}\E\left[\left(\sum_{t=1}^{T}\ell_{t}(\bx_{t})-\ell_{t}^{\star}\right)^{\frac{\nu}{1+\nu}}\right]+D\sqrt{B_{T}+\sum_{t=1}^{T}G_{t}^{2}}+D\left(\sum_{t=1}^{T}\sigma_{t}^{\p}\right)^{\frac{1}{\p}}\nonumber \\
 & \leq D\left(\sum_{t=1}^{T}H_{t}^{\frac{2}{1-\nu}}\right)^{\frac{1-\nu}{2(1+\nu)}}\left(\E\left[\sum_{t=1}^{T}\ell_{t}(\bx_{t})-\ell_{t}^{\star}\right]\right)^{\frac{\nu}{1+\nu}}+D\sqrt{B_{T}+\sum_{t=1}^{T}G_{t}^{2}}+D\left(\sum_{t=1}^{T}\sigma_{t}^{\p}\right)^{\frac{1}{\p}}\nonumber \\
 & =D\left(\sum_{t=1}^{T}H_{t}^{\frac{2}{1-\nu}}\right)^{\frac{1-\nu}{2(1+\nu)}}\left(\E\left[\reg_{T}^{\A}(\bx)\right]+C_{T}(\bx)\right)^{\frac{\nu}{1+\nu}}+D\sqrt{B_{T}+\sum_{t=1}^{T}G_{t}^{2}}+D\left(\sum_{t=1}^{T}\sigma_{t}^{\p}\right)^{\frac{1}{\p}},\label{eq:OAda-star-5}
\end{align}
where the second step is due to H\"{o}lder's inequality.

If $\E\left[\reg_{T}^{\A}(\bx)\right]\leq0$, we are done. Hence,
we only need to consider the case $\E\left[\reg_{T}^{\A}(\bx)\right]\geq0$.
If $\E\left[\reg_{T}^{\A}(\bx)\right]\leq C_{T}(\bx)$, then (\ref{eq:OAda-star-5})
implies
\[
\E\left[\reg_{T}^{\A}(\bx)\right]\lesssim D\left(\sum_{t=1}^{T}H_{t}^{\frac{2}{1-\nu}}\right)^{\frac{1-\nu}{2(1+\nu)}}\left[C_{T}(\bx)\right]^{\frac{\nu}{1+\nu}}+D\sqrt{B_{T}+\sum_{t=1}^{T}G_{t}^{2}}+D\left(\sum_{t=1}^{T}\sigma_{t}^{\p}\right)^{\frac{1}{\p}}.
\]
Otherwise, we have
\[
\E\left[\reg_{T}^{\A}(\bx)\right]\lesssim D\left(\sum_{t=1}^{T}H_{t}^{\frac{2}{1-\nu}}\right)^{\frac{1-\nu}{2(1+\nu)}}\left(\E\left[\reg_{T}^{\A}(\bx)\right]\right)^{\frac{\nu}{1+\nu}}+D\sqrt{B_{T}+\sum_{t=1}^{T}G_{t}^{2}}+D\left(\sum_{t=1}^{T}\sigma_{t}^{\p}\right)^{\frac{1}{\p}},
\]
which implies
\[
\E\left[\reg_{T}^{\A}(\bx)\right]\lesssim D^{1+\nu}\left(\sum_{t=1}^{T}H_{t}^{\frac{2}{1-\nu}}\right)^{\frac{1-\nu}{2}}+D\left(\sum_{t=1}^{T}H_{t}^{\frac{2}{1-\nu}}\right)^{\frac{1-\nu}{2(1+\nu)}}\left[C_{T}(\bx)\right]^{\frac{\nu}{1+\nu}}+D\sqrt{B_{T}+\sum_{t=1}^{T}G_{t}^{2}}+D\left(\sum_{t=1}^{T}\sigma_{t}^{\p}\right)^{\frac{1}{\p}}.
\]
So in all three cases, we have
\[
\E\left[\reg_{T}^{\A}(\bx)\right]\lesssim D^{1+\nu}\left(\sum_{t=1}^{T}H_{t}^{\frac{2}{1-\nu}}\right)^{\frac{1-\nu}{2}}+D\left(\sum_{t=1}^{T}H_{t}^{\frac{2}{1-\nu}}\right)^{\frac{1-\nu}{2(1+\nu)}}\left[C_{T}(\bx)\right]^{\frac{\nu}{1+\nu}}+D\sqrt{B_{T}+\sum_{t=1}^{T}G_{t}^{2}}+D\left(\sum_{t=1}^{T}\sigma_{t}^{\p}\right)^{\frac{1}{\p}}.
\]
\end{proof}

\section{More Applications\label{sec:ext-applications}}

Based on the results for $\OAda$ given in Appendix \ref{sec:ext-details},
we provide more applications in this section.

\subsection{General Nonsmooth Convex Optimization}

With Theorems \ref{thm:OAda} and \ref{thm:OAda-star}, we immediately
obtain the following convergence rates for convex optimization. The
proof of which is omitted to save space.
\begin{cor}
\label{cor:ext-cvx-avg}Under Assumption \ref{assu:ext-OCO} for $\ell_{t}(\bx)=F(\bx)$
(meaning that $G_{t}=G$, $H_{t}=H$, and $\sigma_{t}=\sigma$) and
let $\bx^{\star}\defeq\argmin_{\bx\in\X}F(\bx)$ and $\bar{\bx}_{T}\defeq\frac{1}{T}\sum_{t=1}^{T}\bx_{t}$:
\begin{itemize}
\item considering the same setting for $\A=\OAda$ as in Theorem \ref{thm:OAda},
we have 
\[
\E\left[F(\bar{\bx}_{T})-F(\bx^{\star})\right]\leq\frac{\E\left[\reg_{T}^{\A}(\bx^{\star})\right]}{T}\lesssim\frac{\left\Vert \nabla F(\bx^{\star})\right\Vert D}{T}+\frac{HD^{1+\nu}}{T^{\frac{1}{2-\nu}}}+\frac{GD}{\sqrt{T}}+\frac{\sigma D}{T^{1-\frac{1}{\p}}}.
\]
\item replacing the third point in Assumption \ref{assu:ext-OCO} by Condition
\ref{cond:star} and considering the same setting for $\A=\OAda$
as in Theorem \ref{thm:OAda-star}, we have 
\[
\E\left[F(\bar{\bx}_{T})-F(\bx^{\star})\right]\leq\frac{\E\left[\reg_{T}^{\A}(\bx^{\star})\right]}{T}\lesssim\frac{\left\Vert \nabla F(\bx^{\star})\right\Vert D}{T}+\frac{HD^{1+\nu}}{T^{\frac{1+\nu}{2}}}+\frac{GD}{\sqrt{T}}+\frac{\sigma D}{T^{1-\frac{1}{\p}}}.
\]
\end{itemize}
\end{cor}
If we consider Assumption \ref{assu:OCO}, (i.e., $H=0$ and $\left\Vert \nabla F(\bx)\right\Vert \leq G$),
Corollary \ref{cor:ext-cvx-avg} is as fast as Corollary \ref{cor:main-cvx-avg}
and hence optimal. Under the same assumption used in Corollary \ref{cor:ext-cvx-smooth-avg}
(i.e., $G=0$ and $\nu=1$), both rates degenerates to $\frac{\left\Vert \nabla F(\bx^{\star})\right\Vert D}{T}+\frac{HD^{2}}{T}+\frac{\sigma D}{T^{1-\frac{1}{\p}}}$
and are faster than the bound $\frac{HD^{2}}{T}+\frac{\left\Vert \nabla F(\bx^{\star})\right\Vert D}{\sqrt{T}}+\frac{\sigma D}{T^{1-\frac{1}{\p}}}$
given in Corollary \ref{cor:ext-cvx-smooth-avg}. If we specialize
in globally H\"{o}lder smooth functions (i.e., $\left\Vert \nabla F(\bx)-\nabla F(\by)\right\Vert \leq H\left\Vert \bx-\by\right\Vert ^{\nu},\forall,\bx,\by\in\R^{d}$),
then both bounds are new under heavy tails and the second one is faster.

\subsection{General Nonsmooth Nonconvex Optimization}

In this section, we move back to the nonsmooth nonconvex optimization
studied before in Section \ref{sec:applications}. But instead of
assuming $F$ is Lipschitz as in Assumption \ref{assu:ncvx}, we will
use the following general Assumption \ref{assu:ext-ncvx}. Note that
a function satisfying Assumption \ref{assu:ncvx} with a Lipschitz
parameter $G$ also fits Assumption \ref{assu:ext-ncvx} for $H=0$.
\begin{assumption}
\label{assu:ext-ncvx}We consider the following series of assumptions:
\begin{itemize}
\item The objective $F$ is lower bounded by $F_{\star}\triangleq\inf_{\bx\in\R^{d}}F(\bx)\in\R$.
\item $F$ is differentiable and well-behaved, i.e., $F(\bx)-F(\by)=\int_{0}^{1}\left\langle \nabla F(\by+t(\bx-\by)),\bx-\by\right\rangle \d t$.
\item $F$ is $(G,H,\nu)$-general nonsmooth on $\R^{d}$, i.e., there exists
$G\geq0$, $H\geq0$ and $\nu\in\left(0,1\right]$ such that $G+H>0$
and $\left\Vert \nabla F(\bx)-\nabla F(\by)\right\Vert \leq2G+H\left\Vert \bx-\by\right\Vert ^{\nu},\forall\bx,\by\in\R^{d}$.
\item Given $\bz_{t}\in\R^{d}$ at the $t$-th iteration, one can query
$\bg_{t}\in\R^{d}$ satisfying $\E\left[\bg_{t}\mid\F_{t-1}\right]=\nabla F(\bz_{t})$
and $\E\left[\left\Vert \err_{t}\right\Vert ^{\p}\right]\leq\sigma^{\p}$
for some $\p\in\left(1,2\right]$ and $\sigma\geq0$, where $\F_{t}$
denotes the natural filtration and $\err_{t}\defeq\bg_{t}-\nabla F(\bz_{t})$
is the stochastic noise.
\end{itemize}
\end{assumption}
Our goal is still to find a $(\delta,\epsilon)$-stationary point,
but under Assumption \ref{assu:ext-ncvx} this time. Fortunately,
we can still start from the $\otnc$ framework (Algorithm \ref{alg:O2NC}),
since the proof of Theorem \ref{thm:main-ncvx-core} only relies on
the first two and the last one conditions in Assumption \ref{assu:ncvx},
which are the same as Assumption \ref{assu:ext-ncvx}. In other words,
Theorem \ref{thm:main-ncvx-core} holds under this more general Assumption
\ref{assu:ext-ncvx}.

\subsubsection{$\protect\OAda$ with Reset}

\begin{algorithm}[H]
\caption{\label{alg:OAdaR}$\protect\OAda$ with Reset ($\protect\OAdaR$)}

\textbf{Input:} initial point $\bx_{1}\in\B^{d}(D)$, parameter $D>0$.

Set $\bg_{0}=\bzero$ and $V_{0}=0$

\textbf{for} $n=1$ \textbf{to} $KT$ \textbf{do}

$\quad$$V_{n}=V_{n-1}\1\left[n\mod T\neq1\right]+\left\Vert \bg_{n}-\bg_{n-1}\right\Vert ^{2}$

$\quad$$\eta_{n}=\sqrt{2}D/\sqrt{V_{n}}$

$\quad$$\bx_{n+1}=\argmin_{\bx\in\B^{d}(D)}\left\langle 2\bg_{n}-\bg_{n-1},\bx\right\rangle +\frac{\left\Vert \bx-\bx_{n}\right\Vert ^{2}}{2\eta_{n}}$

$\quad$\textbf{if} $n+1\mod T=1$\textbf{ do}

$\qquad$$\bx_{n+1}=-D\frac{\bg_{n}}{\left\Vert \bg_{n}\right\Vert }$

$\quad$\textbf{end if}

\textbf{end for}
\end{algorithm}

With Theorem \ref{thm:main-ncvx-core} on hand, our next task is naturally
to find an online learning algorithm $\A$ that has a proper $K$-shifting
regret under Assumption \ref{assu:ext-ncvx} (especially the third
point in which). Thanks to the framework presented in Appendix \ref{sec:ext-details},
we already have some clues, i.e., employing an optimistic algorithm.
In fact, such an idea has also been studied in \citep{pmlr-v202-cutkosky23a}
but only for the deterministic smooth case, i.e., $\sigma=G=0$ and
$\nu=1$ in Assumption \ref{assu:ext-ncvx}.

To handle the more general case, particularly including heavy-tailed
noise, we first present a new method called $\OAda$ with Reset ($\OAdaR$)
in Algorithm \ref{alg:OAdaR}, which can be viewed as running $\OAda$
for $T$ iterations with the hint $\bh_{t}=\bg_{t-1}$ and then resetting
in total of $K$ times.

We clarify that the idea of reset is not new, which was originally
suggested by \citep{pmlr-v202-cutkosky23a}, as also used previously
in Theorem \ref{thm:main-ncvx-general}. However, the step of how
to reset in $\OAdaR$, i.e., $\bx_{n+1}=-D\frac{\bg_{n}}{\left\Vert \bg_{n}\right\Vert }$
if $n+1\mod T=1$, is critical and novel as far as we know. Why is
it important? This is because the hint $\bh_{t}$ is not reset to
$\bzero$. Hence, if we reset $\bx_{n+1}$ to an arbitrary point in
$\B^{d}(D)$, then one can imagine that we will face a redundant term
$\sum_{k=1}^{K}\left\langle \bh_{(k-1)T},\bx_{(k-1)T+1}-\bv_{k}\right\rangle $
in total (see Lemma \ref{lem:OAda-core}), which is however undesired.
Instead, our specially designed way of resetting $\bx_{n+1}$ can
resolve this potential issue, as reflected in the following Lemma
\ref{lem:OAdaR}.
\begin{lem}
\label{lem:OAdaR}For any initial point $\bx_{1}\in\B^{d}(D)$ and
any sequence $\bv_{1},\mydots,\bv_{K}$ satisfying $\bv_{k}\in\B^{d}(D),\forall k\in\left[K\right]$,
the online learning algorithm $\A=\OAdaR$ (Algorithm \ref{alg:OAdaR})
guarantees
\[
\reg_{T}^{\A}\left(\bv_{1},\mydots,\bv_{K}\right)\lesssim D\sum_{k=1}^{K}\sqrt{\sum_{n=(k-1)T+1}^{kT}\left\Vert \bg_{n}-\bg_{n-1}\right\Vert ^{2}},
\]
where $\reg_{T}^{\A}(\bv_{1},\mydots,\bv_{K})$ is the $K$-shifting
regret defined in (\ref{eq:ncvx-shift}).
\end{lem}
\begin{proof}
Given $k\in\left[K\right]$, for $n\in\left\{ (k-1)T+1,\mydots,kT\right\} $,
$\OAdaR$ (Algorithm \ref{alg:OAdaR}) is the same as running $\OAda$
(Algorithm \ref{alg:OAda}) for $T$ iterations with the feasible
set $\X=\B^{d}(D)$ (which implies $\sup_{\bx,\by\in\X}\left\Vert \bx-\by\right\Vert \leq2D$),
initial point $\bx_{(k-1)T+1}$, hint sequence $\bh_{t}=\bg_{(k-1)T+t-1}$,
stepsize $\eta=2D/\sqrt{2}$ and $\gamma_{t}=+\infty$, and the stochastic
gradient sequence $\bg_{(k-1)T+t}$. Therefore, Lemma \ref{lem:OAda-core}
implies that for any $\bv_{k}\in\B^{d}(D)$,
\begin{align}
\sum_{n=(k-1)T+1}^{kT}\left\langle \bg_{n},\bx_{n}-\bv_{k}\right\rangle  & \lesssim\left\langle \bg_{(k-1)T},\bx_{(k-1)T+1}-\bv_{k}\right\rangle +D\sqrt{\sum_{t=1}^{T}\left\Vert \bg_{(k-1)T+t}-\bg_{(k-1)T+t-1}\right\Vert ^{2}}.\nonumber \\
 & =\left\langle \bg_{(k-1)T},\bx_{(k-1)T+1}-\bv_{k}\right\rangle +D\sqrt{\sum_{n=(k-1)T+1}^{kT}\left\Vert \bg_{n}-\bg_{n-1}\right\Vert ^{2}}\label{eq:OAdaR-1}
\end{align}

Observe that if $k=1$, we have
\[
\left\langle \bg_{(k-1)T},\bx_{(k-1)T+1}-\bv_{k}\right\rangle =\left\langle \bg_{0},\bx_{1}-\bv_{1}\right\rangle =\left\langle \bzero,\bx_{1}-\bv_{1}\right\rangle =0.
\]
If $k\neq1$, we use $\bx_{(k-1)T+1}=-D\frac{\bx_{(k-1)T}}{\left\Vert \bx_{(k-1)T}\right\Vert }$
and $\left\Vert \bv_{k}\right\Vert \leq D$ to have
\[
\left\langle \bg_{(k-1)T},\bx_{(k-1)T+1}-\bv_{k}\right\rangle =-D\left\Vert \bg_{(k-1)T}\right\Vert +\left\langle \bg_{(k-1)T},-\bv_{k}\right\rangle \leq0.
\]
Thus there is always 
\begin{equation}
\left\langle \bg_{(k-1)T},\bx_{(k-1)T+1}-\bv_{k}\right\rangle \leq0.\label{eq:OAdaR-2}
\end{equation}

Finally, we combine (\ref{eq:OAdaR-1}) and (\ref{eq:OAdaR-2}) to
know for any $k\in\left[K\right]$,
\[
\sum_{n=(k-1)T+1}^{kT}\left\langle \bg_{n},\bx_{n}-\bv_{k}\right\rangle \lesssim D\sqrt{\sum_{n=(k-1)T+1}^{kT}\left\Vert \bg_{n}-\bg_{n-1}\right\Vert ^{2}},\forall\bv_{k}\in\B^{d}(D),
\]
sum up which from $k=1$ to $K$ to obtain the following desired result
\[
\reg_{T}^{\A}\left(\bv_{1},\mydots,\bv_{K}\right)=\sum_{k=1}^{K}\sum_{n=(k-1)T+1}^{kT}\left\langle \bg_{n},\bx_{n}-\bv_{k}\right\rangle \lesssim D\sum_{k=1}^{K}\sqrt{\sum_{n=(k-1)T+1}^{kT}\left\Vert \bg_{n}-\bg_{n-1}\right\Vert ^{2}}.
\]
\end{proof}

\subsubsection{Convergence Rates}

Equipped with Lemma \ref{lem:OAdaR} for $\OAdaR$, we are ready to
show convergence rates under the new Assumption \ref{assu:ext-ncvx}.
First, we prove the following Theorem \ref{thm:ext-ncvx-general},
which can be viewed as a generalization of Theorem \ref{thm:main-ncvx-general}.
\begin{thm}
\label{thm:ext-ncvx-general}Under Assumption \ref{assu:ext-ncvx}
and let $\nabla\defeq\left\Vert \nabla F(\by_{0})\right\Vert $, $\Delta\defeq F(\by_{0})-F_{\star}$,
and $\bar{\bz}_{k}\defeq\frac{1}{T}\sum_{n=(k-1)T+1}^{kT}\bz_{n},\forall k\in\left[K\right]$,
setting $\A=\OAdaR$ (Algorithm \ref{alg:OAdaR}) in $\otnc$ (Algorithm
\ref{alg:O2NC}) with $D=\delta/T$, we have
\[
\E\left[\frac{1}{K}\sum_{k=1}^{K}\left\Vert \nabla F(\bar{\bz}_{k})\right\Vert _{\delta}\right]\lesssim\frac{\nabla}{KT}+\frac{\Delta}{\delta K}+\frac{H\delta^{\nu}}{T^{\frac{1}{2}+\nu}}+\frac{G}{\sqrt{T}}+\frac{\sigma}{T^{1-\frac{1}{\p}}}.
\]
\end{thm}
\begin{proof}
By Theorem \ref{thm:main-ncvx-core} and Lemma \ref{lem:OAdaR}, there
is
\begin{equation}
\E\left[\sum_{k=1}^{K}\frac{1}{K}\left\Vert \frac{1}{T}\sum_{n=(k-1)T+1}^{kT}\nabla F(\bz_{n})\right\Vert \right]\lesssim\frac{F(\by_{0})-F_{\star}}{DKT}+\frac{\sum_{k=1}^{K}\E\left[\sqrt{\sum_{n=(k-1)T+1}^{kT}\left\Vert \bg_{n}-\bg_{n-1}\right\Vert ^{2}}\right]}{KT}+\frac{\sigma}{T^{1-\frac{1}{\p}}}.\label{eq:ext-ncvx-general-1}
\end{equation}

We first lower bound the L.H.S. of (\ref{eq:ext-ncvx-general-1}).
Same as (\ref{eq:ncvx-general-ball}), we still have
\[
\bz_{n}\in\B^{d}(\bar{\bz}_{k},\delta),\forall n\in\left\{ (k-1)T+1,\mydots,kT\right\} .
\]
By the definition of $\left\Vert \nabla F(\bar{\bz}_{k})\right\Vert _{\delta}$
(see Definition \ref{def:delta-norm}), there is
\begin{equation}
\left\Vert \nabla F(\bar{\bz}_{k})\right\Vert _{\delta}\leq\left\Vert \frac{1}{T}\sum_{n=(k-1)T+1}^{kT}\nabla F(\bz_{n})\right\Vert .\label{eq:ext-ncvx-general-lhs}
\end{equation}

Next, we upper bound the R.H.S. of (\ref{eq:ext-ncvx-general-1}).
Before moving on, we list two facts that will be frequently used later,
i.e., for any $n\in\left[KT\right]$,
\begin{eqnarray}
\left\Vert \bx_{n}\right\Vert \leq D & \text{and} & s_{n}\in\left[0,1\right].\label{eq:ext-ncvx-general-facts}
\end{eqnarray}
Now, let us upper bound $\left\Vert \bg_{n}-\bg_{n-1}\right\Vert ^{2}$
on the R.H.S. of (\ref{eq:ext-ncvx-general-1}) as follows.
\begin{itemize}
\item $n=1$. In this case, we recall $\bg_{0}=\bzero$ and hence have
\begin{align}
\left\Vert \bg_{n}-\bg_{n-1}\right\Vert ^{2} & =\left\Vert \bg_{1}\right\Vert ^{2}=\left\Vert \err_{1}+\nabla F(\bz_{1})-\nabla F(\by_{0})+\nabla F(\by_{0})\right\Vert ^{2}\nonumber \\
 & \lesssim\left\Vert \err_{1}\right\Vert ^{2}+\left\Vert \nabla F(\bz_{1})-\nabla F(\by_{0})\right\Vert ^{2}+\left\Vert \nabla F(\by_{0})\right\Vert ^{2}\nonumber \\
 & \overset{\text{Assumption }\ref{assu:ext-ncvx}}{\lesssim}\left\Vert \err_{1}\right\Vert ^{2}+G^{2}+H^{2}\left\Vert \bz_{1}-\by_{0}\right\Vert ^{2\nu}+\left\Vert \nabla F(\by_{0})\right\Vert ^{2}\nonumber \\
 & \overset{(a)}{\lesssim}\left\Vert \err_{1}\right\Vert ^{2}+G^{2}+H^{2}D^{2\nu}+\left\Vert \nabla F(\by_{0})\right\Vert ^{2},\label{eq:ext-ncvx-general-2}
\end{align}
where $(a)$ is by $\left\Vert \bz_{1}-\by_{0}\right\Vert =s_{1}\left\Vert \bx_{1}\right\Vert \overset{(\ref{eq:ext-ncvx-general-facts})}{\leq}D$.
\item $n\neq1$. In this case, we have
\begin{align}
\left\Vert \bg_{n}-\bg_{n-1}\right\Vert ^{2} & =\left\Vert \err_{n}-\err_{n-1}+\nabla F(\bz_{n})-\nabla F(\bz_{n-1})\right\Vert ^{2}\nonumber \\
 & \lesssim\left\Vert \err_{n}\right\Vert ^{2}+\left\Vert \err_{n-1}\right\Vert ^{2}+\left\Vert \nabla F(\bz_{n})-\nabla F(\bz_{n-1})\right\Vert ^{2}\nonumber \\
 & \overset{\text{Assumption }\ref{assu:ext-ncvx}}{\lesssim}\left\Vert \err_{n}\right\Vert ^{2}+\left\Vert \err_{n-1}\right\Vert ^{2}+G^{2}+H^{2}\left\Vert \bz_{n}-\bz_{n-1}\right\Vert ^{2\nu}\nonumber \\
 & \overset{(b)}{\lesssim}\left\Vert \err_{n}\right\Vert ^{2}+\left\Vert \err_{n-1}\right\Vert ^{2}+G^{2}+H^{2}D^{2\nu},\label{eq:ext-ncvx-general-3}
\end{align}
where $(b)$ is by $\left\Vert \bz_{n}-\bz_{n-1}\right\Vert =\left\Vert s_{n}\bx_{n}+(1-s_{n-1})\bx_{n-1}\right\Vert \leq s_{n}\left\Vert \bx_{n}\right\Vert +(1-s_{n-1})\left\Vert \bx_{n-1}\right\Vert \overset{(\ref{eq:ext-ncvx-general-facts})}{\leq}2D$.
\end{itemize}
Thus, we can find for any $k\in\left[K\right]$,
\[
\sum_{n=(k-1)T+1}^{kT}\left\Vert \bg_{n}-\bg_{n-1}\right\Vert ^{2}\overset{(\ref{eq:ext-ncvx-general-2}),(\ref{eq:ext-ncvx-general-3})}{\lesssim}\left\Vert \nabla F(\by_{0})\right\Vert ^{2}\1\left[k=1\right]+\left(G^{2}+H^{2}D^{2\nu}\right)T+\sum_{n=(k-1)T}^{kT}\left\Vert \err_{n}\right\Vert ^{2},
\]
where $\err{}_{0}\defeq\bzero$ for simplicity. As such, we obtain
\begin{align*}
\sqrt{\sum_{n=(k-1)T+1}^{kT}\left\Vert \bg_{n}-\bg_{n-1}\right\Vert ^{2}} & \lesssim\left\Vert \nabla F(\by_{0})\right\Vert \1\left[k=1\right]+\left(HD^{\nu}+G\right)\sqrt{T}+\sqrt{\sum_{n=(k-1)T}^{kT}\left\Vert \err_{n}\right\Vert ^{2}}\\
 & \leq\left\Vert \nabla F(\by_{0})\right\Vert \1\left[k=1\right]+\left(HD^{\nu}+G\right)\sqrt{T}+\left(\sum_{n=(k-1)T}^{kT}\left\Vert \err_{n}\right\Vert ^{\p}\right)^{\frac{1}{\p}},
\end{align*}
where the last step is by $\left\Vert \cdot\right\Vert _{2}\leq\left\Vert \cdot\right\Vert _{\p}$
for $\p\in\left[1,2\right]$. By H\"{o}lder's inequality, there is
\[
\E\left[\left(\sum_{n=(k-1)T}^{kT}\left\Vert \err_{n}\right\Vert ^{\p}\right)^{\frac{1}{\p}}\right]\leq\left(\sum_{n=(k-1)T}^{kT}\E\left[\left\Vert \err_{n}\right\Vert ^{\p}\right]\right)^{\frac{1}{\p}}\overset{\text{Assumption }\ref{assu:ext-ncvx}}{\lesssim}\sigma(T+1)^{\frac{1}{\p}}.
\]
Thus, we know
\begin{equation}
\E\left[\sqrt{\sum_{n=(k-1)T+1}^{kT}\left\Vert \bg_{n}-\bg_{n-1}\right\Vert ^{2}}\right]\lesssim\left\Vert \nabla F(\by_{0})\right\Vert \1\left[k=1\right]+\left(HD^{\nu}+G\right)\sqrt{T}+\sigma T^{\frac{1}{\p}},\forall k\in\left[K\right].\label{eq:ext-ncvx-general-rhs}
\end{equation}

Finally, we plug (\ref{eq:ext-ncvx-general-lhs}) and (\ref{eq:ext-ncvx-general-rhs})
back into (\ref{eq:ext-ncvx-general-1}), then use $D=\delta/T$,
$\nabla=\left\Vert \nabla F(\by_{0})\right\Vert $, and $\Delta=F(\by_{0})-F_{\star}$
to have
\[
\E\left[\frac{1}{K}\sum_{k=1}^{K}\left\Vert \nabla F(\bar{\bz}_{k})\right\Vert _{\delta}\right]\lesssim\frac{\nabla}{KT}+\frac{\Delta}{\delta K}+\frac{H\delta^{\nu}}{T^{\frac{1}{2}+\nu}}+\frac{G}{\sqrt{T}}+\frac{\sigma}{T^{1-\frac{1}{\p}}}.
\]
\end{proof}

Similar to before, we can consider two situations where problem-dependent
parameters are known or unknown, corresponding to the following Corollaries
\ref{cor:ext-ncvx-dep} and \ref{cor:ext-ncvx-free}, respectively.
To save space, the proofs are omitted since they can be easily checked.
\begin{cor}
\label{cor:ext-ncvx-dep}Under the same setting in Theorem \ref{thm:ext-ncvx-general},
suppose we have $N\geq2$ stochastic gradient budgets, taking $K=\left\lfloor N/T\right\rfloor $
and $T=\left\lceil N/2\right\rceil \land\left(\left\lceil \left(\frac{\delta^{1+\nu}HN}{\Delta}\right)^{\frac{2}{3+2\nu}}\right\rceil \lor\left\lceil \left(\frac{\delta GN}{\Delta}\right)^{\frac{2}{3}}\right\rceil \lor\left\lceil \left(\frac{\delta\sigma N}{\Delta}\right)^{\frac{\p}{2\p-1}}\right\rceil \right)$,
we have
\begin{align*}
\E\left[\frac{1}{K}\sum_{k=1}^{K}\left\Vert \nabla F(\bar{\bz}_{k})\right\Vert _{\delta}\right]\lesssim & \frac{\nabla}{N}+\frac{H\delta^{\nu}}{N^{\frac{1}{2}+\nu}}+\frac{G}{\sqrt{N}}+\frac{\sigma}{N^{1-\frac{1}{\p}}}+\frac{\Delta}{\delta N}\\
 & +\frac{H^{\frac{2}{3+2\nu}}\Delta^{\frac{1+2\nu}{3+2\nu}}}{\delta^{\frac{1}{3+2\nu}}N^{\frac{1+2\nu}{3+2\nu}}}+\frac{G^{\frac{2}{3}}\Delta^{\frac{1}{3}}}{(\delta N)^{\frac{1}{3}}}+\frac{\sigma^{\frac{\p}{2\p-1}}\Delta^{\frac{\p-1}{2\p-1}}}{(\delta N)^{\frac{\p-1}{2\p-1}}}.
\end{align*}
\end{cor}
\begin{cor}
\label{cor:ext-ncvx-free}Under the same setting in Theorem \ref{thm:ext-ncvx-general},
suppose we have $N\geq2$ stochastic gradient budgets, taking $K=\left\lfloor N/T\right\rfloor $
and $T=\left\lceil N/2\right\rceil \land\left\lceil (\delta N)^{\frac{2}{3}}\right\rceil $,
we have
\begin{align*}
\E\left[\frac{1}{K}\sum_{k=1}^{K}\left\Vert \nabla F(\bar{\bz}_{k})\right\Vert _{\delta}\right]\lesssim & \frac{\nabla}{N}+\frac{H\delta^{\nu}}{N^{\frac{1}{2}+\nu}}+\frac{H}{\delta^{\frac{2-\nu}{5}}N^{\frac{1+2\nu}{5}}}\\
 & +\frac{\Delta}{(\delta N)\land(\delta N)^{\frac{1}{3}}}+\frac{G}{\sqrt{N}\land(\delta N)^{\frac{1}{3}}}+\frac{\sigma}{N^{1-\frac{1}{\p}}\land(\delta N)^{\frac{2(\p-1)}{3\p}}}.
\end{align*}
\end{cor}

\subsubsection{A Special Case: H\"{o}lder Smooth Nonconvex Functions}

We now consider a special case of $G=0$ in Assumption \ref{assu:ext-ncvx},
meaning that $F$ is H\"{o}lder smooth. Now, due to the smoothness,
finding an $\epsilon$-stationary point instead of a $(\delta,\epsilon)$-stationary
point is more reasonable. In the following Lemma \ref{lem:O2NC-holder},
we connect these two notions. Especially, when $\nu=1$, Lemma \ref{lem:O2NC-holder}
recovers Proposition 14 of \citep{pmlr-v202-cutkosky23a}.
\begin{lem}
\label{lem:O2NC-holder}If $F$ is $(H,\nu)$-H\"{o}lder smooth,
i.e., there exists $H>0$ and $\nu\in\left(0,1\right]$ such that
$\left\Vert \nabla F(\bx)-\nabla F(\by)\right\Vert \leq H\left\Vert \bx-\by\right\Vert ^{\nu},\forall\bx,\by\in\R^{d}$,
then $\left\Vert \nabla F(\bx)\right\Vert \leq\left\Vert \nabla F(\bx)\right\Vert _{\delta}+H\delta^{\nu},\forall\bx\in\R^{d},\delta>0$.
\end{lem}
\begin{proof}
For any fixed $\bx\in\R^{d}$ and $\delta,\epsilon>0$, by the definition
of $\left\Vert \nabla F(\bx)\right\Vert _{\delta}$, there exists
a finite set $S\subset\B^{d}(\bx,\delta)$ such that $\frac{1}{\left|S\right|}\sum_{\by\in S}\by=\bx$
and $\left\Vert \frac{1}{\left|S\right|}\sum_{\by\in S}\nabla F(\by)\right\Vert \leq\left\Vert \nabla F(\bx)\right\Vert _{\delta}+\epsilon$.
Hence, we know
\begin{align*}
\left\Vert \nabla F(\bx)\right\Vert  & \leq\left\Vert \frac{1}{\left|S\right|}\sum_{\by\in S}\nabla F(\by)\right\Vert +\left\Vert \nabla F(\bx)-\frac{1}{\left|S\right|}\sum_{\by\in S}\nabla F(\by)\right\Vert \\
 & \leq\left\Vert \nabla F(\bx)\right\Vert _{\delta}+\epsilon+\left\Vert \nabla F(\bx)-\frac{1}{\left|S\right|}\sum_{\by\in S}\nabla F(\by)\right\Vert \\
 & \leq\left\Vert \nabla F(\bx)\right\Vert _{\delta}+\epsilon+\frac{1}{\left|S\right|}\sum_{\by\in S}\left\Vert \nabla F(\bx)-\nabla F(\by)\right\Vert \\
 & \leq\left\Vert \nabla F(\bx)\right\Vert _{\delta}+\epsilon+\frac{H}{\left|S\right|}\sum_{\by\in S}\left\Vert \bx-\by\right\Vert ^{\nu}\leq\left\Vert \nabla F(\bx)\right\Vert _{\delta}+\epsilon+H\delta^{\nu}.
\end{align*}
Take $\epsilon\to0$ to conclude.
\end{proof}

Armed with Lemma \ref{lem:O2NC-holder}, we can prove the following
Theorem \ref{thm:ext-ncvx-holder-general}.
\begin{thm}
\label{thm:ext-ncvx-holder-general}Under Assumption \ref{assu:ext-ncvx}
(with $G=0$) and let $\nabla\defeq\left\Vert \nabla F(\by_{0})\right\Vert $,
$\Delta\defeq F(\by_{0})-F_{\star}$, and $\bar{\bz}_{k}\defeq\frac{1}{T}\sum_{n=(k-1)T+1}^{kT}\bz_{n},\forall k\in\left[K\right]$,
setting $\A=\OAdaR$ (Algorithm \ref{alg:OAdaR}) in $\otnc$ (Algorithm
\ref{alg:O2NC}) with $D=\delta/T$, we have
\[
\E\left[\frac{1}{K}\sum_{k=1}^{K}\left\Vert \nabla F(\bar{\bz}_{k})\right\Vert \right]\lesssim\frac{\nabla}{KT}+\frac{\Delta}{\delta K}+\frac{\sigma}{T^{1-\frac{1}{\p}}}+H\delta^{\nu}.
\]
\end{thm}
\begin{proof}
We first invoke Lemma \ref{lem:O2NC-holder} and Theorem \ref{thm:ext-ncvx-general}
with $G=0$, then use $T\geq1$ to conclude.
\end{proof}

Note that now $\delta$ should also be viewed as a parameter decided
by the user. Therefore, we will also choose the value of $\delta$
in the following two corollaries, corresponding to the cases where
problem-dependent parameters are known and unknown, respectively.
Again, the proofs are omitted to save space, and they are easy to
check.
\begin{cor}
\label{cor:ext-ncvx-holder-dep}Under the same setting in Theorem
\ref{thm:ext-ncvx-holder-general}, let $r\defeq\p\nu+(\p-1)(1+\nu)$,
suppose we have $N\geq2$ stochastic gradient budgets, taking $K=\left\lfloor N/T\right\rfloor $,
$T=\left\lceil N/2\right\rceil \land\left\lceil \frac{\sigma^{\frac{\p(1+\nu)}{r}}N^{\frac{\p\nu}{r}}}{H^{\frac{\p}{r}}\Delta^{\frac{\p\nu}{r}}}+1\right\rceil $,
and $\delta=\left(\frac{\Delta T}{HN}\right)^{\frac{1}{1+\nu}}$,
we have
\[
\E\left[\frac{1}{K}\sum_{k=1}^{K}\left\Vert \nabla F(\bar{\bz}_{k})\right\Vert \right]\lesssim\frac{\nabla}{N}+\frac{H^{\frac{1}{1+\nu}}\Delta^{\frac{\nu}{1+\nu}}}{N^{\frac{\nu}{1+\nu}}}+\frac{\sigma}{N^{1-\frac{1}{\p}}}+\frac{\sigma^{\frac{\p\nu}{r}}H^{\frac{\p-1}{r}}\Delta^{\frac{(\p-1)\nu}{r}}}{N^{\frac{(\p-1)\nu}{r}}}.
\]
\end{cor}
\begin{cor}
\label{cor:ext-ncvx-holder-free}Under the same setting in Theorem
\ref{thm:ext-ncvx-holder-general}, suppose we have $N\geq2$ stochastic
gradient budgets, taking $K=\left\lfloor N/T\right\rfloor $, $T=\left\lceil \sqrt{N}\right\rceil $,
and $\delta=1/N^{1/4}$, we have
\[
\E\left[\frac{1}{K}\sum_{k=1}^{K}\left\Vert \nabla F(\bar{\bz}_{k})\right\Vert \right]\lesssim\frac{\nabla}{N}+\frac{\Delta}{N^{\frac{1}{4}}}+\frac{H}{N^{\frac{\nu}{4}}}+\frac{\sigma}{N^{\frac{\p-1}{2\p}}}.
\]
\end{cor}
As a sanity check, when $\nu=1$ (i.e., the standard smooth case),
using the fact $\nabla=\left\Vert \nabla F(\by_{0})\right\Vert \lesssim\sqrt{H(F(\by_{0})-F_{\star})}=\sqrt{H\Delta}$,
Corollary \ref{cor:ext-ncvx-holder-dep} reduces to a rate $\sqrt{\frac{H\Delta}{N}}+\frac{\sigma}{N^{1-\frac{1}{\p}}}+\left(\frac{\sigma^{\frac{\p}{\p-1}}H\Delta}{N}\right)^{\frac{\p-1}{3\p-2}}$
and Corollary \ref{cor:ext-ncvx-holder-free} gives a rate $\frac{\Delta+H}{N^{\frac{1}{4}}}+\frac{\sigma}{N^{\frac{\p-1}{2\p}}}$,
both of which match the best possible results in the respective situations
\citep{pmlr-v258-hubler25a,liu2025nonconvex}.

For general $\nu\in\left(0,1\right)$, as far as we know, no previous
works consider heavy-tailed noise. Hence, both corollaries are the
first and new.

\section{Algebraic Facts}

We give three useful algebraic facts in this section.
\begin{fact}
\label{fact:last-ineq}For any $T\in\N$ and $a\in\left(0,1\right)$,
there is 
\[
\sum_{t=1}^{T-1}\frac{\sum_{s=t+1}^{T}s^{a}}{t(T-t)^{2}}\lesssim\frac{1+\log T}{T^{1-a}}.
\]
\end{fact}
\begin{proof}
Note that $\sum_{s=t+1}^{T}s^{a}\leq(T-t)T^{a}$, which implies
\[
\sum_{t=1}^{T-1}\frac{\sum_{s=t+1}^{T}s^{a}}{t(T-t)^{2}}\leq\sum_{t=1}^{T-1}\frac{T^{a}}{t(T-t)}=\frac{1}{T^{1-a}}\sum_{t=1}^{T-1}\frac{1}{t}+\frac{1}{T-t}=\frac{2}{T^{1-a}}\sum_{t=1}^{T-1}\frac{1}{t}\lesssim\frac{1+\log T}{T^{1-a}}.
\]
\end{proof}

\begin{fact}
\label{fact:N/2}Given $2\leq N\in\N$, $K=\left\lfloor N/T\right\rfloor $
and $T\in\N$ satisfying $T\leq\left\lceil N/2\right\rceil $, there
is $KT\geq N/4$.
\end{fact}
\begin{proof}
Note that $KT=\left\lfloor N/T\right\rfloor T\geq N-T\geq(N-1)/2\geq N/4$.
\end{proof}

\begin{fact}
\label{fact:p-q}Given $\p\in\left(1,2\right]$ and $q\in\left(0,1\right)$,
there are $q^{\p-1}+(1-q)^{\p-1}\leq2^{2-\p}$ and $1-q\leq\frac{1-q^{\p-1}}{\p-1}$.
\end{fact}
\begin{proof}
Note that $x^{\p-1}$ is concave when $\p\in\left(1,2\right]$, we
hence have $\frac{q^{\p-1}+(1-q)^{\p-1}}{2}\leq\left(\frac{q+1-q}{2}\right)^{\p-1}\Rightarrow q^{\p-1}+(1-q)^{\p-1}\leq2^{2-\p}$.
Next, let $x=1-q\in\left(0,1\right)$, we have
\[
1-q\leq\frac{1-q^{\p-1}}{\p-1}\Leftrightarrow(1-x)^{\p-1}\leq1-(\p-1)x,
\]
which is true by Bernoulli's inequality.
\end{proof}

\end{document}